\documentclass{article}

\usepackage[utf8]{inputenc}
\usepackage[nottoc,numbib]{tocbibind}

\usepackage[parfill]{parskip}    
\usepackage{amssymb}
\usepackage{epstopdf}

\usepackage{microtype}

\usepackage{booktabs}

\usepackage{hyperref}
\usepackage{soul}

\usepackage[accepted]{icml2021}

\usepackage{amsmath,amsthm,amsfonts,amssymb}
\usepackage{enumerate,color,xcolor}
\usepackage{graphicx}
\usepackage{subfigure}
\usepackage{url}
   
\usepackage{enumitem}

\usepackage{multirow}

 \usepackage[utf8]{inputenc}

\usepackage[parfill]{parskip}    
\usepackage{epstopdf}

\numberwithin{equation}{section}
\theoremstyle{plain}

\newtheorem{theorem}{Theorem}[section]

\newtheorem{corollary}{Corollary}[section]

\newtheorem{lemma}{Lemma}[section]

\newtheorem{proposition}{Proposition}[section]

\newtheorem{remark}{Remark}[section]
\newtheorem{property}{Property}[section]
\newtheorem{assumption}{Assumption}[section]

\usepackage{natbib}
\usepackage{accents}
\setcitestyle{authoryear,open={(},close={)}}

\usepackage{mathrsfs}
\usepackage{amsmath}
\usepackage[toc,page,header]{appendix}
\usepackage{minitoc}

\usepackage{tocvsec2}

\usepackage{pdfpages}
\usepackage{attachfile2}
\usepackage[abs]{overpic}
\usepackage{pict2e}
\usepackage{algorithm}

\usepackage{cancel}
\usepackage{tikz}

\usetikzlibrary{bayesnet}
\usetikzlibrary{arrows,decorations.markings}
\usepackage[english]{isodate}
\usepackage{hyperref}
\usepackage{booktabs}
\usepackage{xcolor}
\usepackage[colorinlistoftodos,prependcaption,textsize=tiny]{todonotes}
\usepackage{mathtools}
\usepackage{mismath}
\usepackage{stackengine}

\fontsize{11}{13}                   

\usetikzlibrary{patterns}
\usetikzlibrary{bayesnet}
\usetikzlibrary{arrows,decorations.markings}
\tikzstyle{disent_latent} = [circle,pattern=north east lines, pattern color=black!20,draw=black,inner sep=1pt,
minimum size=20pt, font=\fontsize{10}{10}\selectfont, node distance=1]

\DeclareMathOperator{\expect}{\mathbb{E}}

\DeclarePairedDelimiterX\MeijerM[3]{\lparen}{\rparen}%
{\begin{smallmatrix}#1 \\ #2\end{smallmatrix}\delimsize\Vert\,#3}

\newtheorem{definition}{Definition}[section]

\newcommand{\vecto}[1]{\boldsymbol{\mathbf{#1}}}
\renewcommand{\v}{\vecto}

\theoremstyle{plain}

\usepackage{comment}

\newcommand{\beq}{\begin{eqnarray}}
\newcommand{\eeq}{\end{eqnarray}}
\newcommand{\beqs}{\begin{eqnarray*}}
\newcommand{\eeqs}{\end{eqnarray*}}

\DeclareMathOperator*{\argmin}{arg\min}

\DeclareMathOperator{\sign}{sign}

\usepackage{mathtools}

\usepackage{silence}
\ActivateWarningFilters[pdftoc]
\ActivateWarningFilters[xclr]

\icmltitlerunning{Asymmetric Heavy Tails and Implicit Bias in Gaussian Noise Injections}
\begin{document}

\doparttoc
\faketableofcontents

\twocolumn[
\icmltitle{Asymmetric Heavy Tails and Implicit Bias in Gaussian Noise Injections}



\icmlsetsymbol{equal}{*}

\begin{icmlauthorlist}
\icmlauthor{Alexander Camuto}{equal,ox}
\icmlauthor{Xiaoyu Wang}{equal,fsu}
\icmlauthor{Lingjiong Zhu}{fsu}
\icmlauthor{Mert G\"{u}rb\"{u}zbalaban}{ru}
\icmlauthor{Chris Holmes}{ox}
\icmlauthor{Umut \c{S}im\c{s}ekli}{tpt}
\end{icmlauthorlist}

\icmlaffiliation{tpt}{
INRIA - D\'{e}partement d'Informatique de l'\'{E}cole Normale Sup\'{e}rieure - PSL Research University, Paris, France}
\icmlaffiliation{ox}{Alan Turing Institute, University of Oxford, Oxford, UK}
\icmlaffiliation{fsu}{Department of Mathematics, Florida State University, Tallahassee, USA}
\icmlaffiliation{ru}{Department of Management Science and Information Systems, Rutgers Business School, Piscataway, USA}

\icmlcorrespondingauthor{Alexander Camuto}{acamuto@turing.ac.uk}

%\icmlkeywords{Stochastic Gradient Descent, Stochastic Differential Equations, Kinetic Langevin, Deep Learning}

\vskip 0.3in
]



%\printAffiliationsAndNotice{} 
\printAffiliationsAndNotice{\icmlEqualContribution}

\begin{abstract}
Gaussian noise injections (GNIs) are a family of simple and widely-used regularisation methods for training neural networks where one injects additive or multiplicative Gaussian noise to the network activations at every iteration of the optimisation algorithm, which is typically chosen as stochastic gradient descent (SGD).
In this paper we focus on the so-called `implicit effect' of GNIs, which is the effect of the injected noise on the dynamics of SGD.
We show that this effect induces an \emph{asymmetric heavy-tailed noise} on SGD gradient updates.
In order to model this modified dynamics, we first develop a Langevin-like stochastic differential equation that is driven by a general family of \emph{asymmetric} heavy-tailed noise.
Using this model we then formally prove that GNIs induce an `implicit bias', which varies depending on the heaviness of the tails and the level of asymmetry.
Our empirical results confirm that different types of neural networks trained with GNIs are well-modelled by the proposed dynamics and that the implicit effect of these injections induces a bias that degrades the performance of networks. 
\end{abstract}

% tirnak `

%%%%%%%%%%%%%%%%%%%%%%%%%%%%%%%%%%%%%%%
\section{Introduction}
\label{intro}

 \begin{figure}[t!]
    \centering
    \includegraphics[width=0.45\textwidth]{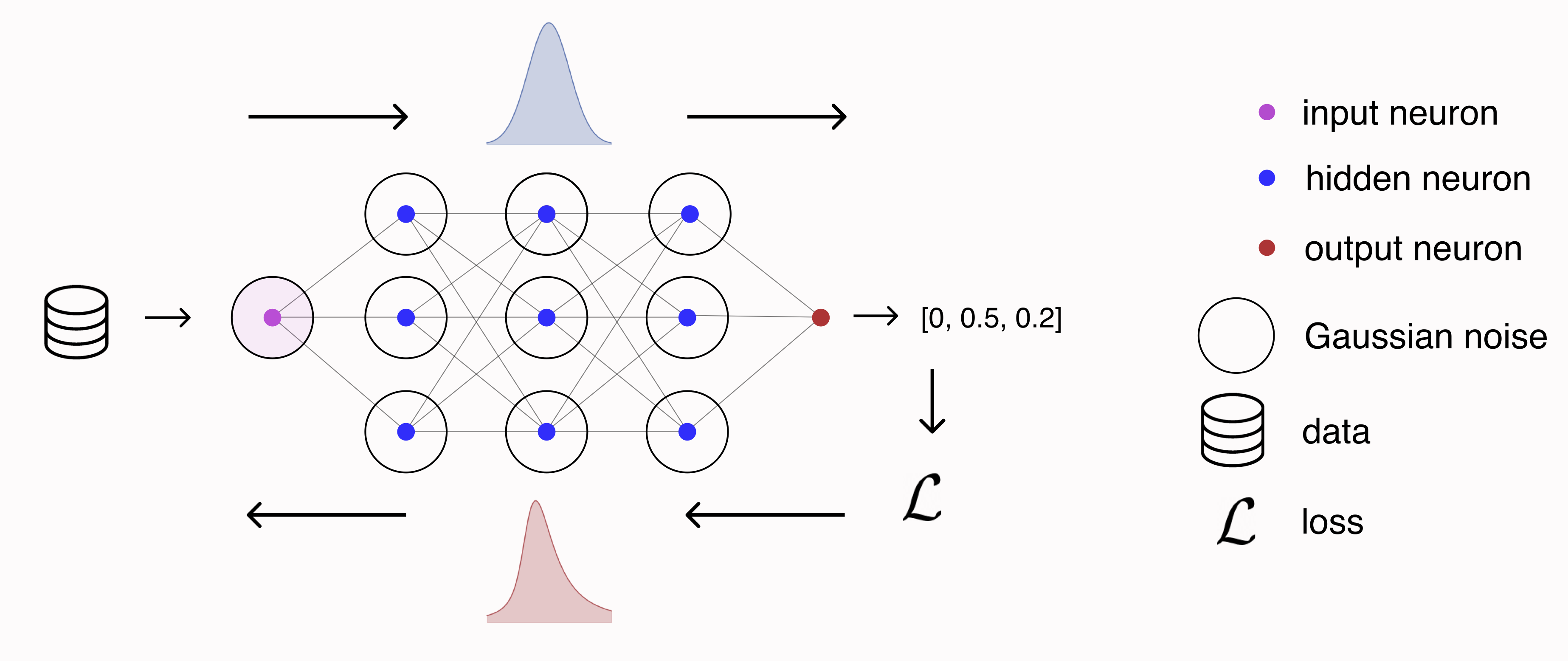}
   
    \caption{
    Illustration of the effect of GNIs \textit{added} to a network's activations. Each colored dot represents a neuron's activations.
    We add GNIs, represented as circles, to each layer's activations bar the output layer. 
    Perhaps counter-intuitively, though the forward pass experiences Gaussian noise, gradient updates in the backward pass experience heavy-tailed asymmetric noise.
    }
    \vspace{-10pt}
    \label{fig:illustration}
\end{figure}

Noise injections are a family of methods that involve adding or multiplying samples from a noise distribution to the weights and activations of a neural network during training.
The most commonly used distributions are Bernoulli distributions and Gaussian distributions \citep{Srivastava2014, Poole2014} and the noise is most often inserted at the level of network activations.

Though the regularisation conferred by Gaussian noise injections (GNIs) can be observed empirically, and there have been many studies on the benefits of noising data \citep{Bishop1995, Cohen2019, Webb1994}, the mechanisms by which these injections operate are not fully understood. 
Recently, the \emph{explicit} effect of GNIs, which is the added term to the loss function obtained when marginalising out the injected noise, has been characterised analytically \citep{Camuto2020_2}: it corresponds to a penalisation in the Fourier domain which improves model generalisation.

Here we extend this analysis and focus on the \textit{implicit} effect of GNIs. 
This is the effect of the remaining noise that has been marginalised out when studying the explicit effect. 
In particular we focus on the manner in which such noise alters the dynamics of Stochastic Gradient Descent (SGD) \citep{Wei2020, Zhang2019}. 
We show that the implicit effect is driven by an \emph{asymmetric heavy-tailed noise} on the SGD gradient updates, as illustrated in Figure~\ref{fig:illustration}.

To study the effect of this gradient noise, we model the dynamics of SGD for a network experiencing GNIs by a stochastic differential equation (SDE) driven by an \emph{asymmetric heavy-tailed} $\alpha$-stable noise.
We demonstrate that this model captures the dynamics of networks trained with GNIs
and we show that the stationary distribution of this process becomes arbitrarily distant from the so-called \textit{Gibbs measure}, whose modes exactly match the local minima of the loss function, as the gradient becomes increasingly heavy-tailed and asymmetric.
% \umut{update here-}  we prove that this $\alpha$-stable noise induces bias in SGD, in that network parameters will not converge to the global minimum of the network cost function.  
Heavy-tailed and asymmetric gradient noise thus degrades network performance and this suggests that models trained with the full effect of GNIs will underperform networks trained solely with the explicit effect. 
We confirm this experimentally for a variety of dense and convolutional networks.\footnote{See \url{https://github.com/alexander-camuto/asym-heavy-tails-bias-GNI} for all code.}
% \umut{update} This means that Langevin dynamics driven by \textit{asymmetric} heavy-tailed $\alpha$-stable noise are a good model for SGD operating with Gaussian noise injections \citep{Welling2011, Simsekli2019a, pmlr-v97-simsekli19a, Nguyen2019, Simsekli2020, Jastrzebski2017}.

% We show that such noise induces dynamics which preferentially land in wider minima in a neural network's loss landscape. 
% Wider minima have been shown, though this is a point of contention \citep{Dinh2017}, to induce networks with better generalisation properties \citep{Keskar2019OnMinima, Jastrzebski2017}. 

% To summarise, our contributions are: 
% \begin{itemize}
%    
%      \item We show empirically and analytically that the implicit effect induces asymmetric and heavy-tailed noise on gradient updates. 
%      \item By modelling this gradient noise as Langevin dynamics driven by $\alpha$-stable noise noise we are able to show that the implicit effect of GNIs induces gradient descent dynamics that preferentially land in wider minima but are inherently biased. 
% \end{itemize}

%%%%%%%%%%%%%%%%%%%%%%%%%%%%%%%%%%%%%%%%%%%%%%%%%%%%%%%%%%%%
\section{Background}
\label{sec:background}

%%%%%%%%%%%%%%%%%%%%%%%%%%%%%%%%%%%%%%%%%%%%%

%%%%%%%%%%%%%%%%%%%%%%%%%%%%%%%%%%%%%%%%%%%%%%%%%%
\textbf{Stable Distributions.} 
The Generalised Central Limit Theorem (GCLT) \citep{gnedenko1968limit} states that for a sequence of independent and identically distributed (i.i.d.) random variables whose distribution
has a power-law tail with index $0< \alpha <2$, 
the normalised sum converges to a heavy-tailed distribution called the $\alpha$-stable distribution ($\mathcal{S}_{\alpha}$) as the number of summands grows. 
%%%%%%%%%%%%%%%%%%%%%%%%%%%%%%%%%%%%%%%%%%%%%%%%%
% \textbf{Stable Distributions.}
An $\alpha$-stable distributed random variable $X$ is 
denoted by 
$X \sim \mathcal{S}_{\alpha}(\sigma,\theta,\mu)$, where $\alpha\in(0,2]$ is the \emph{tail-index}, $\theta\in[-1,1]$ is the \emph{skewness} parameter, $\sigma\geq 0$ is the \emph{scale} parameter, and $\mu\in\mathbb{R}$ is called the \emph{location} parameter.
% if $X$ has the characteristic function:
% \begin{align}\label{antisym}
% &\mathbb{E}[\exp(i\zeta X)] 
% \\
% &=
% \begin{cases}
% e^{-\sigma^{\alpha}|\zeta|^{\alpha}(1-i\theta\sign(\zeta)\tan(\frac{\pi\alpha}{2}))+i\mu\zeta},
% &\text{if $\alpha\neq 1$},
% \\
% e^{-\sigma|\zeta|(1+i\theta\frac{2}{\pi}\sign(\zeta)\log|\zeta|)+i\mu\zeta},
% &\text{if $\alpha=1$},
% \end{cases}
% \nonumber
% \end{align}
% for any $\zeta\in\mathbb{R}$.
The mean of $X$ coincides with $\mu$ if $\alpha>1$,
and otherwise the mean of $X$ is undefined.
In this work, we always assume $\mu=0$.
The parameter $\theta$ is a measure of asymmetry. We say that $X$ follows a \emph{symmetric} $\alpha$-stable distribution
denoted as
$\mathcal{S}\alpha\mathcal{S}(\sigma)=\mathcal{S}_{\alpha}(\sigma,0,0)$ if $\theta=0$ (and $\mu=0$).
% so that $X$ has the characteristic function:
% $\mathbb{E}[\exp(i\zeta X)] = \exp(-|\sigma\zeta|^\alpha)$, $\zeta\in\mathbb{R}$.
The parameter $\alpha\in(0,2]$ determines
the tail thickness of the distribution, 
and $\sigma>0$ measures the spread
of $X$ around its mode. 
% and $\mu$ is the location parameter,
% which is the mean of $X$ when $1<\alpha<2$.
% In particular, for $X\sim\mathcal{S}_{\alpha}(\sigma,\theta,\mu)$
% with $0<\alpha<2$, the left tail and right tail of $X$
% are described by the formulas:
% \begin{align}
% &\lim_{x\rightarrow\infty}x^{\alpha}\mathbb{P}(X>x)
% =\frac{1+\theta}{2}C_{\alpha}\sigma^{\alpha},
% \\
% &\lim_{x\rightarrow\infty}x^{\alpha}\mathbb{P}(X<-x)
% =\frac{1-\theta}{2}C_{\alpha}\sigma^{\alpha},
% \end{align}
% where $C_{\alpha}:=(1-\alpha)/(\Gamma(2-\alpha)\cos(\pi\alpha/2))$
% if $\alpha\neq 1$ and $C_{\alpha}:=2/\pi$ if $\alpha=1$,
% (cf.\ Property~1.2.15 in \cite{ST1994}).
Note that when $\alpha<2$, $\alpha$-stable distributions 
have heavy tails such that their moments
are finite only up to the order $\alpha$.  
% (cf.\ Property~1.2.16 in \cite{ST1994}).
% More precisely, let $X\sim\mathcal{S}_{\alpha}(\sigma,\theta,\mu)$
% with $0<\alpha<2$.
% Then $\mathbb{E}[|X|^{p}]<\infty$ for any $0<p<\alpha$
% and $\mathbb{E}[|X|^{p}]=\infty$
% for any $p\geq\alpha$,
% which implies infinite variance (cf.\ Property~1.2.16 in \cite{ST1994}).

The probability density function (p.d.f.) of an $\alpha$-stable random variable, $\alpha\in(0,2]$,
does not have a closed-form expression except for a few special cases.
When $\alpha=1$ and $\alpha=2$, the symmetric $\alpha$-stable distribution
reduces to the Cauchy and the Gaussian distributions, respectively,
(cf.\ Section~1.1. in \cite{ST1994}).
By their flexibility, such distributions can model many complex stochastic phenomena for which exact analytic forms are intractable \citep{Sarafrazi2019, Fiche2013}.
% Given this flexibility such distributions can model a gamut of complex stochastic phenomena for which analytic forms for the underlying distribution from which they are drawn would be difficult to find \citep{Sarafrazi2019, Fiche2013, Farsad2015}.
%%%%%%%%%%%%%%%%%%%%%%%%%%%%%%%%%%%%%%%%%%%%%%

\textbf{L\'{e}vy Processes.}
A L\'{e}vy process (motion) is a stochastic process with independent, stationary increments. 
% It represents the motion of a point whose successive displacements are random and independent, and statistically identical over different time intervals of the same length, which can be viewed as the continuous-time analogue of a random walk. 
Formally, $\v{L}_{t}$ is L\'{e}vy process if
\begin{enumerate}[label=(\roman*),noitemsep,topsep=0pt,leftmargin=*,align=left]
\item $\v{L}_{0}=0$ almost surely;
\item  For any $t_{0}<t_{1}<\cdots<t_{N}$, the increments $\v{L}_{t_{n}}-\v{L}_{t_{n-1}}$
are independent, $n=1,2,\ldots,N$;
\item  The difference $\v{L}_{t}-\v{L}_{s}$ and $\v{L}_{t-s}$
have the same distribution;
\item  $\v{L}_{t}$ is continuous in probability, i.e.
for any $\delta>0$ and $s\geq 0$, $\mathbb{P}(|\v{L}_{t}-\v{L}_{s}|>\delta)\rightarrow 0$
as $t\rightarrow s$.
\end{enumerate}
The $\alpha$-stable L\'{e}vy process is an important class of L\'{e}vy processes.
In particular, for $\alpha \in (0, 2]$, let $\v{L}_{t}^{\alpha,\theta}$ denote the $d$-dimensional $\alpha$-stable L\'{e}vy process with independent components, i.e. each component is an independent scalar $\alpha$-stable Levy motion \citep{duan2015}
such that $\v{L}_{t-s}^{\alpha,\theta}$ has the distribution 
$\mathcal{S}_{\alpha}((t-s)^{1/\alpha},\theta,0)$ for any $s<t$.

% Furthermore, \citet{Zhu2019} draw on these connections between SDEs and SGD to show that anisotropic noise on gradient updates induces better regularisation than isotropic noise due to the highly anisotropic nature of loss landscapes. 
% As we show, much of the implicit effect of Gaussian noise injections is counter-intuitively driven by anisotropic heavy-tailed noise on gradient updates. 

%%%%%%%%%%%%%%%%%%%%%%%%%%%%%%%%%%%%%%%%%%%%%
% \subsection{Stochastic Gradient Descent and Differential Equations}
% \utodo{this looks a bit ugly}
% \pagebreak
\textbf{Stochastic Gradient Descent and Differential Equations. }
Let $\mathcal{D}$ be a training dataset composed of data-label pairs of the form $(\v{x}, \v{y})$, and let $\v{w} \equiv \{\v{W}_1,\ldots,\v{W}_L\} \in \mathbb{R}^d$  be the $d$ parameters of an $L$ layer neural network in vector form. 
% \todo{MG: $W_i$ are matrices or vectors? Here we seem to treat $W_i$ as a vector but below it is a matrix?}
%\mtodo{The usual notation is that $W_i$ is a matrix}
When a neural network operates on input data $\v{x}$, we obtain the activations ${\v{h}} \equiv \{{\v{h}}_0, \ldots , {\v{h}}_{L-1} \}$, where $\v{h}_{0}=\v{x}$ and ${\v{h}} \in \mathbb{R}^{n_0 + \dots + n_L}$ where $n_i$ is the number of neurons in the $i^{\mathrm{th}}$ layer: we consider a non-linearity $\kappa: \mathbb{R} \to \mathbb{R}$, 
% \begin{align}
$\v{h}_{i}(\v{x})= \kappa(\v{W}_i \v{h}_{i-1}(\v{x}))$,
% \label{eq:nonoise_acts}
% \end{align}
where $\kappa$ is applied element-wise to each coordinate of its argument.
In supervised settings, our objective is to find the optimal parameters $\v{w}_*$ that minimise the negative log-likelihood $- \log p_{\v{w}}(\v{y}|\v{x})$ of the labels $\v{y}$, given the parameters $\v{w}$ and data $\v{x}$: 
\begin{align}
 & \v{w}_* = \argmin_{\v{w}} \mathcal{L}( \mathcal{D}; \v{w}),\nonumber
\\
& \mathcal{L}( \mathcal{D}; \v{w}):=- \mathbb{E}_{\v{x},\v{y} \sim \mathcal{D}}\left[ \log p_{\v{w}}(\v{y}|\v{x})\right]\,.
    \label{eq:marginal_likelihood}
\end{align}
SGD and its variants are the most prevalent optimisation routines that underpin the training of very large neural networks. 
Under SGD, we estimate equation~\eqref{eq:marginal_likelihood} by sampling a \emph{random} mini-batch of data-label pairs $\mathcal{B} \subset \mathcal{D}$,
\begin{align}\label{eq:marginal_likelihood_batched}
\mathcal{L}(\mathcal{B}; \v{w}) = -\mathbb{E}_{\v{x},\v{y} \sim \mathcal{B}} \left[\log p_{\v{w}}(\v{y}|\v{x})\right] \approx \mathcal{L}( \mathcal{D}; \v{w}).
\end{align}
% \umut{Isn't $\mathcal{L}(\mathcal{B}; \v{w}) = \mathcal{L}(\mathcal{B}; \v{w})$? The extra "SGD" can be confusing. And it's not a new definition so $:=$ is misleading.}
SGD optimises this equation and approximates $\v{w}_*$ using iterative parameter updates. 
At training step $k$
\begin{align}\label{eq:sgd_updates}
\v{w}_{k+1} = \v{w}_{k} - \eta \nabla \mathcal{L}(\mathcal{B}_{k+1}; \v{w}_k),
\end{align}
where $\eta$ is the step-size for updates (the network's learning rate) \citep{Robbins1951, Ruder2016}. 

Studying the dynamics of SGD allows us to understand the subtle effects that batching may have on neural network training. 
% In modern machine learning applications, where many datasets do not fit in memory and practitioners must batch data to train networks,  understanding these effects is of great importance.
The similarities between the SGD algorithm and Langevin diffusions \cite{roberts2002langevin} have inspired many studies modelling the dynamics of SGD using stochastic differential equations (SDEs) under different noise conditions \citep{Welling2011, Simsekli2019a,  Raginsky,Gao2018, Gao2020,Jastrzebski2017,pmlr-v70-li17f}. 
In this approach, one models the discrete SGD updates \eqref{eq:sgd_updates}, as the discretisation of a continuous-time stochastic process, making assumptions about the properties of the `noise' that drives this process \citep{Mandt2016, Jastrzebski2017}.
This noise stems from the stochasticity in approximating the `true' gradient over the dataset $\nabla \mathcal{L}(\mathcal{D}; \v{w}_k)$ with that of a mini-batch $\mathcal{B}$, $\nabla \mathcal{L}(\mathcal{B}; \v{w}_k)$. 
We denote this noise as, 
\begin{align}
U_{k+1}(\v{w}) :=  \nabla \mathcal{L}(\mathcal{D}; \v{w}_k)-\nabla_{\v{w}_{k}} \mathcal{L}(\mathcal{B}_{k+1}; \v{w}_k).
\label{eq:sgd_noise}
\end{align}
The most prevalent assumption is that the gradient noise admits a multi-variate Gaussian noise: 
%to model the gradient on the network parameters as as having multi-variate Gaussian noise: 
$U_k(\v{w}) \sim \mathcal{N}(0, \sigma^2\v{I})$. 
This is rationalised by the central limit theorem, as the sum of estimation errors in equation~\eqref{eq:sgd_noise} is approximately Gaussian for sufficiently large batches. 
Under this assumption, we can rewrite the SGD parameter update as:
\begin{align}
\v{w}_{k+1} = \v{w}_{k} - \eta
\nabla\mathcal{L}(\mathcal{D}; \v{w}_k) + \eta U_{k+1}(\v{w}).
\end{align}
Then, one obtains the following continuous-time SDE to approximate the gradient updates \citep{WellingTeh, Mandt2016, Jastrzebski2017}: 
\begin{align}\label{eq-brownian-sde}
d\v{w}_t = -\nabla \mathcal{L}(\mathcal{D};\v{w}_t)dt + \sqrt{\eta \sigma^2}d\v{B}_{t},
\end{align}
where $\v{B}_t$ is the Brownian motion 
in $\mathbb{R}^{d}$ and $\sigma$ is the assumed noise variance for $U$. 
% Brownian motion can be viewed as a special case of L\'{e}vy motion.
% In particular, $\v{L}_{t}^{\alpha,\theta}=\sqrt{2}\v{B}_{t}$ with $\alpha=2, \theta = 0$.
% for \textit{large} neural networks,
However, recent work suggests that the Gaussian assumption might not be always appropriate \cite{gsz2020heavy,hodgkinson2020multiplicative}, 
% \umut{Also mention \cite{favaro2020stable,martin2019traditional} and \cite{zhou2020towards,zhang2020adaptive}}. 
and connectedly, the gradient noise is observed to be heavy-tailed in different settings \citep{pmlr-v97-simsekli19a, zhou2020towards}.
% By applying the generalised central limit theorem to characterise this noise, 
Under this noise assumption, 
the corresponding SDE is such that $\v{B}_t$ in \eqref{eq-brownian-sde} is replaced with the \emph{symmetric} stable process $\v{L}_t^{\alpha,0}$ where
$\theta=0$:
 \begin{align}
 \label{eqn:sde_levysym}
d\v{w}_t = -\nabla \mathcal{L}(\mathcal{D};\v{w}_t)dt + \eta^{(\alpha-1)/\alpha} \sigma d\v{L}_t^{\alpha,0}. 
\end{align}
\textbf{Gaussian Noise Injections.} 
GNIs are regularisation methods that consist of injecting Gaussian noise to the network activations. More precisely, let $\v{\epsilon}$ be a collection of `noise vectors' 
% \umut{Define $\v{\epsilon}$ more clearly, it's a collection of random vectors in this space etc.} 
injected to the network activations at each layer \textit{except the final layer}:  $\v{\epsilon} \equiv \{{\v{\epsilon}}_0, \ldots , {\v{\epsilon}}_{L-1} \}$, where $\v{\epsilon}_i \in \mathbb{R}^{n_i}$, $\v{\epsilon} \in \mathbb{R}^{n_0 + \dots + n_{L-1}}$ and
$n_i$ is the number of neurons in the $i^{\mathrm{th}}$ layer.
% When injecting noise, the value of next layer's activations depends on the noised value of the previous layer.
We have two values for an activation: the soon-to-be noised value $\widehat{\v{h}}_{i}$,  and the subsequently noised value $\widetilde{\v{h}}_{i}$. 
% \umut{the symbols look a bit similar, is there way to make them bigger?}. 
For a multi-layer perceptron (MLP),
\begin{align}
\widehat{\v{h}}_{i}(\v{x})= \kappa\left(\v{W}_{i}\widetilde{\v{h}}_{i-1}(\v{x})\right)\,, \qquad  
\widetilde{\v{h}}_{i} = \widehat{\v{h}}_{i} \circ \v{\epsilon}_{i}\,, 
\label{eq:noise_recursion}
\end{align}
where $\circ$ is some element-wise operation (e.g., addition or multiplication).
For additive GNIs typically, $\v{\epsilon}_i \sim \mathcal{N}\left(0,\sigma_i^2 \v{I}\right)$, and for multiplicative GNIs we often use $\v{\epsilon}_i \sim \mathcal{N}\left(1,\sigma_i^2 \v{I}\right)$, 
% Adding or multiplying isotropic Gaussian noise to activations gives: 
% \umut{$\v{\epsilon_i}$ is a vector right? Then its distr. should be a mv. Gaussian}
% \begin{align}
% \widetilde{\v{h}}_i(\v{x}) &= \widehat{\v{h}}_i(\v{x}) + \v{\epsilon}_i, 
% \qquad \v{\epsilon}_i \sim \mathcal{N}\left(0,\sigma_i^2 \v{I}\right), \label{eq:noise_add} \\
% \widetilde{\v{h}}_i(\v{x}) &= \widehat{\v{h}}_i(\v{x}) \times 
% \v{\epsilon}_i, 
% \qquad \v{\epsilon}_i \sim \mathcal{N}\left(1,\sigma_i^2 \v{I}\right),
% \label{eq:noise_mult}
% \end{align}
where $\mathcal{N}$ is the Gaussian distribution. The `accumulated' noise at each layer, induced by the noise injected to the layer \textit{and} in previous layers is
\begin{align} \label{eq:accum_noise}
    \mathcal{E}_i(\v{x}; \v{w}, \v{\epsilon}) = \widetilde{\v{h}}_i(\v{x}) - \v{h}_i(\v{x}).
\end{align}
Given that GNIs are commonly used as regularisation methods \cite{Camuto2020_2, dieng2018noisin, Srivastava2014, Poole2014, Kingma, Bishop1995}, our goal is to understand better the mechanisms by which they affect neural networks.

% and generally this term can be expressed as, if we denote each layer's Jacobian as
% $\v{J}_{k}(\v{x}) \in \mathbb{R}^{N_L \times N_k}, \quad \v{J}_{k}(\v{x})_{i,j} = \partial \v{h}_{L}(\v{x})_i / \partial \v{h}_{k}(\v{x})_j$: 

% \begin{theorem}[\citet{Camuto2020_2}]
% Consider an $L$ layer neural network experiencing GNIs $\v{\epsilon}_{i}$ at each layer $i \in \{0, \dots, L - 1\}$.
% We assume the third derivative of  $\mathcal{L}(\v{x})$ w.r.t $\v{h}_{L}(\v{x})$ is finite, which is the case for most loss functions.
% We can marginalise out the injected noise $\v{\epsilon}$ to obtain an added regularizer: 
% \begin{align}
% &\expect_{\v{\epsilon} \sim p(\v{\epsilon})} \left[ \Delta \mathcal{L}(\v{\epsilon}) \right] 
% \nonumber
% \\
% &= \frac{1}{2}\mathbb{E}_{\v{x} \sim \mathcal{B}} \left[ \sum_{k=0}^{L-1}\sigma^2_k \mathrm{Tr}(\v{J}^T_{k}(\v{x})
%     \v{H}_{L}(\v{x})\v{J}_{k}(\v{x}))  \right] +   \mathcal{O}(\v{\epsilon}), 
%     \label{eqn:deltaL}
% \end{align}
% where $\mathcal{L}(\v{x})$ is the loss for a datapoint $\v{x}$, $\v{H}_{L}(\v{x})$ is $\nabla^2 \mathcal{L}\rvert_{\v{h}_{L}(\v{x})} \in \mathbb{R}^{N_L
% \times N_L}$ ($N_L$ is the number of output neurons), and $\mathcal{O}(\v{\epsilon})$ represents higher order terms in $\v{\v{\epsilon}}$  that disappear in the limit of small variance noise. 
% \label{prop:explicit_reg}
% \end{theorem}

%%%%%%%%%%%%%%%%%%%%%%%%%%%%%%%%%%%%%%%%%%%%%%%%%
\section{The Implicit Effect of GNIs}\label{sec:impliciteffect}

%%%%%%%%%%%%%%%%%%%%%%%%%%%%%%%%%%%%%%%%%
% \paragraph{The Explicit Effect of GNIs}

% We can express the effect of GNIs on the cost function as a term  $\Delta\mathcal{L}(\mathcal{B};\v{w},\v{\epsilon})$ that is added to the loss, given noise injections $\v{\epsilon}$ \citep{Camuto2020_2},
Recently, \citet{Camuto2020_2} showed that the effect of GNIs on the cost function can be expressed as a term $\Delta\mathcal{L}(\mathcal{B};\v{w},\v{\epsilon})$ that is added to the loss, i.e.,
\begin{align}
\widetilde{\mathcal{L}}(\mathcal{B};\v{w}, \v{\epsilon}) :=  \mathcal{L}(\mathcal{B}; \v{w}) + \Delta\mathcal{L}(\mathcal{B};\v{w},\v{\epsilon})\,,
\end{align}
where $\widetilde{\mathcal{L}}$ is the modified loss that SGD ultimately aims to minimise. The term $\Delta \mathcal{L}$ can be further broken down into explicit and implicit effects, as described in Section~\ref{intro}. 

% We  take the recently proposed approach of \citet{Camuto2020_2} and define the explicit effect as the additional term obtained on the loss when we marginalise out the noise we have injected.
% It offers a consistently positive objective for gradient descent to optimise and we denote it as 
We build on the approach of \citet{Camuto2020_2} and define the explicit effect as the additional term obtained on the loss when we marginalise out the noise we have injected.
It offers a consistently positive objective for gradient descent to optimise and we denote it as 
% $\expect_{\v{\epsilon} \sim p(\v{\epsilon})} (\Delta \mathcal{L}(\mathcal{B};\v{w},\v{\epsilon}))$ or 
$\expect_{\v{\epsilon}} (\Delta \mathcal{L}(\mathcal{B};\v{w},\v{\epsilon}))$.
% as shorthand.
% This term has been shown to regulate the Fourier spectrum of the function learnt by a network, and by doing so can induce better network generalisation \citep{Camuto2020_2}.
The implicit effect is then the remainder of the terms marginalised out in the explicit effect:
\begin{align}\label{defn:implicit}
E_\mathcal{L}(\mathcal{B};\v{w},\v{\epsilon}) := \Delta\mathcal{L}(\mathcal{B};\v{w},\v{\epsilon}) - \expect_{\v{\epsilon}} (\Delta \mathcal{L}(\mathcal{B};\v{w},\v{\epsilon})).
\end{align}
% which is the difference between the added penalisation $\Delta \mathcal{L}(\mathcal{B};\v{w},\v{\epsilon})$ and the explicit effect.
While the explicit effect focuses on the consistent and \textit{non-stochastic} regularisation induced by GNIs, the implicit effect instead studies the effect of the \textit{inherent stochasticity} of GNIs. 
This term does not offer a consistent objective for SGD to minimise. 
Rather we show that it affects neural network training by way of the  \textit{heavy-tailed and skewed} noise it induces on gradient updates. 

Subsequently, we first characterise the tail properties of the noise accumulated during the forward pass. 
As this accumulated noise defines the implicit effect we can then use this result to show that gradients induced by the implicit effect are heavy-tailed and skewed and apt to be modelled by heavy-tailed and asymmetric $\alpha$-stable noise. 
% Finally, we provide novel theoretical results that show that $\alpha$-stable noise induces bias in SGD; implying that the implicit effect of GNIs potentially degrades the performance of the network.  
% Finally, having found that the implicit effect  
% characterise the tail properties of the gradient updates induced 
% \umut{A short roadplan.}
% To understand the implicit effect we need to understand the characteristics of the gradient noise it induces.

\textbf{The Properties of the Accumulated Noise.} 
Before considering the properties of the implicit effect gradients, we first need to study the noise that is accumulated during the \textit{forward pass} of a neural network experiencing GNIs. 
Here we use asymptotic analysis to bound the moments of this noise, allowing us to characterise the tail properties of gradients using the broad class of \textit{sub-Weibull} distributions that include a range of heavy-tailed distributions  \citep{Vladimirova2019UnderstandingLevel, Vladimirova2020, Kuchibhotla2018}. 
\begin{definition}{(Asymptotic order)}
A positive sequence $a_m$ is of the same order as another positive sequence $b_m$ ($a_m \lesssim b_m$) if $\exists C>0$ such that $\frac{a_m}{b_m} \leq C ~ \forall m \in \mathbb{N}$.
$a_m$ is of the same order of magnitude as $b_m$ ($a_m \asymp b_m$, i.e. `asymptotically equivalent') if there exist some $c,C>0$ such that: $c \leq \frac{a_m}{b_m} \leq C$ for any $m \in \mathbb{N}$.
\end{definition}

\begin{definition}{(Sub-Weibull distributions),
:}
We say that a random variable $X$ is sub-Weibull \citep{Vladimirova2020, Kuchibhotla2018} with tail parameter $r$
if $\|X\|_m \lesssim m^r, r >0$, where $\|X\|_m := \expect[|X|^m]^{\frac{1}{m}}$. In this case, we write $X \sim \mathrm{subW}(r)$. Such distributions satisfy the tail bound
% \begin{equation}
    $\mathbb{P}(|X| > x) \leq 2 e^{-( \frac{x}{C})^\frac{1}{r}}$,
% \end{equation}
where $C>0 $ is some constant.  
Note that for $r=\frac{1}{2}$ and $r=1$ we recover the sub-Gaussian and sub-exponential families and that if $X \sim \mathrm{subW}(r')$, then $X \sim \mathrm{subW}(r)$ for $r>r'$.
%for $r' < r$, 
%($\mathrm{subW}(r') \subset \mathrm{subW}(r)$). 
\end{definition}

% \umut{Some explanation is needed for subweibull.}
As $r$ increases the tail distribution becomes heavier-tailed.
% though at a different rate to that of an $\alpha$-stable distribution \umut{this sentence is not very clear. }.  
To simplify our analysis we consider activation functions $\phi$ that obey the extended envelope property. 
% Commonly used functions such as $\mathrm{ReLU}$ satisfy this property. 

\begin{definition}[Extended Envelope Property \citep{Vladimirova2019UnderstandingLevel}:] \label{defn:extended-envelope}
A non-linear function $\kappa:\mathbb{R} \to \mathbb{R}$ is said to obey the \textit{extended envelope property} if $\exists~c_1,c_2 \geq 0, ~ d_1, d_2 \geq 0 $ such that: 
\begin{itemize}
    \item $|\kappa(x)| \geq c_1 +d_1|x|$, for any $x\in \mathbb{R}^+$ or $x\in \mathbb{R}^-$;
    \item $|\kappa(x)| \leq c_2 +d_2|x|$, for any $x\in \mathbb{R}$.
\end{itemize}
\end{definition}

Activation functions that obey this property, such as $\mathrm{ReLU}$, are broadly moment preserving \citep{Vladimirova2020}. 
Using this property, we can characterise the moments for the  $l^\mathrm{th}$ noised activation in a layer $i$, $\widetilde{h}_{i,l}(\v{x})$. 

\begin{lemma}\label{prop:accum_tail_properties}
    For feed-forward neural networks with an activation function $\phi$ that obeys the extended envelope property, the noised activations at each layer $i < L-1$, resulting from additive-GNIs $\v{\epsilon}$ obey
    \begin{align}
        \left\Vert\widetilde{h}_{i,l}(\v{x})\right\Vert_m \lesssim \sqrt{m},  \quad \text{for any $m \geq 1$}; \ l= 1,\dots,n_i\,. \nonumber
    \end{align}
    For multiplicative-GNIs we have,
    \begin{align}
        \left\Vert\widetilde{h}_{i,l}(\v{x})\right\Vert_m \lesssim m^{\frac{i+1}{2}},  \quad \text{for any $m \geq 1$}; \ l= 1,\dots,n_i\,, \nonumber
    \end{align}
where $n_i$ is the dimensionality of the $i^\mathrm{th}$ layer.
\end{lemma}
Lemma \ref{prop:accum_tail_properties} shows that when the injected noise is additive,
%For the additive case, 
the noised activations at each layer will have \textit{(sub)-Gaussian tails}. 
By equation~\eqref{eq:accum_noise}, the accumulated noise $\mathcal{E}_{i}(\v{x}; \v{w}, \v{\epsilon})$ will also have sub-Gaussian tails as the non-noised activations $\v{h}(\v{x})$ are deterministic and do not affect the asymptotic relationships of moments. 
See Figure~\ref{fig:forward_pass} of the Appendix for a demonstration  that the activations experience Gaussian-like noise for additive-GNIs.
For multiplicative-GNIs the noise at each layer, except the input layer which experiences Gaussian noise, behaves with a sub-Weibull tail. 
We demonstrate this behaviour in Figure~\ref{fig:forward_pass_mult} of the Appendix. 

We can now study the properties of the gradient noise induced by the implicit effect by taking the gradients of this forward pass noise. 
% \umut{it's not clear what is the significance of this part. Why is this outcome useful?}
% Also note that because of the feed-forward nature of neural networks, the noised activations at different layers are likely to be correlated, i.e. $\widetilde{\v{h}}_i(\v{x})$ is likely to be highly correlated with $\widetilde{\v{h}}_{i-1}(\v{x})$. \alex{insert figures ? or table to demonstrate ? or is it just kind of obvious ?}

%%%%%%%%%%%%%%%%%%%%%%%%%%%%%%%%%%%%%%%%%%%%%
\textbf{Kurtosis of The Gradient Noise.}
We characterise the gradient noise corresponding to $W_{i,l,j}$, the weight that maps from neuron $l$ in layer $i-1$ to neuron $j$ in layer $i$.
\begin{theorem}\label{thm:gn_tail_properties}
    Consider a feed-forward neural network with an activation function $\phi$ that obeys the extended envelope property and a cross-entropy or mean-squared-error cost (see  Appendix~\ref{sec:costfunctions}). The gradient noise from additive-GNIs $\v{\epsilon}$, has zero mean and has moments that obey for a pair $(\v{x}, \v{y})$:
    \begin{align*}
         &\left\Vert\frac{\partial E_{\mathcal{L}}((\v{x}, \v{y}); \v{w}, \v{\epsilon})}{\partial W_{i,l,j}}\right\Vert_m
        \lesssim m,  \quad \text{for any $m \geq 1$},
    \end{align*}
    where $E_{\mathcal{L}}((\v{x}, \v{y}); \v{w}, \v{\epsilon})$ is defined in \eqref{defn:implicit}.
    For multiplicative GNIs, we have 
    \begin{align*}
         &\left\Vert\frac{\partial E_{\mathcal{L}}((\v{x}, \v{y}); \v{w}, \v{\epsilon})}{\partial W_{i,l,j}}\right\Vert_m
        \lesssim m^\frac{L+i}{2},  \quad\text{for any $m \geq 1$}, \\
        &i = 1, \dots, L;\ l=1,\dots, n_{i-1}; \ j=1,\dots,n_i. 
        \nonumber
    \end{align*}
\end{theorem}

For the additive noise, these bounds infer that the gradient noise at each layer will have \textit{sub-exponential tails}.
For the multiplicative case, gradient noise will be \textit{sub-Weibull}, with a tail parameter that increases with $i$ the layer index.  
Unlike the forward pass, which experienced noise bounded in its tails by a Gaussian, the backward pass experiences noise that is bounded in its tails by heavy-tailed Weibull distributions with tail parameter $r\geq 1$. 
% \mtodo{The theorem above provides an upper bound on the tail, is it possible to construct a simple example where the bound is tight? If so, that would strengthen the result.}
% \alex{The asymptotic bound becomes an asymptotic equivalence in the case of 1-D NN, i.e. 1D data + 1 neuron per hidden layer -- do you think this is worth mentioning ? as its a pretty unrealistic scenario. }

\begin{remark}
The bounds defined by Lemma~\ref{prop:accum_tail_properties} and Theorem~\ref{thm:gn_tail_properties} become an asymptotic equivalence ($\asymp$) in the case of 1-D data and 1-neuron-wide neural networks, i.e. the bounds are maximally tight in this case.
% This can be shown trivially by refactoring the proofs of the Lemma and Theorem to be for a single variable.
% {\color{blue}I don't understand this sentence. Also do you mean Lemma~\ref{prop:accum_tail_properties} and Theorem~\ref{thm:gn_tail_properties}?}
\end{remark}

% Note that the tail bounds described will not necessarily be attained and will depend on the activation function chosen that obeys the extended envelope property. 
% We can refine Theorem~\ref{thm:gn_tail_properties}
% for ReLU activation functions for instance. 

% \begin{theorem}\label{thm:gn_tail_properties_relu}
%     Consider a feed-forward neural network with an activation function $\phi=\mathrm{ReLU}$ and a cross-entropy or mean-square-error cost. We can show that the gradient noise from GNIs $\v{\epsilon}$, is 0 mean and has moments that obey for a data-label pair $(\v{x}, \v{y})$:
%     \begin{align*}
%          \left\Vert\frac{\partial E_{\mathcal{L}}((\v{x}, \v{y}); \v{w}, \v{\epsilon})}{\partial W_{i,l,j}}\right\Vert_m
%         \lesssim k,  \ \forall m \geq 1 \\
%         i = 1, \dots, L;\ l=1,\dots, d_{i-1}; \ m=1,\dots,n_i 
%         \nonumber
%     \end{align*}
% \end{theorem}
% \alex{Can merge Theorems 3 and 4 into 1}

% The bound now is much tighter in the case and the noise is \textit{at most exponential in its tails}.

Our result applies to the gradient for a single pair $(\v{x}, \v{y})$. 
During SGD, we take the mean gradient across a batch $\mathcal{B}$ of size $B$.
To study the tail distribution of this mean gradient, we restate in a simplified manner \citet{Kuchibhotla2018}'s generalisation of the Bernstein inequality for zero-mean sub-Weibull random variables in Theorem~\ref{thm:bersteinweibull}.
% of the Appendix. 

\begin{theorem}[Theorem~3.1 in \citet{Kuchibhotla2018}]\label{thm:bersteinweibull}
Let $X_1,\dots,X_B$ be independent mean-zero sub-Weibull random variables with a shared tail parameter $p \geq 1$. 
Then, for every $x\geq0$, we have
\begin{align*}
&\mathbb{P}\left\{ \left|(1/B)\sum\nolimits_{i=1}^B X_i \right| \geq x \right\} \\ 
&\leq 2\exp\left[-\min \left(\frac{Bx^2}{C \sum^B_{i=1}\|X_i\|^2_{\psi 2}}, \frac{ Bx^\frac{1}{\theta}}{L\max_i \|X_i\|_{\psi \frac{1}{p}}} \right) \right],
\end{align*}
where $C, L>0$ are constants that depend on the tail parameter and where 
\begin{equation}
\|X\|_{\psi \frac{1}{p}}= \inf \left\{\nu\geq 0: \expect \left[(|X|/\nu)^\frac{1}{p}%/x
\leq 1\right]\right\} 
\end{equation}
is the `sub-Weibull norm'. 
\end{theorem}

This theorem states that the the tails of the mean of zero-mean i.i.d.\ %iid 
sub-Weibull random variables are produced by a single variable, say $X_i$, with the maximal sub-Weibull norm $\|X_i\|_{\psi \frac{1}{p}} = \inf \{\nu\geq 0: \expect [(|X_i|/\nu)^\frac{1}{p}%/x
\leq 1]\}$, i.e., the one with the heaviest tails. 
% \mtodo{not clear to me, we can consider revising this sentence. $r$ is an intermediate (dummy) variable in the theorem, we can refer to the sub-Weibull norm instead?} \alex{Was a typo} 
% Closer to the origin however the mean behaves like a Gaussian.
Assuming gradients are independent across data points, we can use this inequality to bound the tail probability for the sum of our zero-mean gradients.
If a single gradient is sufficiently heavy-tailed, then the mean across the batch will also be heavy-tailed.

In Figure~\ref{fig:backward_pass} we show that the implicit effect gradient, averaged over \textit{the entire dataset} $\mathcal{D}$ (i.e., the largest batch-size possible), is heavy-tailed, unlike the forward pass.
To calculate these gradients, we estimate the explicit regulariser in equation \eqref{defn:implicit} using Monte Carlo sampling, $\expect_{\v{\epsilon}} (\Delta \mathcal{L}(\mathcal{D};\v{w},\v{\epsilon})) \approx \frac{1}{M}\sum_{m=0}^M \Delta \mathcal{L}(\mathcal{D};\v{w},\v{\epsilon}_m)$, similarly to \cite{Wei2020,Camuto2020_2}. 
We show these results for multiplicative-GNIs in Figure~\ref{fig:backward_pass_mult} of the Appendix. In this setting as well, the backward pass experiences \textit{heavy-tailed} noise from GNIs. 

% \begin{figure}[h]
%     \centering
%     \includegraphics[width=0.25\textwidth]{figures/skewcorrplot.png}
%     \caption{Skew vs correlation $r$ of the product of two Gaussian with identical mean $\mu$.}
%     \label{fig:gaussian_skew_corr}
% \end{figure}

%%%%%%%%%%%%%%%%%%%%%%%%%%%%%%%%%%%%%%%%%%%%%%%%%%%%%%%%%%%%
\textbf{Skewness of The Gradient Noise.} 
In the proof of Theorem~\ref{thm:gn_tail_properties} we decompose $\partial E_{\mathcal{L}}(\cdot)/\partial W_{i,l,j}$ as  $(\partial E_{\mathcal{L}}(\cdot)/ \partial \widetilde{h}_{i,j}) \cdot (\partial \widetilde{h}_{i,j} / \partial W_{i,l,j})$, where $\widetilde{h}_{i,j}$ is the (noised) activation of the $j^\mathrm{th}$ neuron in the $i^\mathrm{th}$ layer. 
Both these derivatives are likely to be skewed due to the asymmetry of the activation functions and their gradients. 
% \mtodo{the variable $\kappa$ is used before in the envelope property, should we say $\phi'$ instead of $\kappa'$?}
Ignoring potential correlations between variables, we observe that product of two zero-mean skewed independent variables $X$ and $Y$ is also skewed $
    \left|\mathrm{skew}(XY)\right| =  \left|\expect\left[X^3\right]\expect\left[Y^3\right]\right| > 0$.
The skewness of the derivatives in the backward pass will induce \textit{skewed} gradient noise, as seen in Figure~\ref{fig:backward_pass}. 
Though correlations between gradients could also cause skewness, we show that this is not the case in Appendix~\ref{app:skew}.

%%%%%%%%%%%%%%%%%%%%%%%%%%%%%%%%%%%%%%%%%%%%%

%%%%%%%%%%%%%%%%%%%%%%%%%%%%%%%%%%%%%%%%%%%%%%%%%%%%
\section{An SDE Model for SGD with GNIs.}\label{sec:SDE}

In this section, we will analyse the effects of the skewed and heavy-tailed noise on the SGD dynamics.
Recall that the modified loss function by the GNIs is the sum of the explicit regulariser and the original loss over the dataset $\mathcal{D}$,
\begin{align*}
 \expect_{\v{\epsilon}} \left(\widetilde{\mathcal{L}}(\mathcal{D}; \v{w}, \v{\epsilon})\right) = \mathcal{L}(\mathcal{D}; \v{w}) + \expect_{\v{\epsilon}} \left(\Delta \mathcal{L}(\mathcal{D};\v{w},\v{\epsilon})\right),
\end{align*}
and the SGD recursion takes the following form:
\begin{align}
\label{eqn:sgd_true}
\v{w}_{k+1} = \v{w}_{k} - \eta
\nabla \expect_{\v{\epsilon}} \left(\widetilde{\mathcal{L}}(\mathcal{D}; \v{w}_k, \v{\epsilon})\right) + \eta U_k(\v{w}),
\end{align}
where $U_{k+1}(\v{w})$ is given as follows:
\begin{align}
\label{eqn:gni_noise} &  \nabla \left[\expect_{\v{\epsilon}} \left(\widetilde{\mathcal{L}}(\mathcal{D}; \v{w}_k, \v{\epsilon})\right)\right] 
\\
\nonumber &\phantom{as}-\nabla\left[ \expect_{\v{\epsilon}}\left(\widetilde{\mathcal{L}}(\mathcal{B}_{k+1}; \v{w}_k, \v{\epsilon})\right) +  E_\mathcal{L}(\mathcal{B}_{k+1};\v{w}_k,\v{\epsilon})\right]. 
\end{align}
To ease the notation, let us denote the modified loss function by $f(\v{w}) := \expect_{\v{\epsilon}} \left(\widetilde{\mathcal{L}}(\mathcal{D}; \v{w}, \v{\epsilon})\right)$.
\begin{remark}
Note that when the gradients are computed over the whole dataset $\mathcal{D}$, the gradient noise solely stems from the implicit effect, $U_k(\v{w}) =  -\nabla\left[ E_\mathcal{L}(\mathcal{D};\v{w}_k,\v{\epsilon})\right]$.
\label{rem:minibatch_implicit}
\end{remark}

\begin{figure}[t!]
    \centering
     \includegraphics[width=0.48\textwidth]{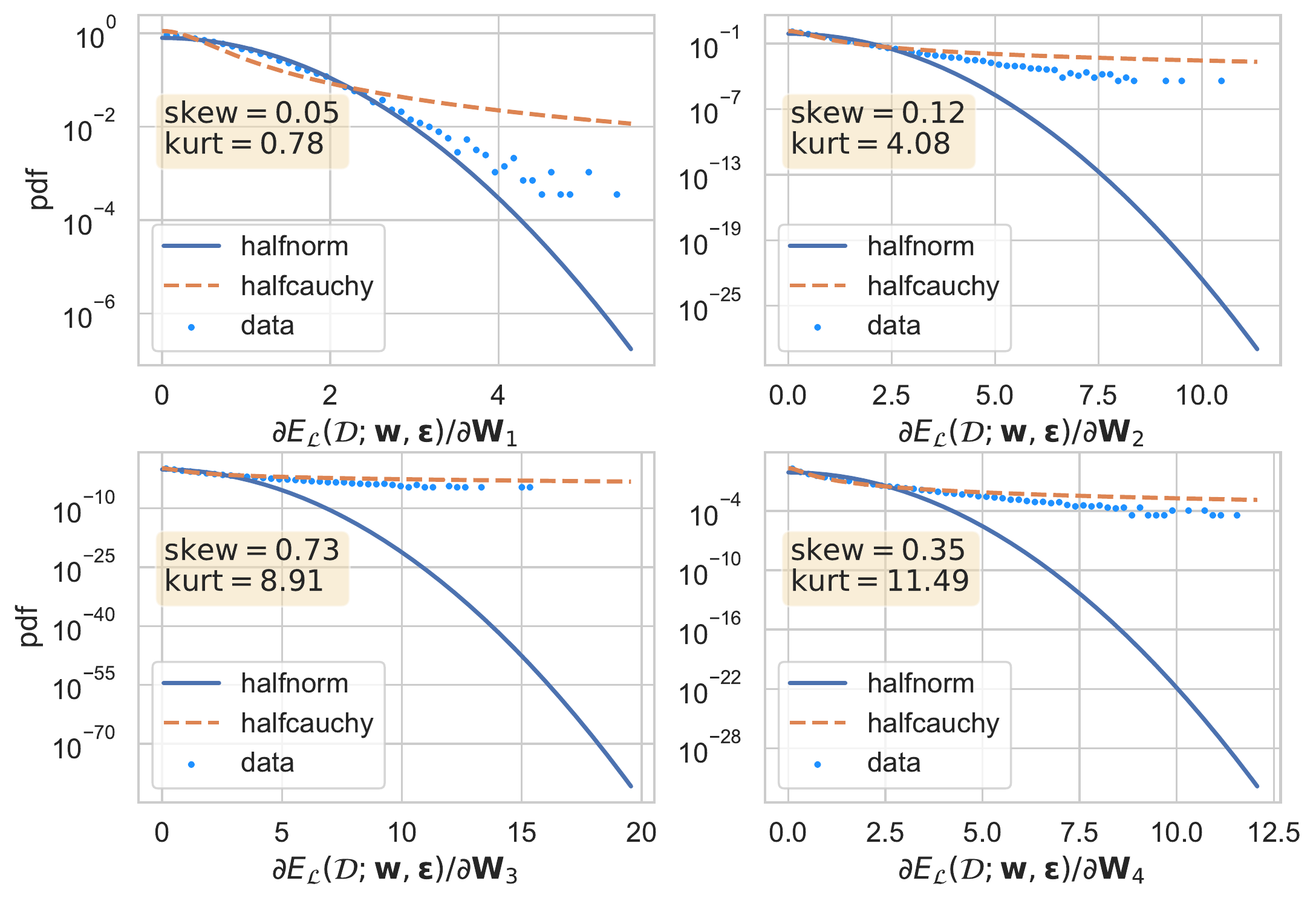}
     \vspace{-15pt}
    \caption{We measure the skewness and kurtosis \textit{at initialisation} of the gradients noise accrued on networks weights during the backward pass for additive-GNIs. The model is a 4-layer-256-unit-wide MLP trained to regress $\lambda(x) = \sum_i \sin(2\pi q_ix+\phi(i))$ with $q_i  \in(5,10,\dots,45,50), x \in \mathbb{R}$. 
    We plot the probability density function (p.d.f.) of positive samples, comparing against half-normal and half-Cauchy distributions.
    Each blue point represents the mean gradient noise over the entire dataset $\mathcal{D}$ of an individual weight in a layer $i$. 
    This gradient noise is \textit{skewed and heavy-tailed}, with a p.d.f.\ that is more Cauchy-like than Gaussian. 
    }
    \label{fig:backward_pass}
\end{figure}

Recent studies have proven that heavy-tailed behaviour can already emerge in stochastic optimisation (without GNIs) \cite{hodgkinson2020multiplicative,gsz2020heavy,gh2020}, which can further result in an overall heavy-tailed behaviour in the gradient noise, as empirically reported in \cite{Simsekli2019a,zhang2020adaptive,zhou2020towards}. Accordingly, \cite{Simsekli2019a,zhou2020towards} proposed modelling the gradient noise by using a centred \emph{symmetric} $\alpha$-stable noise, which then paved the way for modelling the SGD dynamics by using an SDE driven by a symmetric $\alpha$-stable process (see \eqref{eqn:sde_levysym}). \looseness=-1

We take a similar route for modelling the trajectories of SGD with GNIs; however, due to the skewness arising from the GNIs, the symmetric noise assumption is not appropriate for our purposes. Hence, we propose modelling the skewed gradient noise \eqref{eqn:gni_noise} by using an \emph{asymmetric} $\alpha$-stable noise, which aims at modelling both the heavy-tailed behaviour and the asymmetries at the same time.  
In particular, we consider an SDE driven by an %a
asymmetric stable process and its Euler discretisation as follows:
\begin{align}\label{unmodified}
&d\v{w}_{t}= -\nabla f(\v{w}_{t}) dt
+\varepsilon d\v{L}_{t}^{\alpha,\theta}, \\
&\v{w}_{k+1}= \v{w}_k - \eta_{k+1} \nabla f(\v{w}_k) + \varepsilon \eta_{k+1}^{1/\alpha} \Delta \v{L}^{\alpha,\theta}_{k+1}, \label{eqn:sgd_approx}
\end{align}
where $\theta=(\theta_{i},1\leq i\leq d)$ is the $d$-dimensional skewness and each coordinate can have
its idiosyncratic skewness $\theta_{i}$. 
$\v{L}_{t}^{\alpha,\theta}=(L_{t}^{\alpha,\theta_{1}},\ldots,L_{t}^{\alpha,\theta_{d}})$ is a $d$-dimensional asymmetric $\alpha$-stable
L\'{e}vy process with independent components, 
and $\varepsilon$ encapsulates all the scaling parameters. Furthermore, $(\eta_k)_k$ denotes the sequence of step-sizes, which can be taken as constant or decreasing, and finally $(\Delta \v{L}^{\alpha,\theta}_{k})_k$ is a sequence of i.i.d.\ random vectors where each component of $\Delta \v{L}^{\alpha,\theta}_{k}$ is i.i.d.\ with $\mathcal{S}_\alpha(1,\theta_i,0)$. We then propose the discretised process \eqref{eqn:sgd_approx} as a proxy to the original recursion \eqref{eqn:sgd_true} and we will directly analyse the theoretical properties of \eqref{eqn:sgd_approx}. Note that our approach strictly extends \cite{Simsekli2019a}, which appears as a special case when $\theta=0$. 

Before deriving our theoretical results, we first  verify empirically that the proxy dynamics \eqref{eqn:sgd_approx} are indeed a good model for representing \eqref{eqn:sgd_true}.
We ascertain that $\Delta \v{L}^{\alpha,\theta}_{k}$ is sufficiently general in the sense that it can capture the gradient noise induced by the implicit effect even when there is no batching noise (when the batch size approaches the size of the dataset for example, see Remark~\ref{rem:minibatch_implicit}).
As a first line of evidence, in Figure~\ref{fig:scatter_alpha_stable} of the Appendix
we model $\nabla E_\mathcal{L}(\cdot)$ as being drawn from a univariate $\mathcal{S}_{\alpha}$.
The equivalent $\mathcal{S}_{\alpha}$ distributions are skewed ($|\theta| > 0$) and heavy-tailed ($\alpha < 2$), demonstrating that $\v{L}_{t}^{\alpha,\theta}$ captures the core properties of the implicit effect gradients highlighted in Section~\ref{sec:impliciteffect}. 
To illustrate this more clearly, 
in Figure~\ref{fig:weibullastable} we use symmetric sub-Weibull distributions with $r=0.8,1,2$, and fit $\mathcal{S}_\alpha$ and Gaussians ($\mathcal{N}$) using maximum likelihood (MLE) from $10^4$ samples.
We plot the MLE densities, and clearly MLE $\mathcal{S}_\alpha$ ($\alpha<2$) better model the tails of the sub-Weibulls than a Gaussian distribution, even for $r=0.8$ which is close to Gaussian tails ($r=0.5$).
The $\mathcal{S}_\alpha$ modelling of the tails improves as $r$ increases. This further illustrates the appropriateness of our noise model.

\begin{figure}[h]
\centering
    \includegraphics[width=0.4 \textwidth]{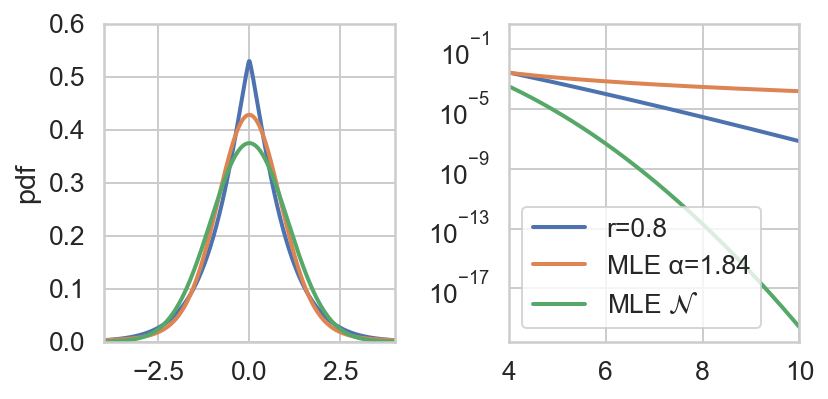}
    \includegraphics[width=0.4 \textwidth]{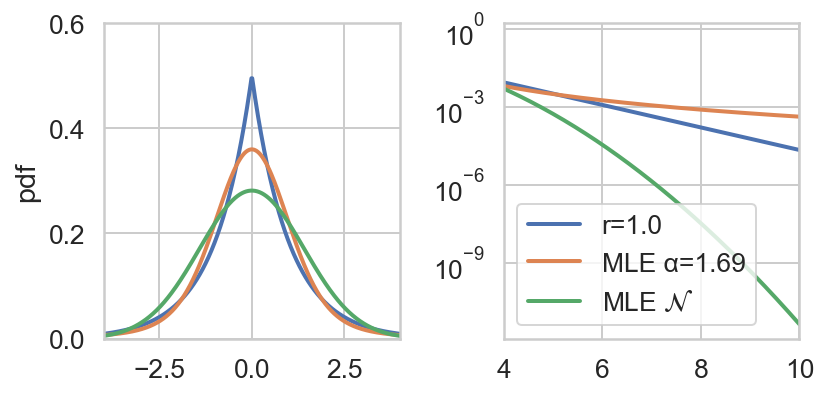}
    \includegraphics[width=0.4 \textwidth]{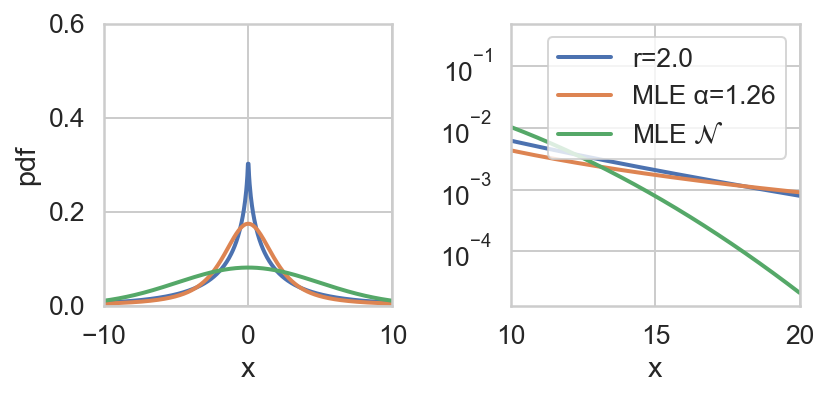}
    % \vspace{-1em}
    \caption{[left col.]  sub-Weibull pdfs ($r$=$0.8, 1.0, 2.0$), and pdfs of MLE fitted $\mathcal{S}_\alpha$ and $\mathcal{N}$. [right col.] pdf in the tails ($x>4$ and $x>10$, note log y-axis). 
    }
    \label{fig:weibullastable}
    % \vspace{-5em}
\end{figure}

\begin{figure}[t]
    \centering
     \includegraphics[width=0.2 \textwidth]{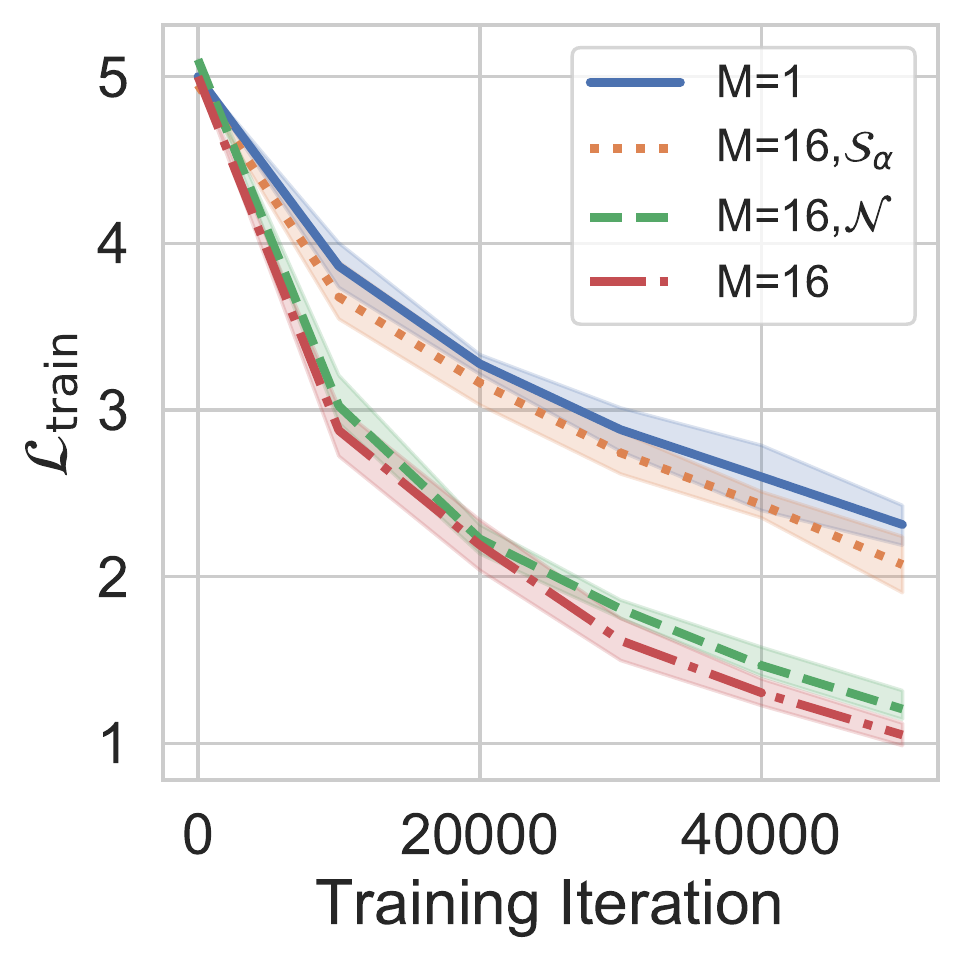}
     \includegraphics[width=0.2 \textwidth]{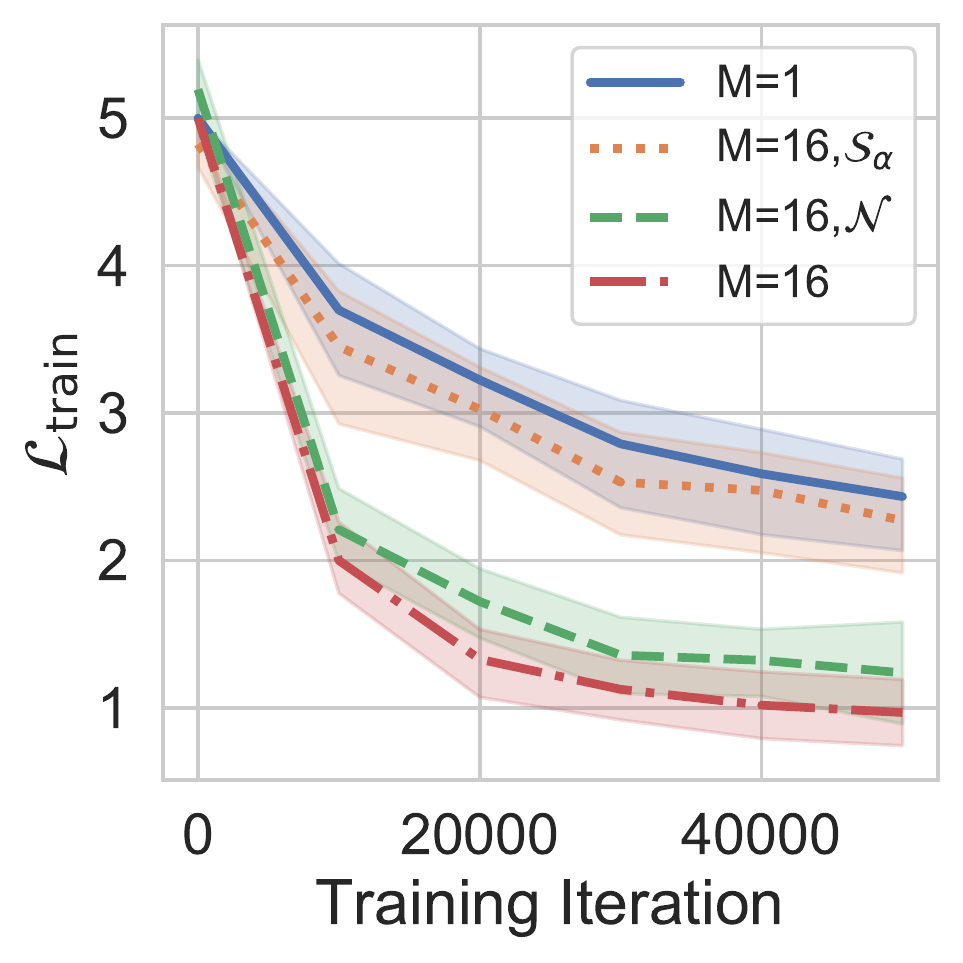}
    \caption{
    We train the networks in Figure~\ref{fig:backward_pass} on  $\widetilde{\mathcal{L}}_M$ \eqref{eq:Mmargloss}
    for the sinusoidal toy-data with additive GNIs [left] and multiplicative GNIs [right]. We fit univariate $\mathcal{S}_{\alpha}$ and univariate $\mathcal{N}$ via maximum likelihood \citep{Nolan2001} to $\nabla E_\mathcal{L}(\cdot)$ for $M=1$ models \textit{at each training step}. We add draws from these distributions to the gradients of $M=16$ models and plot the training loss ($\mathcal{L}_{\mathrm{train}}$). Shading is the standard deviation over 5 random seeds.
    }
    % \vspace{-3mm}
    \label{fig:alpha_replace}
\end{figure}

In Figure~\ref{fig:alpha_replace}, we further show that the gradient noise of the implicit effect and gradient noise drawn from an equivalent $\mathcal{S}_\alpha$ distribution will have similar effects on gradient descent.
We sample $M$ GNI samples and evaluate: 
\begin{align}\label{eq:Mmargloss}
\widetilde{\mathcal{L}}_M(\mathcal{D};\v{w}, \v{\epsilon}) = (1/M)\sum\nolimits_{m=0}^M\widetilde{\mathcal{L}}(\mathcal{D};\v{w}, \v{\epsilon}_m).
\end{align}
The objective is over the entire dataset such that we eliminate noise from the batching process.
$M$ allows us to control the `degree' to which the implicit effect is marginalised out.  
$M=1$ corresponds to the usual training with GNI and larger values of $M$ mimic the effects of marginalising out the implicit effect.
We model $\nabla E_\mathcal{L}$ for $M=1$ as being drawn from a univariate $\mathcal{S}_{\alpha}$ or a univariate normal distribution and estimate distribution parameters using maximum likelihood estimation, as in \cite{Nolan2001}, at each training iteration. 
We add draws from the estimated distributions to the gradients of $M=16$ models to mimic the combined implicit and explicit effects.
$M=16$ models with the added  $\mathcal{S}_{\alpha}$ noise have the same training path as $M=1$ models, whereas those with Gaussian gradient noise do not. 
Thus, $\mathcal{S}_{\alpha}$ distributions
are able to faithfully capture the dynamics induced by the implicit effect on gradient descent.

In these same experiments $M=16$ models outperform $M=1$ models on training data.
We refine this study for a greater range of $M$ values in Figure~\ref{fig:implicit_bias_sinusoids}. 
As $M$ increases, performance of models on training data improves gradually, suggesting that the implicit effect degrades performance. 
Further, in Figure~\ref{fig:alpha_replace}, $M=16$ models \textit{trained with Gaussian noise} added to gradients outperform $M=1$ models, suggesting that the heavy-tails and skew of the implicit effect gradients are responsible for this performance degradation. 
We now study this apparent bias.

\begin{figure}[t!]
    \centering
    \includegraphics[width=0.2\textwidth]{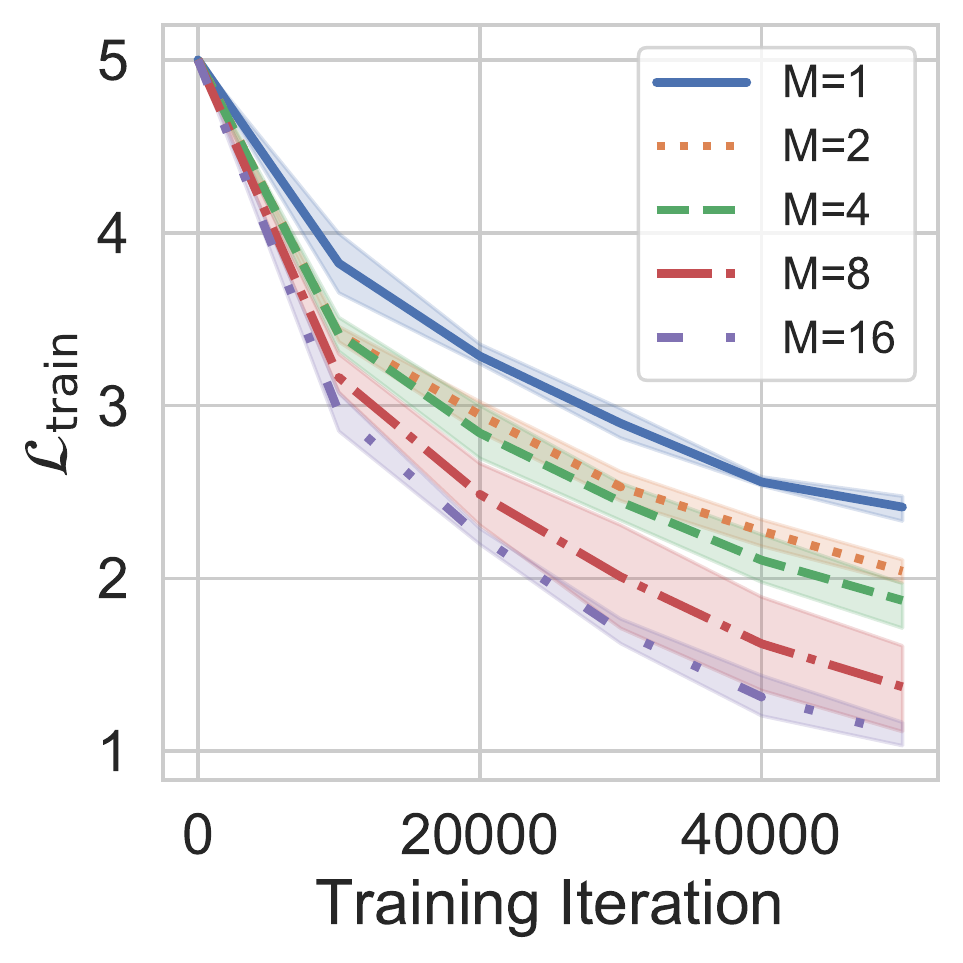} 
   \includegraphics[width=0.215\textwidth]{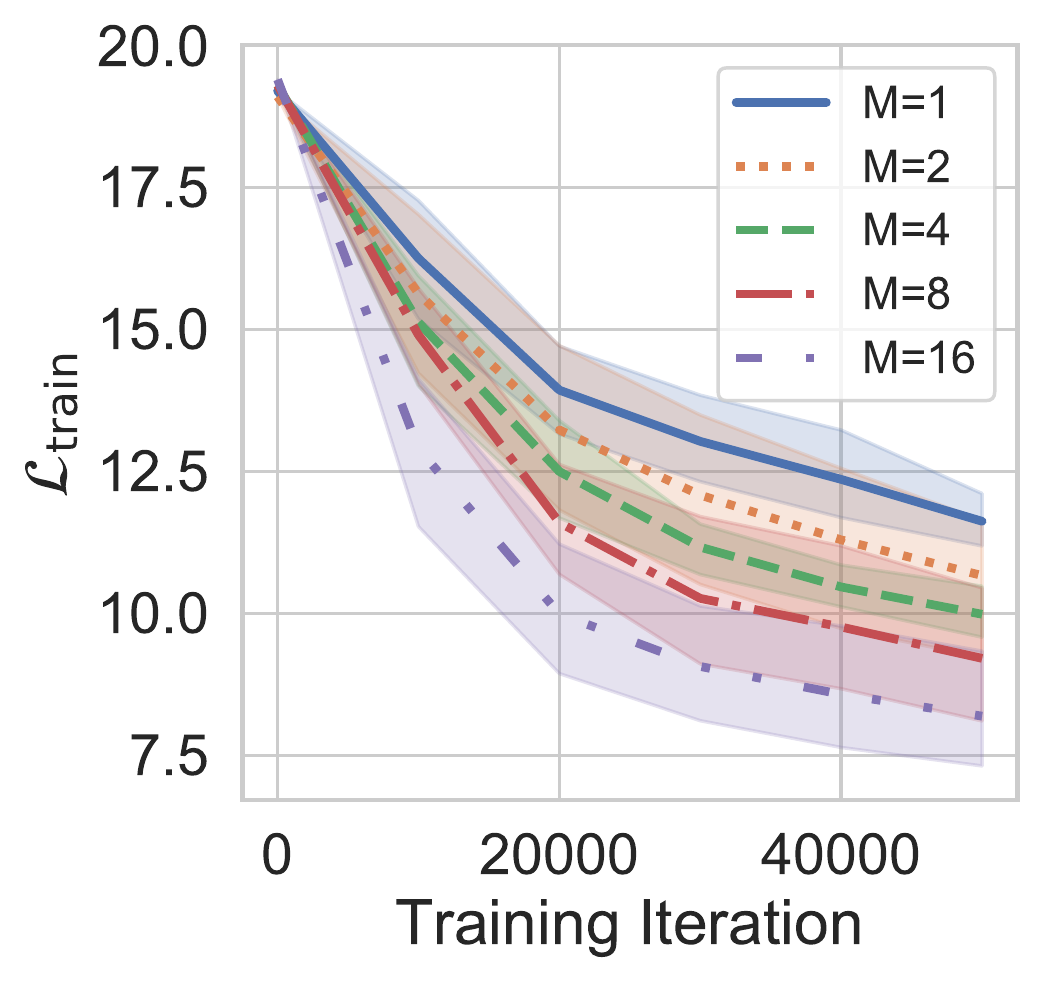}
        \\
    \caption{
    We train the networks in Figure~\ref{fig:backward_pass} on $\widetilde{\mathcal{L}}_M$ \eqref{eq:Mmargloss} for the sinusoidal toy-data with additive GNIs [left] and MNIST with multiplicative GNIs [right].
    We plot the training loss ($\mathcal{L}_{\mathrm{train}}$). Shading is the standard deviation over 5 random seeds.}
    \vspace{-10pt}
    \label{fig:implicit_bias_sinusoids}
\end{figure}

%%%%%%%%%%%%%%%%%%%%%%%%%%%%%%%%%%%%%%%%%%%%%%%%%%%%%%%%%%%%%%%%%%
%%%%%%%%%%%%%%%%%%%%%%%%%%%%%%%%%%%%%%%%%%%%%%%%%%%%%%%%%%%%%%%%%%

\textbf{Theoretical analysis of implicit bias.} 
Due to their heavy-tailed nature, stable processes have significantly different statistical properties from those of their Brownian counterparts. 
Their trajectories have a countable number of discontinuities, called \emph{jumps}, whereas Brownian motion is continuous almost everywhere. 
With these jumps the process can escape from  `narrow' basins and spend more time in  `wider' basins (see Appendix~\ref{sec:fet_metastability}
% \cite{pavlyukevich2007cooling} and  \cite{imkeller2010first} 
for the definition of width).
Theoretical results demonstrating this have been  provided for symmetric stable processes ($\theta = 0$) \citep{Simsekli2019a}. 
By translating the related metastability results from statistical physics \cite{IP2008} to our context, in Appendix~\ref{sec:fet_metastability} we illustrate that this also holds for SDEs driven by asymmetric stable processes ($\theta \neq 0$). 
In this sense, the SDE \eqref{unmodified} is  `biased' towards wider basins.

While driving SGD iterates towards wider minima could be beneficial, we now show that heavy tails can also introduce an undesirable bias and that this bias is magnified by asymmetries. 
To quantify this bias, we focus on the invariant measure (i.e., the stationary distribution) of the Markov process \eqref{eqn:sgd_approx} and investigate its modes (i.e., its local maxima), around where the process resides most of the time.  

In a statistical physics context,
\citet{dybiec2007stationary} empirically illustrated that the asymmetric stable noise can cause `shifts', in the sense that the modes of the stationary distribution of \eqref{eqn:sgd_approx} can shift away from the true local minima of $f$, which are of our interest as our aim is to minimise $f$.
They further illustrated that such shifts can be surprisingly large when $\alpha$ gets smaller and $|\theta|$ gets larger.  
We illustrate this outcome by reproducing one of the experiments provided in \cite{dybiec2007stationary} in Figure~\ref{fig:modeshift}. Here, we consider a one-dimensional problem with the quartic potential $f(w) = w^4/4 - w^2/2$, and simulate \eqref{eqn:sgd_approx} for $10$K iterations with constant step-size $\eta_k=0.001$ and $\varepsilon=1$. By using the generated iterates, we estimate the density of the invariant measure of \eqref{eqn:sgd_approx} by using the kernel density estimator provided in $\mathrm{scikit}$-$\mathrm{learn}$ \cite{pedregosa2011scikit}, for different values of $\alpha$ and $\theta$.
When $\alpha$ is larger (left), the heavy-tails cause a shift in the modes of the invariant measure, where these shifts become slightly larger with increasing asymmetries ($|\theta|>0$). 
When the tails are heavier (right), we observe a much stronger interaction between $\alpha$ and $\theta$, and observe \emph{drastic} shifts as the asymmetry is increased. 

\begin{figure}[t]
    \centering
     \includegraphics[width=0.99\columnwidth]{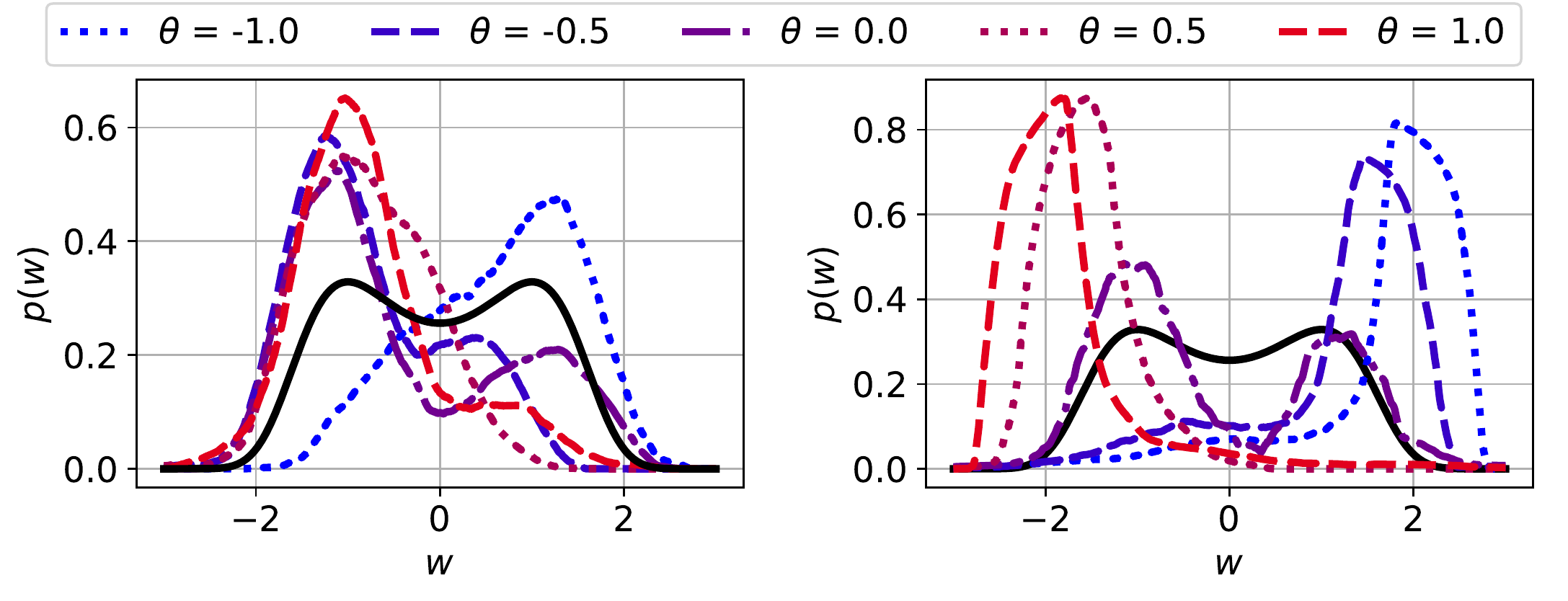}
     \vspace{-15pt}
    \caption{The stationary distributions of \eqref{eqn:sgd_approx} with $\alpha=1.9$ (left) and $\alpha=1.1$ (right). The solid black line represents the density of the Gibbs measure $\exp(-f(w))$ with $f(w) = w^4/4 - w^2/2$.
    }
    \vspace{-5pt}
    \label{fig:modeshift}
\end{figure}

From an optimisation perspective, these results are rather unsettling as they imply that
SGD might spend most of its time in regions that are arbitrarily far from the local minima of the objective function $f$, since the high probability regions of its stationary distribution might be shifted away from the local minima of interest.
In the symmetric case ($\theta=0$), this observation has been formally proven in \cite{sliusarenko2013stationary} when $f$ is chosen as the one dimensional quartic potential of Figure~\ref{fig:modeshift} and when $\alpha=1$. A direct quantification of such shifts is non-trivial; and in the presence of asymmetries ($\theta \neq 0$), even further difficulties emerge. \looseness=-1

In a recent study \citet{Simsekli2020} focused on eliminating the undesired bias introduced by 
\emph{symmetric} stable noise in SGD with momentum \cite{qian1999momentum}, and proposed an indirect way to ensure that modes of the stationary distribution exactly match the objective function's local minima. 
They developed a `modified' SDE whose invariant distribution can be proven to be the \emph{Gibbs measure}, denoted by $\pi(d \v{w})$, which is a probability measure that has a density proportional to $\exp(- C f(\v{w}))$ for some $C>0$. 
Clearly, all the local maxima of this density coincide with the local minima of the function $f$; hence, their approach eliminates the possibility of a shift in the modes by imposing a stronger condition which controls the entire invariant distribution.

By following a similar approach, we will bound the gap between the invariant measure of \eqref{eqn:sgd_approx} and the Gibbs measure in terms of the tail index $\alpha$ and the skewness $\theta$, using it as a quantification of the bias induced by the asymmetric heavy-tailed noise.
In particular, for any sufficiently regular test function $g$, we consider its expectation under the Gibbs measure  $\nu(g) := \int g(\v{w}) \pi(d\v{w})$, and its sample average computed over \eqref{eqn:sgd_approx}, i.e., ${\nu}_N(g) : = \frac{1}{H_N} \sum_{k=1}^{N}\eta_k g({\v{w}}_k)$, where $H_N = \sum_{k=1}^{N}\eta_k$. We then bound the \emph{weak error}: 
% \begin{align}
$|\nu(g) - \lim_{N \to \infty} {\nu}_N(g)|$, whose convergence to zero is sufficient for ensuring the modes do not shift.

% To achieve this goal, our roadmap will be as follows.
We derive this bound for our case in three steps: (i) We first link the discrete-time process \eqref{eqn:sgd_approx} to its continuous-time limit \eqref{unmodified} by directly using the results of \cite{panloup2008recursive}.
% This relation has already been established in , and we directly use their results. 
(ii) We then design a modified SDE that has the unique invariant measure as the Gibbs measure, with all the modes matching those of the loss function (Theorem~\ref{thm:AFLD}). (iii) Finally, we show that the SDE \eqref{unmodified} is a poor numerical approximation to the modified SDE, and we develop an upper-bound for the approximation error (Theorem~\ref{thm:bias:approx}).

To address (ii), we introduce 
a modification to \eqref{unmodified}, coined
\textit{asymmetric fractional Langevin dynamics}:
\begin{align}\label{eqn:AFLD}
d\v{w}_t=b(\v{w}_{t-},\alpha,\theta)dt + \varepsilon d\v{L}_t^{\alpha,\theta}\,.
\end{align}
where, the drift function $b(\v{w},\alpha,\theta):=((b(\v{w},\alpha,\theta))_{i},1\leq i\leq d)$ is defined as follows:
\begin{align}\label{b:defn}
(b(\v{w},\alpha,\theta))_{i}
=\frac{\varepsilon^{\alpha}}{\varphi(\v{w})}\mathcal{D}^{\alpha-2,-\theta_{i}}_{w_{i}}(\partial_{w_{i}}\varphi(\v{w})),
\end{align}
where $\theta_{i}\in(-1,1)$, $1\leq i\leq d$, $1<\alpha<2$, and
$\varphi(\v{w}):=e^{-\varepsilon^{-\alpha} f(\v{w})}$. 
Here, the operator $\mathcal{D}^{\alpha-2,-\theta_{i}}$ denotes a \emph{Riesz-Feller type fractional derivative} \cite{GM1998,GM2001}
whose exact 
(and rather complicated)
definition is not essential in our problematic, and is given in Appendix~\ref{sec:numerical:approx} in order to avoid obscuring the main results.
The next theorem states that the SDE \eqref{eqn:AFLD} targets the Gibbs measure.

\begin{theorem}\label{thm:AFLD}
The Gibbs measure $\pi(d\v{w}) \propto \exp(-\varepsilon^{-\alpha}f(\v{w}))d\v{w}$ 
% in \eqref{eqn:gibbmeas}  
is an invariant distribution of 
% the asymmetric fractional Langevin dynamics 
\eqref{eqn:AFLD}.
If $b(\v{w},\alpha,\theta)$ is Lipschitz continuous in $\v{w}$, 
then $\pi(d\v{w})$ is the unique invariant distribution of \eqref{eqn:AFLD}.
\end{theorem}
This theorem states that the use of the modified drift $b$ in place of $-\nabla f$, prevents any potential shifts in the modes of the invariant measure.

The fractional derivative in \eqref{b:defn} is a non-local operator that requires the knowledge of the full function, and does not admit a closed-form expression.
In the next step, we develop an approximation scheme for the drift $b$ in \eqref{eqn:AFLD}, and show that the gradient $-\nabla f$ in \eqref{eqn:sgd_approx} appears as a special case of this scheme. To simplify notation, we consider the one-dimensional case ($d=1$); however, our results can be easily extended to multivariate settings by applying the same approach to each coordinate. Hence, in the general case the bounds will scale linearly with $d$.
We define the following  approximation for $b$:
\begin{align}
b_{h,K}(w,\alpha,\theta):=\frac{\varepsilon^{\alpha}}{\varphi(w)}\Delta_{h,K}^{\alpha-2,-\theta}(\partial_{w}\varphi(w)),
\end{align}
where, for an arbitrary function $\psi$, we have
\begin{align*} 
\Delta_{h,K}^{\gamma,-\theta}\psi(w)
:= \frac{c_\gamma}{h^{\gamma}} \sum_{k=-K}^{K} (1+ \theta \sgn (k)) \tilde{g}_{\gamma,k}  \psi(w - kh).
\end{align*}
Here, $\sgn$ denotes the sign function, $h>0$, $K\in\mathbb{N}\cup\{0\}$, $c_\gamma := 1/(2\cos(\gamma\pi/2))$, and $\tilde{g}_{\gamma,k} := (-1)^k\Gamma(-\gamma+k)/\Gamma(k+1)\Gamma(-\gamma)$. This approximation is designed in the way that we recover the original drift $b$ as $h \to 0$ and $K \to \infty$ for sufficiently regular $\varphi$.
It is clear that when we set $K=0$ and $h = h_0:= [2 \varepsilon^{-\alpha} \cos((\alpha-2)\pi/2)]^{1/(2-\alpha)}$, we have $b_{h,K}(w,\alpha,\theta) = -\partial_w f(w)$.
In the multidimensional case, where we apply this approximation to each coordinate, the same choice of $K$ and $h$ gives us the original gradient $-\nabla f$; hence, we fall back to the original recursion  \eqref{eqn:sgd_approx}. 
By considering the recursion with this approximate drift
\[\tilde{\v{w}}_{n+1} = \tilde{\v{w}}_n + \eta_{n+1}b_{h,K}(\tilde{\v{w}}_n,\alpha,\theta) + \varepsilon \eta_{n+1}^{1/\alpha} \Delta \v{L}^{\alpha,\theta}_{n+1},\]
and the corresponding sample averages $\tilde{\nu}_N(g) := \frac{1}{H_N} \sum_{k=1}^{N}\eta_k g(\tilde{\v{w}}_k)$\footnote{Note that, with the choice of $K=0$ and $h= h_0$, $\tilde{\nu}_N(g)$ reduces to the original sample average $\nu_N(g)$.},
we are ready to state our error bound.
We believe this result is interesting on its own, and would be of further interest in statistical physics and applied probability.
To avoid obscuring the result, we state the required assumptions in the Appendix, which mainly require decreasing step-size and ergodicity.

\begin{theorem}\label{thm:bias:approx}
Let $\gamma:=\alpha-2\in(-1,0)$. 
Suppose that the assumptions stated in the Appendix hold. 
Then, the following bound holds almost surely:
\begin{align}
\label{eqn:thm_errbnd}
&\left\lvert \nu(g) - \lim\nolimits_{N \rightarrow \infty} \tilde{\nu}_N (g) \right\rvert 
\\
&  \leq\frac{\tilde{C}}{4\pi(|\gamma|+2)}\left[|\theta||\gamma| + \left|\tan\left(\gamma\pi/2\right)\right||\gamma|\right] h 
\nonumber \\
& \qquad+ \left((1+\theta)C'_0 + (1-\theta)C''_0  \right)\frac{1}{hK}+\mathcal{O}(h^{2})\,,\nonumber
\end{align}
where $\tilde{C},\,C'_{0},\,C''_{0}>0$ are constants.
\end{theorem}
We note our result extends the case $\alpha = 2$, $\theta = 0$ in \citet{DM2015} and the case $\alpha\neq 2$, $\theta=0$ in \citet{FLMC}; whereas we cover the case $\alpha \neq 2$, $\theta \in (-1,1)$.
The right-hand-side of \eqref{eqn:thm_errbnd} contains two main terms. The second term shows that the error increases linearly with decreasing $K$, indicating that the error can be \emph{arbitrarily large} when $K=0$, and the gap cannot be controlled without imposing further assumptions on $f$. 

More interestingly, even when $K$ goes to infinity (i.e., the second term vanishes), the first term stays unaffected. 
Note that for large enough $\varepsilon$, $h_0$ increases as $\alpha \in (1,2)$ decreases. 
In this regime, the first term indicates that the error increases with decreasing $\alpha$, and an additional error term appears whenever $\theta \neq 0$, which is further amplified with the heaviness of the tails (measured by $|\gamma|$). 
This outcome provides a theoretical justification to the empirical observations stated in Figures~\ref{fig:modeshift} and \ref{fig:svhn_implicit_effect}.

% %%%%%%%%%%%%%%%%%%%%%%%%%%%%
%%%%%%%%%%%%%%%%%%%%%%%%%%%%%%%%%%%%%%%%%%%%%%%%%%%%%%%%

\begin{figure}[t!]
    \centering
    \includegraphics[width=0.35\textwidth]{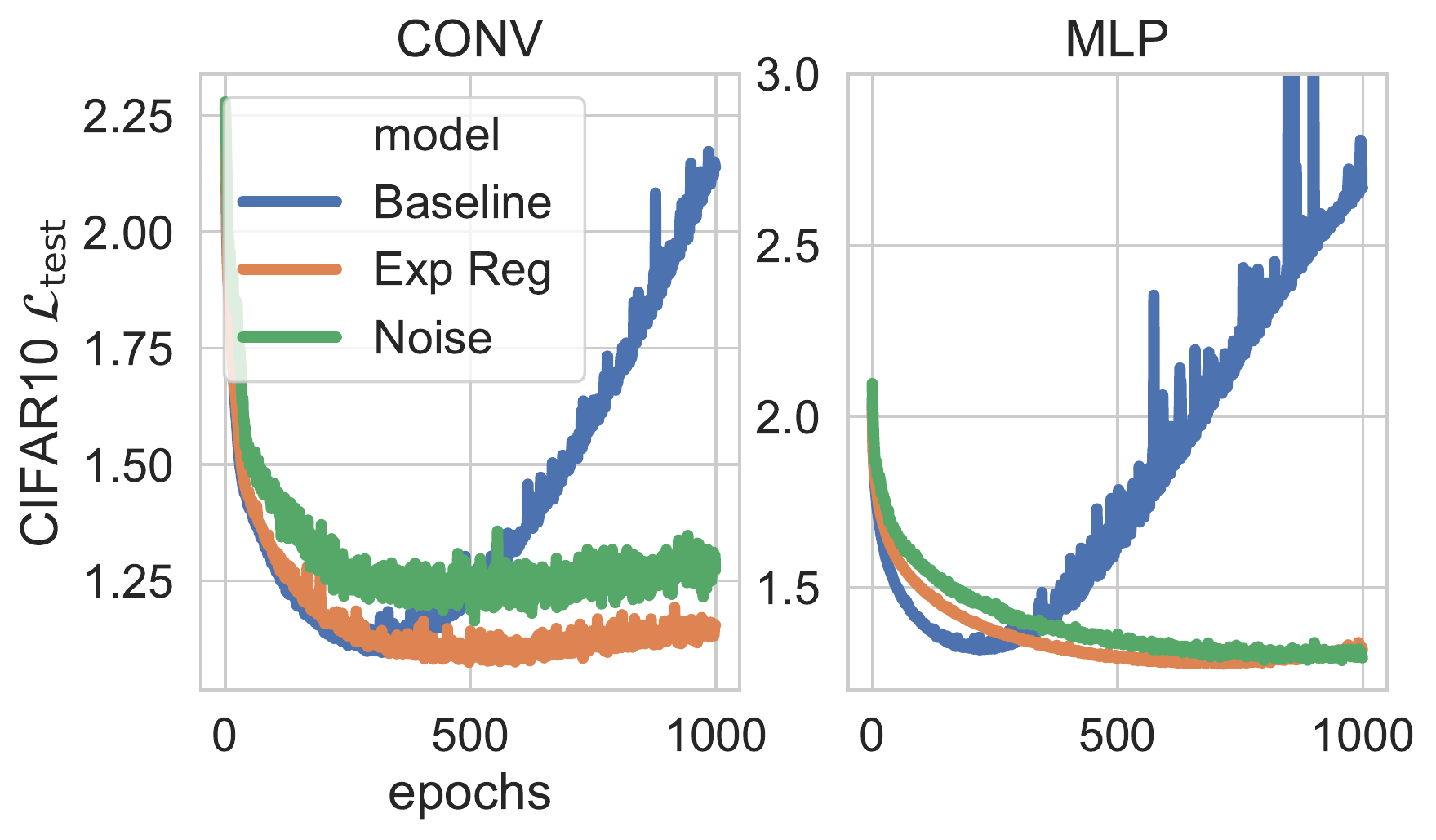} \\
   \includegraphics[width=0.35\textwidth]{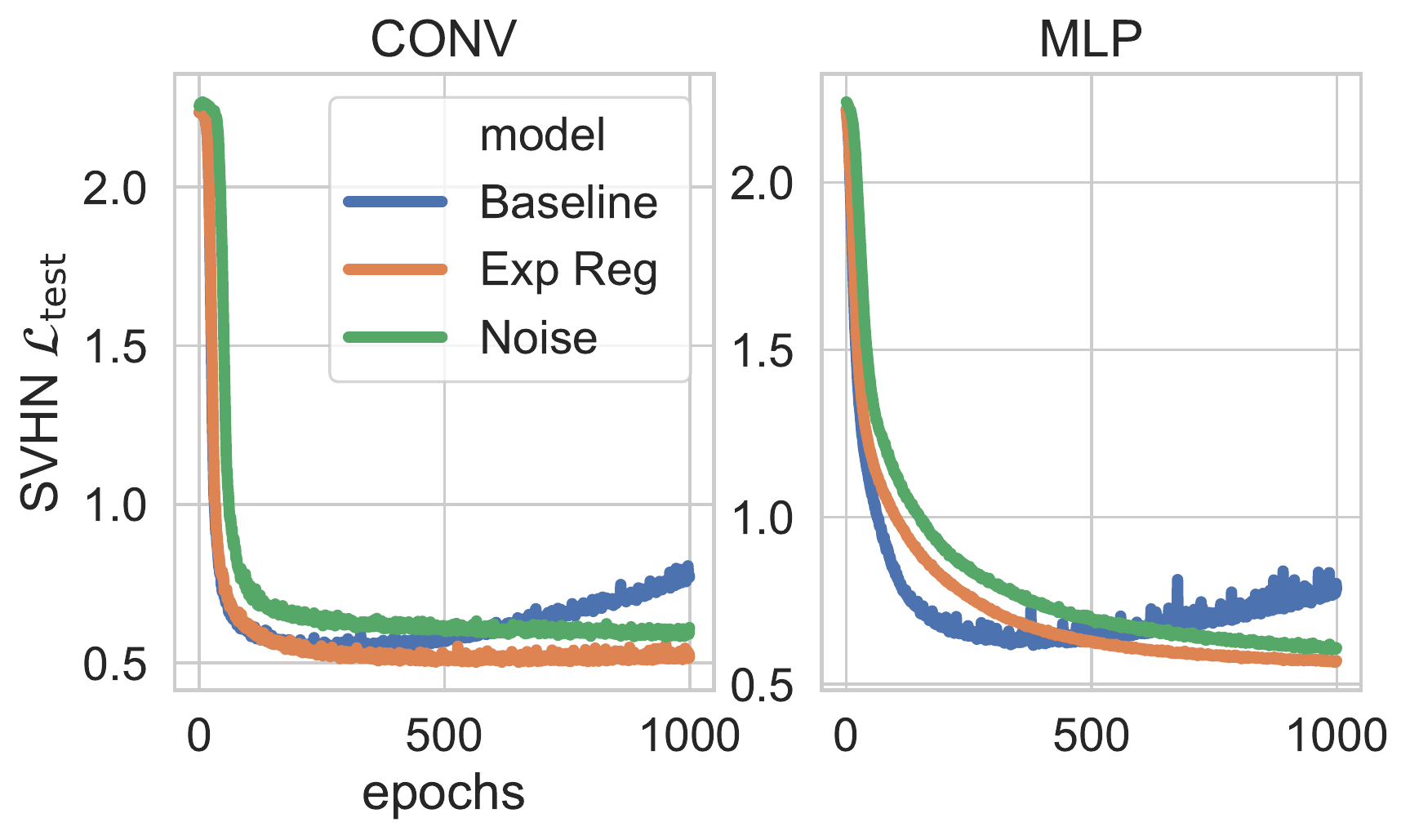}
    \caption{
    We show the test-set loss for SVHN [bottom] and CIFAR10 [top], for 2-layer convolutional (CONV) and 4-layer MLPs with 512 units per layer trained with the explicit regulariser approximation $R(\cdot)$ of \citet{Camuto2020_2} (Exp Reg),  with additive-GNIs ($\sigma^2=0.1$) (Noise), and no regularisation (Baseline).% Exp Reg models  consistently outperform Noise models by a small margin.
    }
    \label{fig:exp_reg_noise_gap}
\end{figure}

%%%%%%%%%%%%%%%%%%%%%%%%%%%%%%%%%%%%%%%%%%%%
\subsection{Further Experiments}
We have already ascertained that the bias implied by Theorem~\ref{thm:bias:approx} has a visible impact on training performance in Figures~\ref{fig:alpha_replace} and \ref{fig:implicit_bias_sinusoids}. 
We have also already shown that the asymmetry and heavy-tails of the implicit-effect gradient noise are responsible for this performance degradation in Figure~\ref{fig:alpha_replace}: models trained with Gaussian noise on gradients outperform models trained with $\mathcal{S}_{\alpha}$ gradient noise, and those trained with the implicit effect, \emph{on training data without batching}. 
%%%%%%%%%%%%%%%%%%%%%%%%%%%%%%%%%%%
We corroborate these findings with experiments \textit{with mini-batching and results on test data}.
In Figure~\ref{fig:exp_reg_noise_gap}
we use the approximation of the explicit regulariser $R(\mathcal{B}; \v{w})$ derived by \citet{Camuto2020_2}  for computational efficiency. 
Convolutional networks trained with $R$ consistently outperform those trained with GNIs and mini-batching on \textit{held-out data}, supporting that the implicit effect degrades performance. 
%%%%%%%%%%%%%%%%%%%%%%%%%%%%%%%%%%%%%
In Figure~\ref{fig:svhn_implicit_effect}, we sample $M$ multiplicative-GNI samples and marginalise out the implicit effect as before.  We model the gradients of the implicit effect, $\nabla E_\mathcal{L}(\cdot)$, as an $\mathcal{S}_\alpha$ distribution. 
Empirically, we found that when increasing the variance ($\sigma^2$) of the injected noise, the gradient noise $\nabla E_\mathcal{L}(\cdot)$ becomes increasingly heavy-tailed and skewed, i.e. $\alpha$ decreases and $|\theta|$ increases, and in tandem larger $M$ models begin to outperform smaller $M$ models on \emph{held-out data}, when trained with mini-batches.
These results support that GNIs induce bias in SGD because of the asymmetric heavy-tailed noise they induce on gradient updates.

\begin{figure}[h!]
    \centering
    \includegraphics[width=0.48\textwidth]{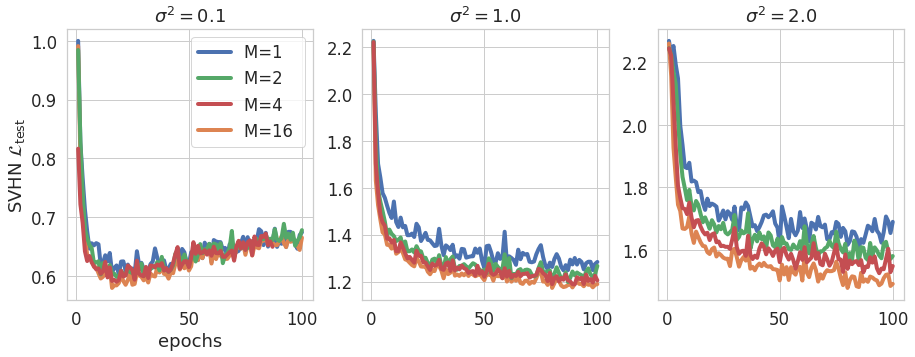} \\
    \hspace{0.0cm}
    \includegraphics[width=0.15\textwidth]{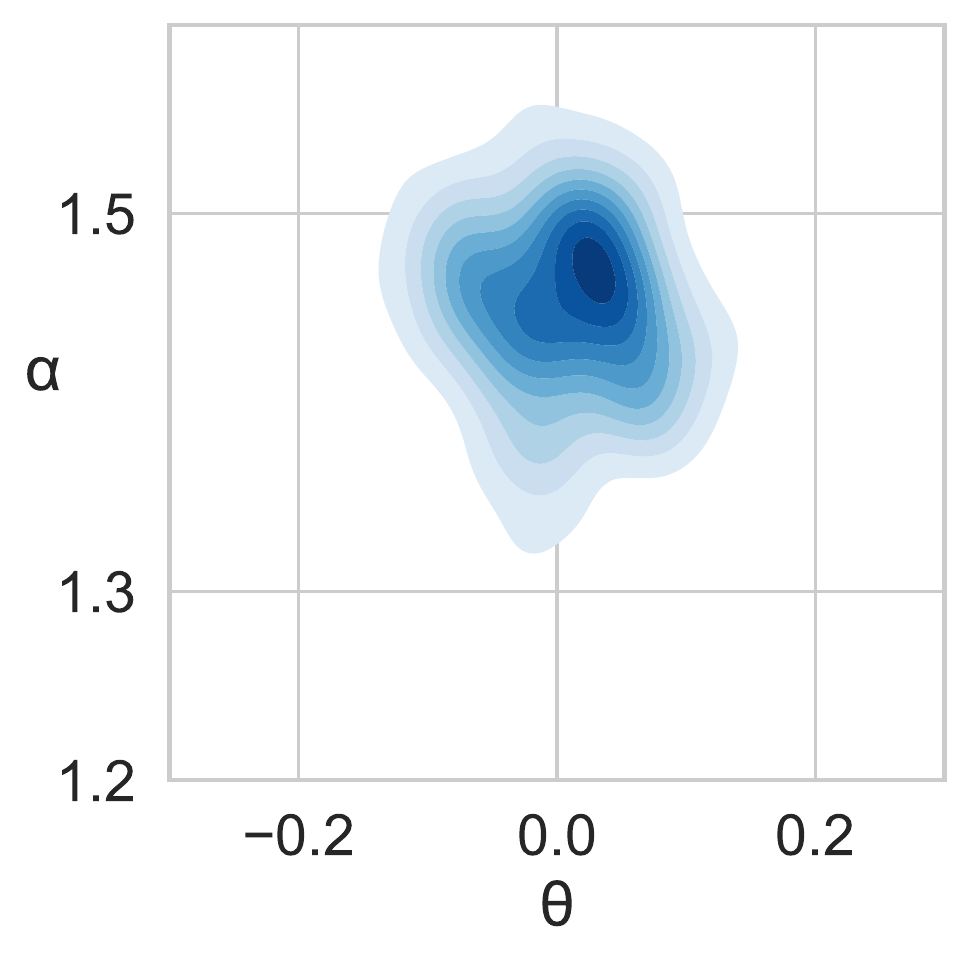}
    % \hspace{0.1cm}
    \includegraphics[width=0.15\textwidth]{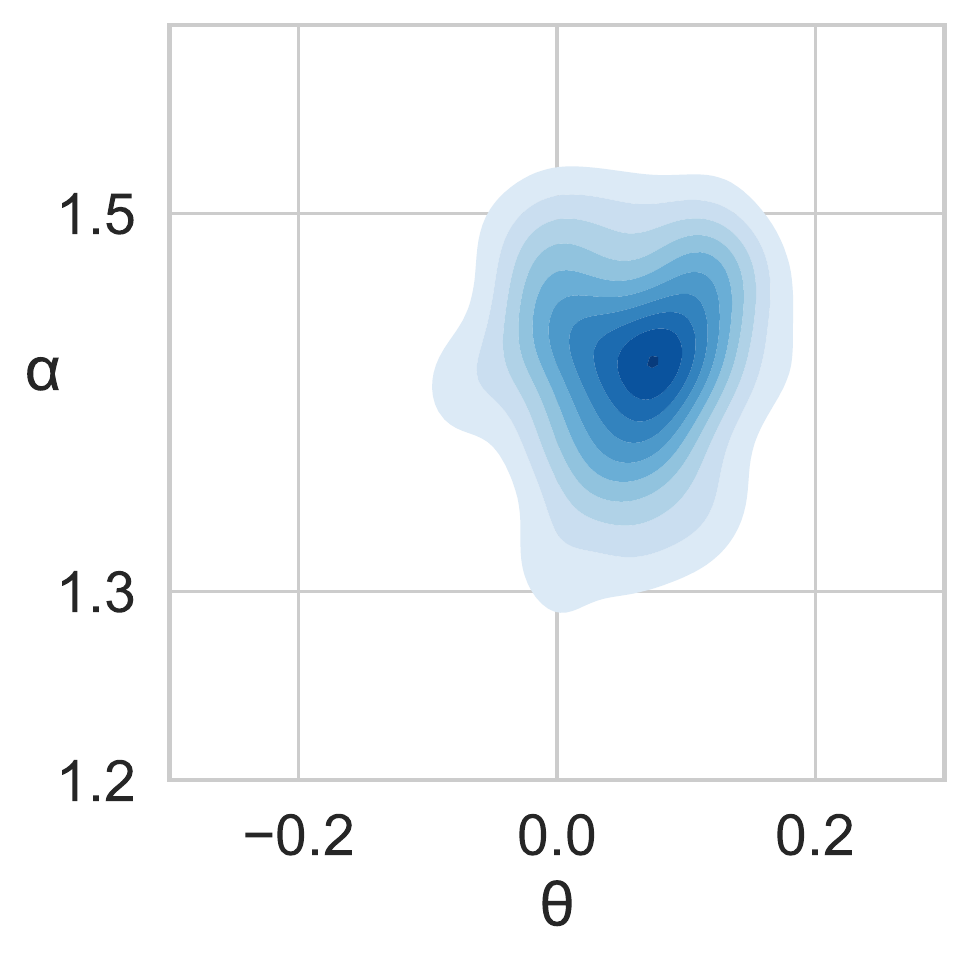}
    %  \hspace{0.1cm}
    \includegraphics[width=0.15\textwidth]{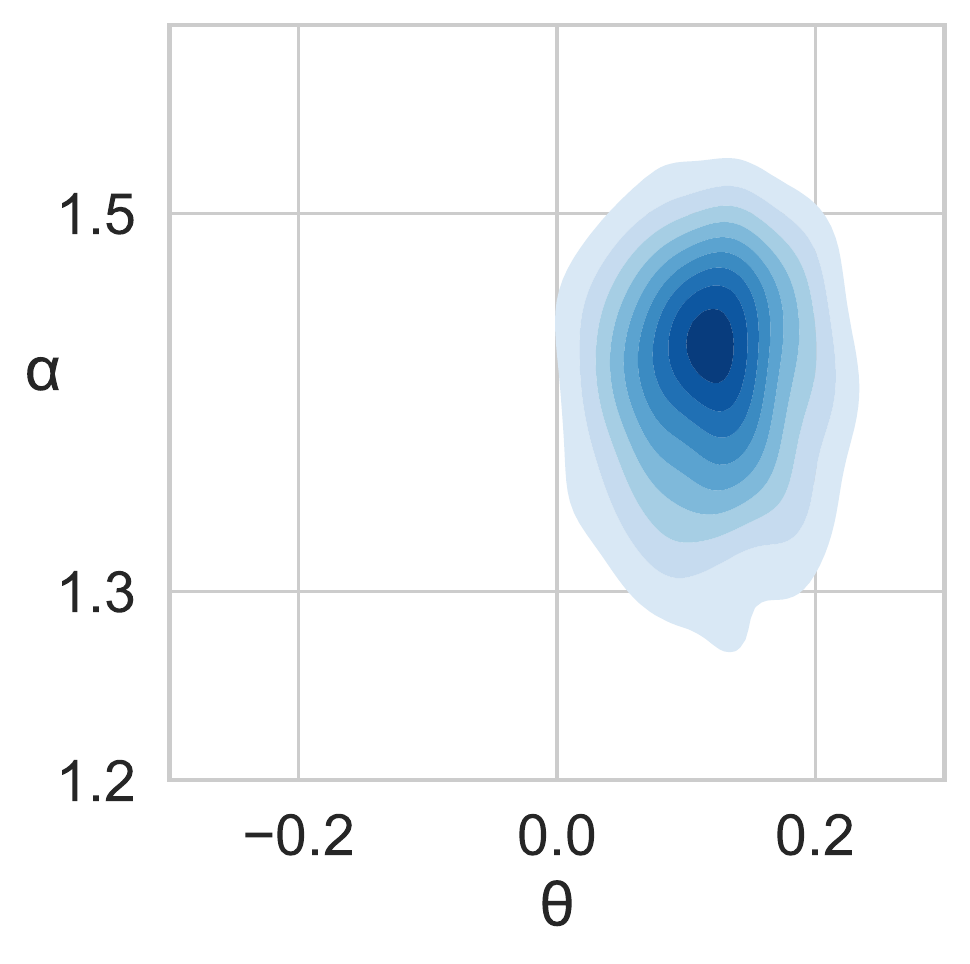}
    \caption{[first row] We train 2-dense-layer-256-unit-per-layer ELU networks on the objective $\frac{1}{M}\sum_{m=0}^M\widetilde{\mathcal{L}}(\mathcal{B};\v{w}, \v{\epsilon}_m)$
    with a cross-entropy loss (see Appendix~\ref{sec:costfunctions}) for SVHN.
    We use \textit{multiplicative} noise of variance $\sigma^2$ and batch size of 512. 
    We plot the test-set loss ($\mathcal{L}_{\mathrm{test}}$).
    [second row] We fit univariate $\mathcal{S}_{\alpha}$ via maximum likelihood \citep{Nolan2001} to $\nabla E_\mathcal{L}(\cdot)$ and show KDE plots of parameters' %parameters
    estimates.  
    }
    \label{fig:svhn_implicit_effect}
    % \vspace{-9pt}
\end{figure}

%%%%%%%%%%%%%%%%%%%%%%%%%%%%%%%%%%%%%%%
\section{Conclusion}

Our work lays the foundations for the study of regularisation methods from the perspective of SDEs.
We have shown that Gaussian Noise Injections (GNIs), though they inject Gaussian noise in the forward pass, induce asymmetric heavy-tailed noise on gradient updates by way of the implicit effect. 
By modelling the overall induced noise using an asymmetric $\alpha$-stable noise, we demonstrate that the stationary distribution of this process gets arbitrarily distant from the so-called Gibbs measure, whose modes exactly match the local minima of the loss function, shedding light on why neural networks trained with GNIs underperform networks trained solely with the explicit effect.
Given the deleterious effects of asymmetric gradient noise on gradient descent, extensions of this work could focus on methods that symmetrise gradient noise, stemming from batching or noise injections, so as to limit these negative effects.

\section*{Acknowledgements}
This research was directly funded by the Alan Turing Institute under Engineering and Physical Sciences Research Council (EPSRC) grant EP/N510129/1.
Alexander Camuto was supported by an EPSRC Studentship.
Xiaoyu Wang and Lingjiong Zhu are partially supported by the grant NSF DMS-2053454 from the National Science Foundation.
Lingjiong Zhu is also grateful to the partial support from a Simons Foundation Collaboration Grant. 
Mert G\"urb\"uzbalaban’s research is supported in part by
the grants Office of Naval Research Award Number N00014-21-1-2244, National Science Foundation (NSF) CCF-1814888, NSF DMS-1723085, NSF DMS-2053485. 
Umut \c{S}im\c{s}ekli's research is partly supported by the French government under management of Agence Nationale de la Recherche as part of the ``Investissements d’avenir'' program, reference ANR-19-P3IA-0001 (PRAIRIE 3IA Institute).

%%%%%%%%%%%%%%%%%%%%%%%%%%%%%%%%%%%%%%%%%%%%%%%%%%
% \newpage

%%%%%%%%%%%%%%%%%%%%%%%%%%%%%%%%
% \newpage

\bibliographystyle{icml2021}
\bibliography{references.bib}

%%%%%%%%%%%%%%%%Reference%%%%%%%%%%%%%%

%%%%%%%%%%%%%%%%Appendix%%%%%%%%%%%%%%

%\pagebreak
\newpage
\appendix
\onecolumn
\clearpage
\setcounter{figure}{0}
\renewcommand\thefigure{\thesection.\arabic{figure}}

%%%%%%%%%%%%%%%%%%%%%%%%%%%%%%%%%%%%%%%%%%%%%%%%%%%%%%%%%%%%%%%%%%

%%%%%%%%%%%%%%%%%%%%%%%%%%%%%%%%%%%%%%%%%%%%%%%%%%%%%%%%%%%%%%%%%

\icmltitle{Asymmetric Heavy Tails and Implicit Bias in Gaussian Noise Injections \\ {\normalsize SUPPLEMENTARY DOCUMENT}}

The supplementary document is organised as follows. 

\begin{enumerate}
    \item The supplementary document begins first with a presentation of additional experiments that are referenced directly in the main text (Section~\ref{sec:additionalres}). 
    \item We then cover the cost-functions used to train neural networks in Section~\ref{sec:costfunctions}; 
    and give an overview in Section~\ref{app:skew} of the other potential sources of the implicit effect gradient noise skew which we explored.
    \item {In Section \ref{section-overview}, we provide an overview of the assumptions we will be making in our analysis}. We then describe in Section~\ref{sec:numerical:approx} the numerical method we use to approximate the drift term $b(\v{w},\alpha,\theta)$ defined in~\eqref{b:defn}.
    \item We end with metastability analysis of asymmetric stable processes (Section~\ref{sec:fet_metastability}); followed by the technical proofs of the lemmas, theorems, and corollaries that we present in the main body and the supplementary document of the paper (Section~\ref{sec:technical:proofs}).
\end{enumerate}

Before beginning the supplementary document we make a quick note of network architectures and training hyper-parameters.

\paragraph{Network Architectures}

Networks were trained using stochastic gradient descent with a learning rate of 0.0003 and batch sizes specified in text.
MLP network architectures are specified in text. 
Convolutional (CONV) networks are 2 hidden layer networks. 
The first layer has 32 filters, a kernel size of 4, and a stride length of 2. 
The second layer has 128 filters, a kernel size of 4, and a stride length of 2. 
The final output layer is a dense layer.

\section{Additional Experimental Results}\label{sec:additionalres}

\begin{figure*}[h!]
    \centering
    \includegraphics[width=0.9\textwidth]{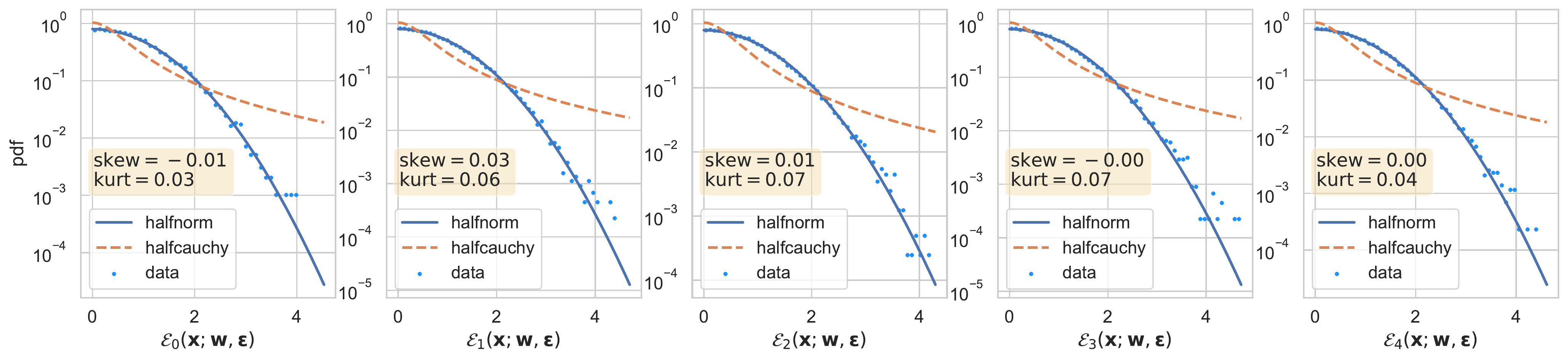} 
    \caption{We measure the skewness and kurtosis \textit{at initialisation} of the noise accumulated on network activations at each layer $i$ for a 4-layer 256-unit wide MLP trained to regress the function $\lambda(\v{x}) = \sum_i \sin(2\pi q_i\v{x}+\phi(i))$ with $q_i  \in(5,10,\dots,45,50), ~ \v{x} \in \mathbb{R}$ and experiencing additive-GNIs. 
    We plot the probability density function of positive samples, comparing against half-normal (non-heavy-tailed) and half-Cauchy (heavy-tailed) distributions, where $\mathcal{E}_i(\v{x}; \v{w}, \v{\epsilon})$ is defined in \eqref{eq:accum_noise}.
    Each blue point represents the noise on an individual activation in a layer $i$ for a point $\v{x}$. 
    This noise is Gaussian (low skew and kurtosis) with a p.d.f. that tracks that of a half-normal. 
    }
    \label{fig:forward_pass}
\end{figure*}

\begin{figure*}[h!]
    \centering
    \includegraphics[width=0.9\textwidth]{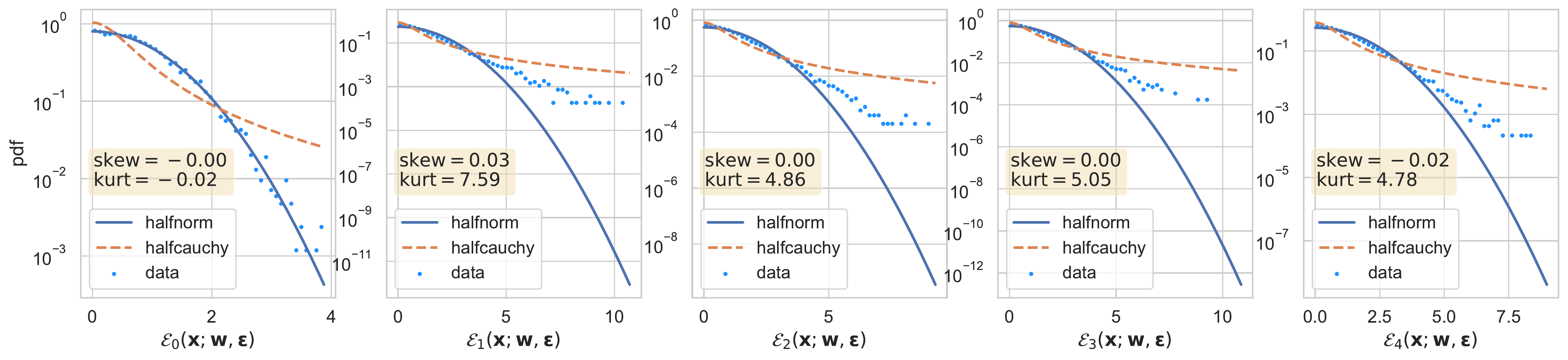}
    \caption{
    Here we show the same plots as in Figure~\ref{fig:forward_pass} but for multiplicative-GNIs. 
    The forward pass here experiences  \textit{symmetric heavy-tailed noise} for all layers past the data layer.
    }
    \label{fig:forward_pass_mult}
\end{figure*}

\begin{figure*}[h!]
     \centering
     \includegraphics[width=0.75\textwidth]{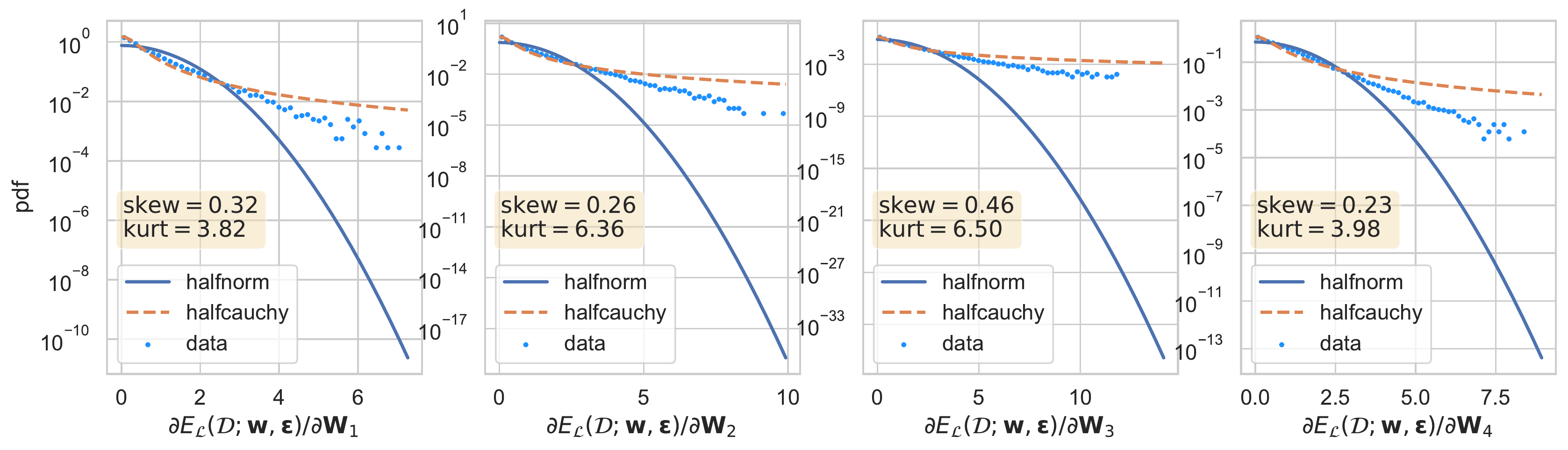}
    \caption{
    Here we show the same plots as in Figure~\ref{fig:backward_pass} but for multiplicative-GNIs. 
    The gradient noise is \textit{skewed and heavy-tailed}, with a p.d.f. that is more Cauchy-like than Gaussian. 
    The kurtosis decays as the gradients approach the input layer, as predicted by Theorem~\ref{thm:gn_tail_properties}.
    }
    \label{fig:backward_pass_mult}
\end{figure*}

\begin{figure*}[t]
    \centering
     \subfigure[CIFAR10 ($\times$)]{\includegraphics[width=0.25 \textwidth]{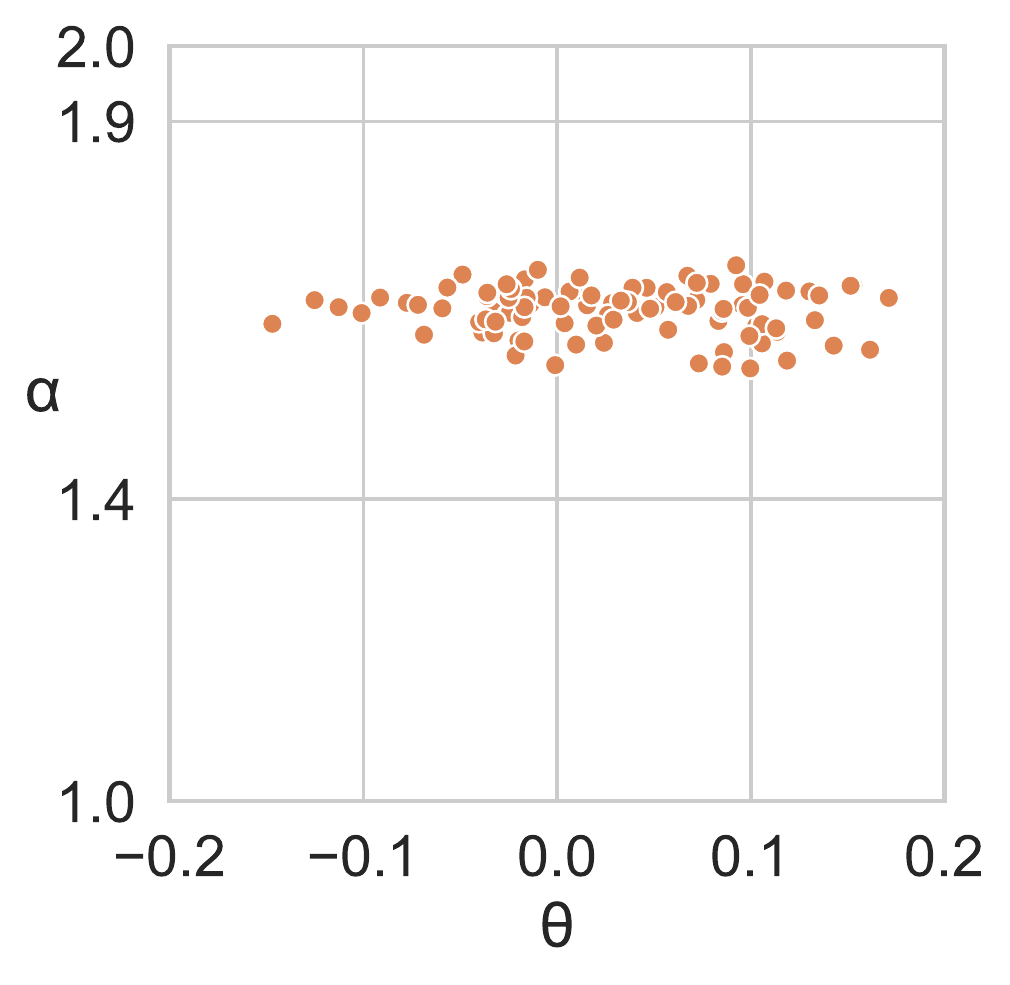}}
     \subfigure[SVHN ($\times$)]{\includegraphics[width=0.25 \textwidth]{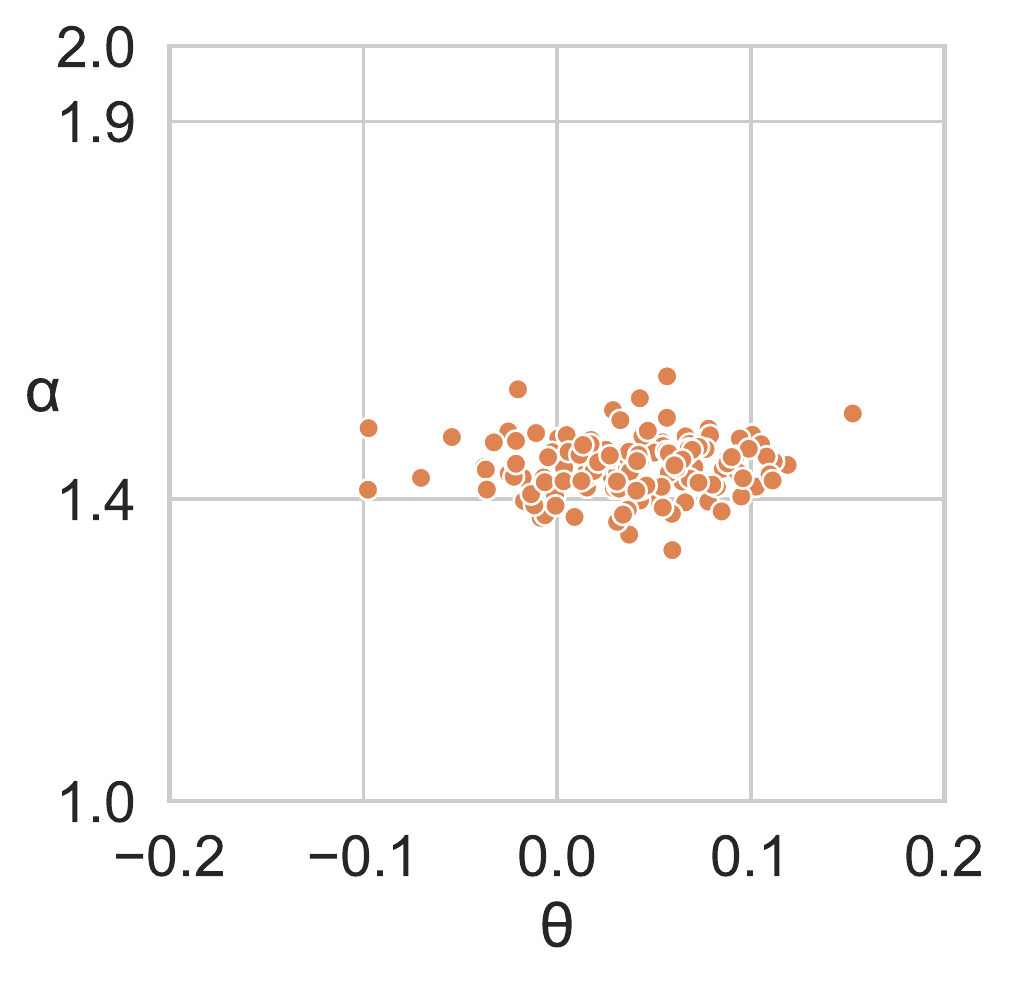}}  
    \subfigure[BHP ($+$)]{\includegraphics[width=0.25 \textwidth]{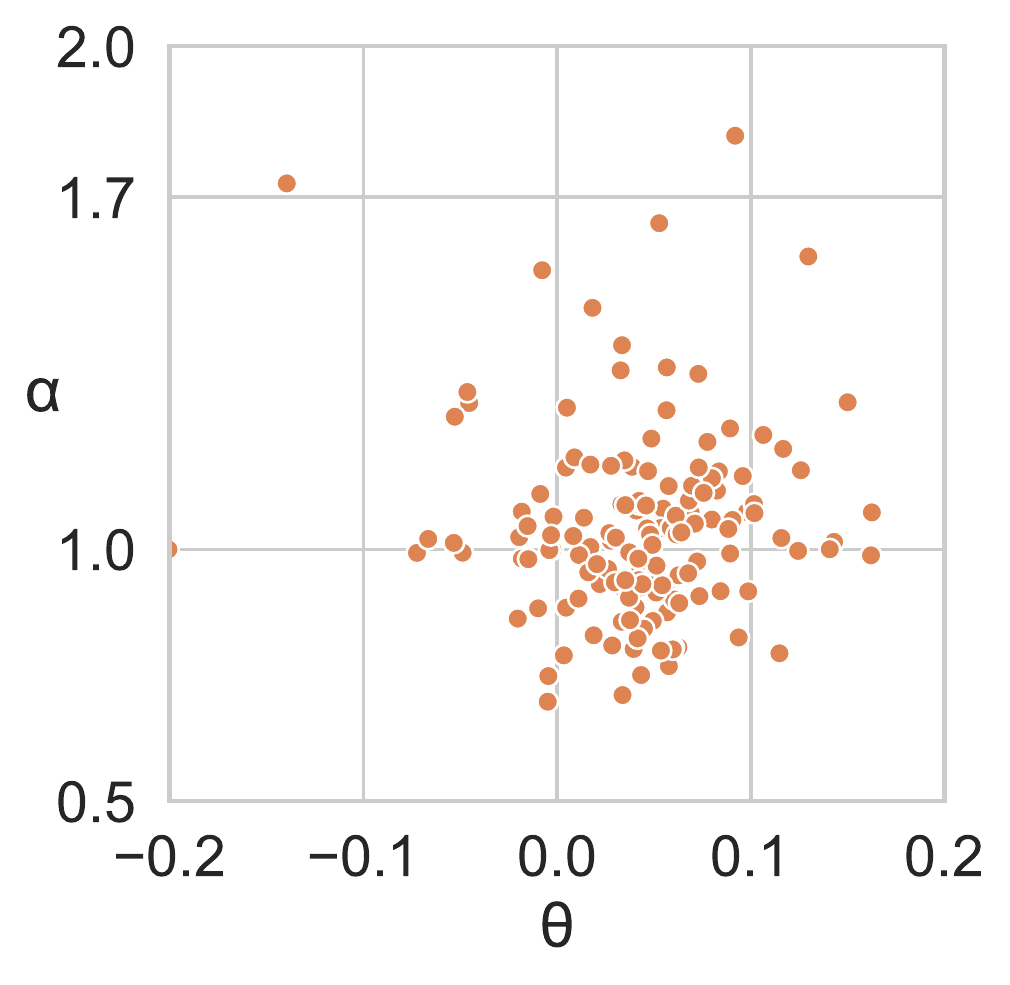}} 
    \caption{We model $\nabla E_\mathcal{L}(\cdot)$ as being drawn from some $\alpha$-stable distribution $\mathcal{S}_\alpha$ and estimate the tail-index $\alpha$ and skewness $\theta$ using 
    maximum likelihood estimation as in \citet{Nolan2001}.
    We plot the results as a scatter for a batch of size $B=512$ for CIFAR10 and SVHN and $B=32$ for Boston House Prices. Additive ($+$) and multiplicative ($\times$) GNIs have $\sigma^2=0.1$.
    }
    \label{fig:scatter_alpha_stable}
\end{figure*}

%%%%%%%%%%%%%%%%%%%%%%%%%%%%%%%%%%%%%%%%%%%%%%%%%%%%%%%%%%%%%
\clearpage
\section{Cost Functions}\label{sec:costfunctions}
\subsection{Mean Square Error}

In the case of regression the most commonly used loss is the mean-square error.
\begin{equation}
\mathcal{L}(\v{x}, \v{y}) = \left(\v{y} - \v{h}_{L}(\v{x})\right)^2\,. \nonumber
\end{equation}

\subsection{Cross Entropy Loss}

In the case of classification, we use the cross-entropy loss.
If we consider our network outputs $\v{h}_L$ to be the pre-$\mathrm{softmax}$ of logits of the final layer then the loss is for a data-label pair $(\v{x}, \v{y})$
\begin{equation}
\mathcal{L}(\v{x}, \v{y}) = - \sum_{c=0}^C \v{y}_{c} \log \left(\mathrm{softmax}(\v{h}_{L}(\v{x}))_c\right)\,,
\label{eqapp:ce}
\end{equation}
where $c$ indexes over the $C$ possible classes of the classification problem.

%%%%%%%%%%%%%%%%%%%%%%%%%%%%%%%%%%%%%%%%%%%%%%%%%%%%%%%%%%%%%%%%%%
\section{Other Potential Sources of the Skewness in the Gradient Noise}\label{app:skew}

The product of correlated random variables can be skewed \citep{Oliveira2016, NADARAJAH2016201}. 
Our first hypothesis was that the skew came from the correlation of $(\partial E_{\mathcal{L}}(\cdot)/ \partial h_{i}^{m})$ and $(\partial h_{i}^{m} / \partial W_{i,l,j})$.
As a test, Figure~\ref{fig:forward_backward_pass_linear} of the Appendix reproduces Figure~\ref{fig:backward_pass} with linear $\kappa$, isolating gradient correlation as a potential source of skew.
Gradients are not skewed, demonstrating that the asymmetry stems from non-linear $\kappa$.

\begin{figure*}[h]
    \centering
     \includegraphics[width=0.75\textwidth]{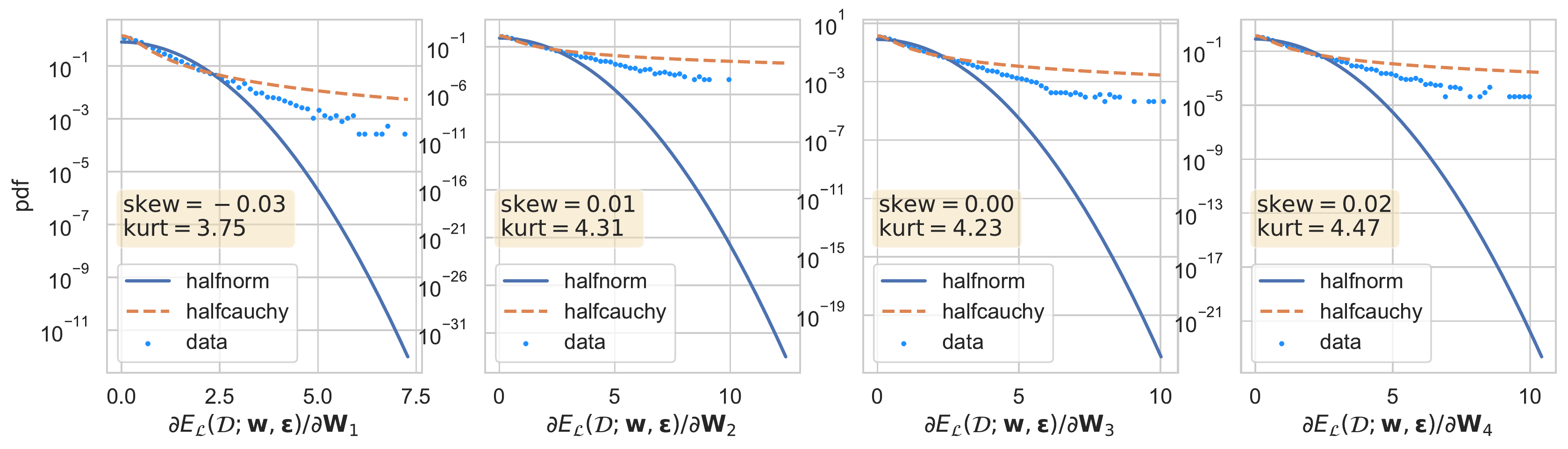}
    \caption{Here we reproduce Figure~\ref{fig:backward_pass} but with no non-linearities. 
    This gradient noise is clearly \textit{heavy-tailed} but \textit{not} skewed. 
    }
    \label{fig:forward_backward_pass_linear}
\end{figure*}

%%%%%%%%%%%%%%%%%%%%%%%%%%%%%%%%%%%%%%%%%%%%%%%%%%%%%%%%%%%%%%%%%%%

\section{Overview of the Assumptions}
\label{section-overview}

Due to space limitations and to avoid obscuring the main take home messages of our theoretical results, we did not present the two assumption required for Theorem~\ref{thm:bias:approx}. These two assumptions are properly presented in their respective sections (Sections~\ref{sec:numerical:approx} and \ref{sec:appx_thmbias}), where we first provide the required technical context for defining them in each section. In this section, we will shortly discuss the semantics of these assumptions from a higher-level perspective for the convenience of the reader. 

\begin{itemize}
    \item \textbf{Assumption \ref{H0}.} This assumption is essentially an assumption of the tails of the function $\partial \varphi$, with $\varphi({w})=e^{-\varepsilon^{-\alpha} f({w})}$. In particular, in order to make our approximation scheme (to the fractional derivatives) convergent, this assumption makes sure that outside of a compact region, the function $\partial \varphi$ exponentially decays. 
    \item \textbf{Assumption \ref{H1}.} This assumption enforces a certain structure on the Euler-Maruyama discretisation given in Section~\ref{sec:SDE}:
    \[\tilde{\v{w}}_{n+1} = \tilde{\v{w}}_n + \eta_{n+1}b_{h,K}(\tilde{\v{w}}_n,\alpha,\theta) + \varepsilon \eta_{n+1}^{1/\alpha} \Delta \v{L}^{\alpha,\theta}_{n+1}.\]
    As a first condition, we make sure that the step-sizes are decreasing while their sum is diverging, which is a standard assumption. The second condition is essentially a Lyapunov condition that requires the modified drift $b$ behaves well, so that we can control the weak error of the sample averages by using \cite{panloup2008recursive}. The final condition is similar to the second condition in nature, and requires ergodicity of an SDE defined through the approximate drift $b_{h,K}$, in order to enable us link the weak error to the error induced by the approximation scheme used for the fractional derivatives. 
\end{itemize}

%%%%%%%%%%%%%%%%%%%%%%%%%%%%%%%%%%%%%%%%%%%%%%%%%%%%%%%%%%%%%%%%%%%
\section{Fractional Differentiation and the Approximation Scheme}\label{sec:numerical:approx}

In this section, we provide the details of the Riesz-Feller type fractional derivative $\mathcal{D}^{\alpha-2,-\theta}$, whose definition was omitted in the main document for clarity. We then present the details of the approximation method for the drift term $b(\v{w},\alpha,\theta)$ defined in~\eqref{b:defn}.

The building block of our analysis is a first-order approximation of $\mathcal{D}^{\alpha-2,-\theta} \partial_w\varphi$ for any $\partial_w\varphi \in L^1(\mathbb{R}) \cap \mathcal{C}^4(\mathbb{R})$. 
We consider the one-dimensional case for simplicity since
the L\'{e}vy motion we consider has independent coordinates,
and the multi-dimensional numerical approximation can be reduced
to the one-dimensional case.
Assume the tail index $1<\alpha<2$ and  the skewness parameter satisfies $-1 < \theta < 1$.  

When $\theta=0$, 
\citet{FLMC} developed the numerical approximation method for the drift term $b(\v{w},\alpha,0)$ by approximately computing the Riesz potential\footnote{Note that when $\theta=0$, $\mathcal{D}^{\alpha-2}=\mathcal{D}^{\alpha-2,0}$ recovers the Riesz potential.} $\mathcal{D}^{\gamma}$ via the fractional centred difference method provided by \citet{ortigueira2006riesz,orti2006b}. It is shown that 
for any $-1 < \gamma < 0$, we have the following numerical error, 
\begin{align*}
\left|\mathcal{D}^{\gamma}\partial_{w}\varphi(w)-\Delta_{h,K}^{\gamma}\partial_{w}\varphi(w)\right|
=\mathcal{O}\left(h^{2}+1/(hK)\right)\,,
\end{align*}
as $h\rightarrow 0$, 
where $K\in\mathbb{N}\cup\{0\}$ is the truncation parameter
and $\partial_{w}\varphi(w)$ satisfies some regularity conditions,
and the operator $\Delta_{h,K}^{\gamma}$ is given by 
\begin{align*}
\Delta_{h,K}^{\gamma}f(w)
=\frac{1}{h^{\gamma}}\sum_{k=-K}^{K}g_{\gamma,k}f(w-kh),
\end{align*}
for any test function $f$ satisfying some regularity conditions, where 
\begin{equation*}
g_{\gamma,k}:=\frac{(-1)^{k}\Gamma(\gamma+1)}{\Gamma\left(\frac{\gamma}{2}-k+1\right)\Gamma\left(\frac{\gamma}{2}+k+1\right)}.
\end{equation*}

We study the numerical error 
when approximating the drift term $b(\v{w},\alpha,\theta)$ with the skewness parameter $-1 < \theta < 1$ and provide the truncation error with a truncation parameter $K$ in Corollary~\ref{cor:K}. Based on this result, Theorem~\ref{thm:bias:approx} follows, 
which quantifies the bias induced by  $\alpha$-stable noise on gradient updates using the Euler-Maruyama scheme.

Instead of using the centred difference method to implement the approximation for the $\theta=0$ case, we tackle the more general $\theta\neq 0$ case by using shifted Gr\"unwald-Letnikov difference operators to approach the left and right fractional derivative respectively. Let us define the parameter $-1 < \gamma := \alpha -2 <0$. Then, we can now formally define the Riesz-Feller type fractional derivative operator as follows:
\begin{align}\label{D:gamma:theta}
\mathcal{D}^{\gamma,-\theta}f(w) :=\frac{1}{2\cos(\gamma\pi/2)}\left[(1-\theta)\mathcal{I}^{-\gamma}_{+}f(w) +(1+\theta)\mathcal{I}^{-\gamma}_{-}f(w)\right]\,,
\end{align}
with 
\begin{equation}\label{eqn:frac:int}
\mathcal{I}^{-\gamma}_{\pm}f(w):=\frac{1}{\Gamma(-\gamma)}\int_{0}^{\infty}\frac{f(w\pm\xi)}{\xi^{\gamma+1}}d\xi\,.
\end{equation}
Before we proceed, we first introduce difference operators $\mathcal{A}_{h,p}^{\gamma}$ and $\mathcal{B}_{h,q}^{\gamma}$, 
where $p$ and $q$ are two non-negative integers chosen to be the shifted parameters, 
\begin{align}
& \mathcal{A}^{\gamma}_{h,p}f(w) = \frac{1}{h^{\gamma}} \sum_{k=0}^{\infty} \tilde{g}_{\gamma,k}  f(w-(k-p)h), \label{eqn:left:A}
\\
& \mathcal{B}_{h,q}^{\gamma}f(w) = \frac{1}{h^{\gamma}} \sum_{k=0}^{\infty} \tilde{g}_{\gamma,k} f\left(w + (k-q)h\right).
\label{eqn:right:B}
\end{align}
Essentially, we defined a forward shifted difference operator $\mathcal{A}_{h,p}^{\gamma}$ to approximate the left fractional derivatives, and a backward shifted difference operator $\mathcal{B}_{h,p}^{\gamma}$ to approximate the right one. 
The coefficients $\tilde{g}_{\gamma,k}:=\frac{(-1)^{k}\Gamma(-\gamma+k)}{\Gamma(k+1)\Gamma(-\gamma)}$ are from the coefficients of of the power series $(1-z)^{\gamma}$ with $-1 < \gamma < 0$ and $|z|\leq 1$. For any negative fractional number $-1<\gamma<0$ and $|z|\leq 1$, we have
\begin{align} \label{eqn:binomial}
(1-z)^{\gamma} = \sum_{k=0}^{\infty} (-1)^k\binom{-\gamma+k-1}{k}z^k,
\end{align}
where the binomial coefficient $\binom{-\gamma+k-1}{k}$ is well-defined
and the binomial series converges for any complex number $\lvert z \rvert \leq 1$;
see e.g. \citet{Kron2011}. 
Indeed, when $-1<\gamma<0$, we get
\begin{align}
\binom{-\gamma+k-1}{k} = \frac{\Gamma(-\gamma+k)}{\Gamma(k+1)\Gamma(-\gamma)} .
\end{align}

We first present the following first-order approximation result
of the fractional derivative $\mathcal{D}^{\gamma,-\theta}$.

\begin{theorem}\label{thm:approx}
Let $\mathcal{D}^{\gamma,-\theta}$ denote the fractional derivative for $-1<\gamma<0$ and $-1 < \theta < 1$
as in \eqref{D:gamma:theta}.
Suppose the function $f \in L^1(\mathbb{R}) \cap  \mathcal{C}^{4}(\mathbb{R})$. Define
\begin{align}
\Delta_{h,p,q}^{\gamma,-\theta}f(w) = \frac{1}{2\cos(\gamma\pi/2)}\left[(1+\theta)\mathcal{A}_{h,p}^{\gamma}f(w) + (1-\theta)\mathcal{B}_{h,q}^{\gamma}f(w)\right].
\end{align}
Then $\Delta_{h,p,q}^{\gamma,-\theta}f(w)$ is an approximation of $\mathcal{D}^{\gamma,-\theta}f(w)$ with the first-order accuracy:
\begin{align}
& \left\lvert \mathcal{D}^{\gamma,-\theta}f(w) - \Delta_{h,p,q}^{\gamma,-\theta}f(w) \right\rvert
\nonumber \\
& \qquad \leq \left[|p-q| + |\theta|(p+q-\gamma) + \left|\tan\left(\frac{\gamma\pi}{2}\right)\right|\left(p+q-\gamma + |\theta||p-q| \right)\right] \, \frac{C}{4\pi(|\gamma|+2)}h + \mathcal{O}\left(h^2\right) ,
\end{align}
as $h\rightarrow 0$, uniformly for all $w \in \mathbb{R}$, where $C>0$ is a constant that
may depend on $f$ and $\mathcal{O}(\cdot)$ hides the dependence on $p$, $q$ and $\gamma$.
\end{theorem}

Next, we provide an error bound for numerically computing the drift term $b(\v{w},\alpha,\theta)$ by truncating the approximation series in Theorem~\ref{thm:approx} as follows.
Let us first define
the operators $\mathcal{A}^{\gamma}_{h,p,K}$
and $\mathcal{B}^{\gamma}_{h,q,K}$:
\begin{align}
&\mathcal{A}^{\gamma}_{h,p,K}f(w):= \frac{1}{h^{\gamma}} \sum_{k=0}^{K} \tilde{g}_{\gamma,k} f\left(w - (k-p)h\right), \label{eqn:left:At}
\\
&\mathcal{B}_{h,q,K}^{\gamma}f(w) 
:= \frac{1}{h^{\gamma}}\sum_{k=0}^{K} \tilde{g}_{\gamma,k}f\left(w+(k-q)h\right), \label{eqn:right:Bt}
\end{align}
with $\tilde{g}_{\gamma,k} := \frac{(-1)^k\Gamma(-\gamma+k)}{\Gamma(k+1)\Gamma(-\gamma)}$, and $K\in\mathbb{N}\cup\{0\}$. 

Before we state the next result, let us first introduce the following assumption.

\begin{assumption} \label{H0}
Suppose the function $\partial_{w}\varphi \in L^1(\mathbb{R})\cap \mathcal{C}^4(\mathbb{R})$. In addition, there exist constants $C_p, C_q>0$ satisfying
\begin{equation} 
\lvert \partial_w \varphi(w-|k-p|h) \rvert \leq C_pe^{-|k-p|h}, \quad \lvert \partial_w \varphi(w+|k-q|h) \rvert \leq C_qe^{-|k-q|h},
\end{equation}
and $\min\left\{|k-p|,\,|k-q| \right\} > K$ for the constant $K \in \mathbb{N}\cup\{0\}$.
\end{assumption}

We have the following result.

\begin{corollary}\label{cor:K}
Suppose Assumption~\ref{H0} holds for $\partial_w \varphi$, and recall the truncated series $\mathcal{A}_{h,p,K}^{\gamma}$ and $\mathcal{B}_{h,q,K}^{\gamma}$  with $K \in \mathbb{N}\cup\{0\}$ defined in~\eqref{eqn:left:At} and~\eqref{eqn:right:Bt}. 
Let us also define the operator:
\begin{equation}
\Delta_{h,p,q,K}^{\gamma,-\theta} = \frac{1}{2\cos(\gamma\pi/2)}\left[(1-\theta)\mathcal{B}_{h,q,K}^{\gamma} +(1+\theta)\mathcal{A}_{h,p,K}^{\gamma}\right].
\end{equation}
Then the truncation error is bounded in first-order accuracy as follows,
\begin{align}
& \left| \mathcal{D}^{\gamma,-\theta}\partial_{w}\varphi(w) - \Delta_{h,p,q,K}^{\gamma,-\theta}\partial_{w}\varphi(w) \right| 
\nonumber \\
& \qquad\leq \frac{C}{4\pi(|\gamma|+2)}\left[|p-q| + |\theta|(p+q-\gamma) + \left|\tan\left(\frac{\gamma\pi}{2}\right)\right|\left(p+q-\gamma + |\theta||p-q| \right)\right] h 
\nonumber \\
& \qquad\qquad\qquad + \left((1+\theta)C_p + (1-\theta)C_q  \right)\frac{1}{hK} +  \mathcal{O}(h^2),
\end{align}
where $C, C_p, C_q>0$ are constants that may depend on $\partial_w \varphi$ and $\mathcal{O}(\cdot)$ hides the dependence on $p$ and $q$.
\end{corollary}

In particular, by taking $p=q=0$, Corollary~\ref{cor:K} implies that
\begin{align}
& \left| \mathcal{D}^{\gamma,-\theta}\partial_{w}\varphi(w) - \Delta_{h,p=0,q=0,K}^{\gamma,-\theta}\partial_{w}\varphi(w) \right| 
\nonumber \\
& \qquad\leq \frac{C}{4\pi}\left(|\theta| + \left|\tan\left(\frac{\gamma\pi}{2}\right)\right|\right) h + \left((1+\theta)C_{p=0} + (1-\theta)C_{q=0}  \right)\frac{1}{hK} +  \mathcal{O}(h^2),
\end{align}
where $\Delta_{h,K}^{\gamma,-\theta}=\Delta_{h,p=0,q=0,K}^{\gamma,-\theta}$.

Corollary~\ref{cor:K} implies that one can approximate $\mathcal{D}^{\gamma,-\theta}$ 
by the truncated $\Delta_{h,K}^{\gamma,-\theta}$
instead of $\Delta_{h}^{\gamma,-\theta}$. 
Based on this result, we are able to quantify in Theorem~\ref{thm:bias:approx} the bias induced when implementing Euler-Maruyama scheme to approximate the expectation of a test function $g$ with respect to the target distribution $\pi$, where $\nu(g) = \int g(\v{w}) \pi(d\v{w})$.

%%%%%%%%%%%%%%%%%%%%%%%%%%%%%%%%%%%%%%%%%%%%%%

%%%%%%%%%%%%%%%%%%%%%%%%%%%%%%%%%%%%%%%%%%%%%%%%%%
\section{Metastability Analysis}
\label{sec:fet_metastability}

In this section, we will focus on the metastability properties of the process
\begin{align}\label{w:1:dim}
dw_{t}=-\nabla_{w}f(w_{t})dt+\varepsilon dL_{t}^{\alpha,\theta}.
\end{align}
We will be interested in the first exit time, which is, roughly speaking, the expected time required for the process to exit a neighborhood of a local minimum. We will summarise the related theoretical results, which show that the first exit time behaviour of systems driven by asymmetric stable processes are similar to the ones of driven by symmetric stable processes. This informally implies that the process will quickly escape from narrow minima regions and will spend more time (in fact will get stuck) in wide minima regions. In this section, we make this argument rigorous. 

For simplicity of the presentation, we consider the one-dimensional case where $L_{t}^{\alpha,\theta}$ is an asymmetric $\alpha$-stable
L\'{e}vy process with L\'{e}vy measure
\begin{align}
\nu(dy)=\left(\frac{1-\theta}{2}c_{\alpha}1_{y<0}+\frac{1+\theta}{2}c_{\alpha}1_{y>0}\right)\frac{dy}{|y|^{1+\alpha}},
\end{align}
where $\theta\in(-1,1)$ and $\alpha\in(0,2)$
and $c_{\alpha}:=\frac{\alpha}{\Gamma(1-\alpha)\cos(\pi\alpha/2)}$.
Then, the left and right tails of the L\'{e}vy measure are given by
\begin{align*}
&H_{-}(-u):=\int_{(-\infty,-u)}\nu(dy)=\frac{1-\theta}{2}C_{\alpha}u^{-\alpha},
\\
&H_{+}(u):=\int_{(u,+\infty)}\nu(dy)=\frac{1+\theta}{2}C_{\alpha}u^{-\alpha},
\end{align*}
where $C_{\alpha}:=\frac{1-\alpha}{\Gamma(2-\alpha)\cos(\pi\alpha/2)}$, and 
\begin{equation*}
H(u):=H_{-}(-u)+H_{+}(u)=C_{\alpha}u^{-\alpha},\qquad\text{for any $u>0$}.
\end{equation*}
Let us assume that the function $w\mapsto f(w)$ satisfies
the following conditions:

\begin{assumption}
(i) $f\in\mathcal{C}^{1}(\mathbb{R})\cap\mathcal{C}^{3}([-K,K])$
for some $K>0$;

(ii) $f$ has exactly $n$ local minima $m_{i}$, $1\leq i\leq n$
and $n-1$ local maxima $s_{i}$, $1\leq i\leq n-1$, enumerated in increasing order 
with $s_{0}=-\infty$ and $s_{n}=+\infty$:
\begin{align}
-\infty<m_{1}<s_{1}<m_{2}<\cdots<s_{n-1}<m_{n}<+\infty.
\end{align}
All extrema of $f$ are non-degenerate, i.e. $\partial_{w}^{2}f(m_{i})>0$,
$1\leq i\leq n$, and $\partial_{w}^{2}f(s_{i})<0$, $1\leq i\leq n-1$.

(iii) $|\partial_{w}f(w)|>c_{1}|w|^{1+c_{2}}$
as $w\rightarrow\pm\infty$ for some $c_{1},c_{2}>0$.
\end{assumption}

First, we consider the first exit time from a single well.
For $\varepsilon>0$ and $\gamma>0$, define
\begin{equation}
\Omega_{\varepsilon}^{i}:=\left[s_{i-1}+2\varepsilon^{\gamma},s_{i}-2\varepsilon^{\gamma}\right],
\end{equation}
with the convention that
$\Omega_{\varepsilon}^{1}:=(-\infty,s_{1}-2\varepsilon^{\gamma}]$
and $\Omega_{\varepsilon}^{n}:=[s_{n-1}+2\varepsilon^{\gamma},+\infty)$.
The first exit time from the $i$-th well is defined as
\begin{align}
\sigma^{i}(\varepsilon;\theta):=\inf\{t\geq 0:w_{t}\notin[s_{i-1}+\varepsilon^{\gamma},s_{i}-\varepsilon^{\gamma}]\},
\end{align}
for $i=1,\ldots,n$. 
Let us also define
\begin{align}
\lambda^{i}(\varepsilon;\theta)
:=\frac{1-\theta}{2}C_{\alpha}\left|\frac{s_{i-1}-m_{i}}{\varepsilon}\right|^{-\alpha}
+\frac{1+\theta}{2}C_{\alpha}\left|\frac{s_{i}-m_{i}}{\varepsilon}\right|^{-\alpha},
\end{align}
for $i=1,\ldots,n$.
We have the following first exit time result from \citet{IP2008}.

\begin{proposition}[Proposition~3.1. in \cite{IP2008}]
There exists $\gamma_{0}>0$ such that for any $0<\gamma\leq\gamma_{0}$, $i=1,2,\ldots,n$,
\begin{align}
\lambda^{i}(\varepsilon;\theta)\sigma^{i}(\varepsilon;\theta)\rightarrow\exp(1),\qquad\text{in distribution as $\varepsilon\rightarrow 0$},
\end{align}
where $\exp(1)$ denotes the exponential distribution with mean $1$, and
\begin{align}
\lim_{\varepsilon\rightarrow 0}\mathbb{E}_{w}\left[\lambda^{i}(\varepsilon;\theta)\sigma^{i}(\varepsilon;\theta)\right]=1,
\end{align}
where the limit holds uniformly over $w\in\Omega_{\varepsilon}^{i}$. 
\end{proposition}

The above result implies that as $\varepsilon\rightarrow 0$,
\begin{align}
\mathbb{E}_{w}\left[\sigma^{i}(\varepsilon;\theta)\right]
\sim
\Bigg(\frac{1-\theta}{2}C_{\alpha}\left|s_{i-1}-m_{i}\right|^{-\alpha}
+\frac{1+\theta}{2}C_{\alpha}\left|s_{i}-m_{i}\right|^{-\alpha}\Bigg)^{-1}\varepsilon^{-\alpha}.
\end{align}
If $|s_{i}-m_{i}|>|s_{i-1}-m_{i}|$, 
i.e. the $i$-th well is asymmetric
and the local minimum $m_{i}$  is closer
to the saddle point on the left $s_{i-1}$ than
the saddle point on the right $s_{i}$, 
then $\lambda^{i}(\varepsilon;\theta)<\lambda^{i}(\varepsilon;0)$
and $\mathbb{E}_{w}[\sigma^{i}(\varepsilon;\theta)]>\mathbb{E}_{w}[\sigma^{i}(\varepsilon;0)]$
for positive $\theta$
and $\lambda^{i}(\varepsilon;\theta)>\lambda^{i}(\varepsilon;0)$
and $\mathbb{E}_{w}[\sigma^{i}(\varepsilon;\theta)]<\mathbb{E}_{w}[\sigma^{i}(\varepsilon;0)]$
for negative $\theta$.
Similarly, if $|s_{i}-m_{i}|<|s_{i-1}-m_{i}|$, 
i.e. the $i$-th well is asymmetric
and the local minimum $m_{i}$  is closer
to the saddle point on the right $s_{i}$ than
the saddle point on the left $s_{i-1}$, 
then $\lambda^{i}(\varepsilon;\theta)>\lambda^{i}(\varepsilon;0)$
and $\mathbb{E}_{w}[\sigma^{i}(\varepsilon;\theta)]<\mathbb{E}_{w}[\sigma^{i}(\varepsilon;0)]$
for positive $\theta$
and $\lambda^{i}(\varepsilon;\theta)<\lambda^{i}(\varepsilon;0)$
and $\mathbb{E}_{w}[\sigma^{i}(\varepsilon;\theta)]>\mathbb{E}_{w}[\sigma^{i}(\varepsilon;0)]$
for negative $\theta$.
The intuition is that when the well is asymmetric,
the dynamics can exit the well faster
when there is a skewness $\theta$ towards
the the saddle point closer to the minimum of the well.

Next, we consider transitions between the wells.
For any $0<\Delta<\Delta_{0}:=\min_{1\leq i\leq n}\{|m_{i}-s_{i-1}|,|m_{i}-s_{i}|\}$
and $w\in\mathbb{R}$ denote $B_{\Delta}(w):=\{v:|w-v|\leq\Delta\}$.
Define
\begin{equation}
\tau^{i}(\varepsilon;\theta):=\inf\left\{t\geq 0:w_{t}\in\cup_{k\neq i}B_{\Delta}(m_{k})\right\}.
\end{equation}
Then, we have the following result about transitions between
the wells from \citet{IP2008}.

\begin{proposition}[Proposition~4.3. in \cite{IP2008}]
For any $0<\Delta<\Delta_{0}$ and $j\neq i$
\begin{align}\label{transit:1}
\lim_{\varepsilon\rightarrow 0}\mathbb{P}_{w}\left(w_{\tau^{i}(\varepsilon;\theta)}\in B_{\Delta}(m_{j})\right)
=\frac{q_{ij}}{q_{i}},
\end{align}
uniformly for $w\in B_{\Delta}(m_{i})$, $i=1,\ldots,n$, where
\begin{align}
&q_{ij}=\left(\frac{1-\theta}{2}1_{j<i}+\frac{1+\theta}{2}1_{j>i}\right)
\cdot
\left\|s_{j-1}-m_{i}|^{-\alpha}-|s_{j}-m_{i}|^{-\alpha}\right|,\quad i\neq j,\label{eqn:q:i:j}
\\
&-q_{ii}=q_{i}=\sum_{j\neq i}q_{ij}
=\frac{1-\theta}{2}|s_{i-1}-m_{i}|^{-\alpha}+\frac{1+\theta}{2}|s_{i}-m_{i}|^{-\alpha}.\label{transit:2}
\end{align}
\end{proposition}

From \eqref{transit:1}-\eqref{transit:2}, 
we can compute that
\begin{align}
\frac{q_{ij}}{q_{i}}
=
\begin{cases}
\frac{\left\|s_{j-1}-m_{i}|^{-\alpha}-|s_{j}-m_{i}|^{-\alpha}\right|}{\frac{1-\theta}{1+\theta}|s_{i-1}-m_{i}|^{-\alpha}+|s_{i}-m_{i}|^{-\alpha}} &\text{if $j>i$},
\\
\frac{\left\|s_{j-1}-m_{i}|^{-\alpha}-|s_{j}-m_{i}|^{-\alpha}\right|}{|s_{i-1}-m_{i}|^{-\alpha}+\frac{1+\theta}{1-\theta}|s_{i}-m_{i}|^{-\alpha}} &\text{if $j<i$}.
\end{cases}
\end{align}
Therefore, $q_{ij}/q_{i}$ is increasing in $\theta$
for $j>i$ and decreasing in $\theta$ for $j<i$.
This is consistent with the intuition that
when $\theta>0$, it is more likely for the dynamics
to transit to a well on the right side,
and when $\theta<0$, it is more likely
for the dynamics to transit to a well on the left side.

Next, we consider the following metastability result due to Theorem~1.1. in \citet{IP2008}. 
It describes the metastability phenomenon, which basically says that there exists a time
scale under which the system behaves like a continuous time Markov process 
with a finite state space consisting of values in the set of stable attractors.

\begin{theorem}[Theorem~1.1. in \cite{IP2008}]
If $w_{0}=w\in(s_{i-1},s_{i})$ for some $i=1,2,\ldots,n$, then
for any $t>0$, in the sense of finite-dimensional distributions, 
\begin{align}
w_{t/H(1/\varepsilon)}\rightarrow Y_{t}(m_{i}),\qquad\text{as $\varepsilon\rightarrow 0$},
\end{align}
where $w_{t}$ is defined in \eqref{w:1:dim} and $H(1/\varepsilon)=C_{\alpha}\varepsilon^{\alpha}$, 
where $Y_{t}(m_{i})$ that starts at $m_{i}$ is a continuous-time Markov process  
on a finite states space
$\{m_{1},\ldots,m_{n}\}$ with the infinitesimal generator $Q=(q_{ij})_{i,j=1}^{n}$,
where $q_{ij}$ is defined in \eqref{eqn:q:i:j}.
\end{theorem}

The Markov process $Y_{t}(m_{i})$ admits a unique invariant distribution
$\pi$ satisfying $Q^{T}\pi=0$. 
In the case of double well, i.e. $n=2$
and $m_{1}<s_{1}=0<m_{2}$ separated by a local maximum
at $s_{1}=0$, where without loss of generality we assume
that $m_{2}>|m_{1}|$, i.e. the second local minimum lies
in a wider valley. A simple calculation yields that
\begin{equation}
q_{12}=\frac{1+\theta}{2}\frac{1}{|m_{1}|^{\alpha}}=-q_{11},
\quad\text{and}\quad q_{21}=\frac{1-\theta}{2}\frac{1}{m_{2}^{\alpha}}=-q_{22},
\end{equation}
so that it follows from $Q^{T}\pi=0$ and $\pi_{1}+\pi_{2}=1$ that
\begin{align}
&\pi_{1}=\frac{(1+\theta)^{-1}|m_{1}|^{\alpha}}{(1+\theta)^{-1}|m_{1}|^{\alpha}+(1-\theta)^{-1}m_{2}^{\alpha}},
\\
&\pi_{2}=\frac{(1-\theta)^{-1}m_{2}^{\alpha}}{(1+\theta)^{-1}|m_{1}|^{\alpha}+(1-\theta)^{-1}m_{2}^{\alpha}}.
\end{align}
In particular, the ratio $\frac{\pi_{2}}{\pi_{1}}=\frac{1+\theta}{1-\theta}\left(\frac{m_{2}}{|m_{1}|}\right)^{\alpha}$
is increasing in $\frac{m_{2}}{|m_{1}|}$ and $\theta$. 
That reveals that in the equilibrium the process will spend more
time in the second valley if the second valley is wide
and there is a drift towards to the right.
In the symmetric case, i.e. $\theta=0$, $\pi_{2}>\pi_{1}$
since $m_{2}>|m_{1}|$
so that in the equilibrium the process spends more time in the wider valley.
In the asymmetric case, i.e. $\theta\neq 0$, if there is
a strong skewness towards the left, i.e. $\theta<0$ and $|\theta|$
is large, then in the equilibrium the process may spend more
time in the narrower valley. Indeed $\pi_{2}>\pi_{1}$
if and only if $\theta>\frac{|m_{1}|^{\alpha}-m_{2}^{\alpha}}{|m_{1}|^{\alpha}+m_{2}^{\alpha}}$.

%%%%%%%%%%%%%%%%%%%%%%%%%%%%%%%%%%%%%%%%%%%%

\section{Postponed Proofs}\label{sec:technical:proofs}

\subsection{Proof of Lemma~\ref{prop:accum_tail_properties}}

Before we proceed to the proof of Lemma~\ref{prop:accum_tail_properties} we present some intermediary results that are required for the proof.

\begin{definition}{(Asymptotic order of magnitude)}
A positive sequence $a_m$ is of the same order of magnitude as another positive sequence $b_m$ ($a_m \asymp b_m$, i.e. `asymptotically equivalent') if there exist some $c,C>0$ such that: $c \leq \frac{a_m}{b_m} \leq C$ for any $m \in \mathbb{N}$.
\end{definition}

\begin{lemma}[Lemma A.1 in \citep{Vladimirova2019UnderstandingLevel}]\label{lem:gaussianmoments}
Let $X$ be a normal random variable such that $X\sim \mathcal{N}(0,\sigma^2)$. Then the following asymptotic equivalence holds\[\|X\|_m \asymp \sqrt{m}.\]
\end{lemma}

We know that the centering of variables does not change their tail properties \citep{vershynin_2018, Kuchibhotla2018}. 
As such Lemma~\ref{lem:gaussianmoments} also applies to $X\sim \mathcal{N}(\mu,\sigma^2)$, 
as $\|X\| \asymp \|X-\mu\| \asymp \sqrt{m}$.

\begin{lemma}[Lemma~3.1 of \citet{Vladimirova2019UnderstandingLevel}]\label{prop:moment-preservation}

    Let $\kappa: \mathbb{R} \to \mathbb{R} $ be a non-linearity that obeys the extended envelope property. 
    And let $X$ be a variable for which $\|X_+\|_m \asymp \|X_-\|_m$ where $X_-$ and $X_+$ denote the left and right tail of the variable respectively  \footnote{We weaken \citet{Vladimirova2019UnderstandingLevel}'s requirement for $X$ to be symmetric as the proof they give still holds here.}. 
    Then we have: 
\begin{align}
\|\kappa(X)\|_m \asymp \|X\|_m\,,\qquad \text{for any  $m \geq 1$}\,.
\end{align}
\end{lemma}

\begin{lemma}\label{lem:expmomentssum}
Let $X_1, \dots, X_N$ be variables that each obeys $\|X_i\|_m \lesssim m^r, p\in R,  i=1,\dots,N$ and $(W_i, \dots, W_N) \in \mathbb{R}^N$.
\[
\left\Vert\sum_{i=1}^N W_iX_i \right\Vert_m  \lesssim m^r\,.
\]
\end{lemma}

\begin{proof}[Proof of Lemma~\ref{lem:expmomentssum}]
By Minkowski's inequality we have that 
\begin{align*}
    \left\Vert\sum_{i=1}^N W_iX_i \right\Vert_m &\leq \sum_{i=1}^N \|W_iX_i\|_m \leq  \sum_{l=1}^{N} |W_iA_i|m^r, \ (A_1,\dots,A_{N}) \in \mathbb{R}^{N}\\
    \Leftrightarrow  \left\Vert\sum_{i=1}^N W_iX_i \right\Vert_m &\lesssim m^r\,.
\end{align*}

The $A_i$ here are constants that upper bound the asymptotics of each norm $\|X_i\|_m$ in the sum. 
\end{proof}

\begin{proof}[Proof of Lemma~\ref{prop:accum_tail_properties}]
\textbf{Additive Noise.} Consider first the noised data, $\widetilde{\v{h}}_0(\v{x}) = \v{x} + \v{\epsilon}_0, \v{\epsilon}_0 \sim \mathcal{N}(0,\sigma_0^2)$. As Lemma~\ref{lem:gaussianmoments} shows, Gaussian random variables have an $m^{\mathrm{th}}$ norm that is asymptotically equivalent to $\sqrt{m}$,
\[\left\Vert\widetilde{h}_{0,l}(\v{x})\right\Vert_m \asymp  \sqrt{m}, \qquad\text{for any $l = 1, \dots , n_0$}\,,\]
where $n_0$ is the dimensionality of data.

Let us now assume that $\left\Vert\widetilde{h}_{i,l}(\v{x})\right\Vert_m \lesssim \sqrt{m}$, for any $l = 1, \dots , n_i$, for some layer $i$. 
The pre-non-linearity at this layer is given by $\tilde{\v{g}}_i = \v{W}_{i+1}\widetilde{\v{h}}_i$. The $j^\mathrm{th}$ element of $\tilde{\v{g}}_i$ is defined as a sum,
\[
g_{i,j}(\v{x}) = \sum_{l=1}^{n_i} W_{i+1,l,j}\widetilde{h}_{i,l}(\v{x})\,,
\]
where $W_{i+1,l,j}$ is the weight that maps from the $l^\mathrm{th}$ neuron in layer $i$ to the $j^\mathrm{th}$ in layer $i+1$. 
By Lemma~\ref{lem:expmomentssum}, 
\begin{align*}
    \left\Vert g_{i,j}(\v{x})\right\Vert_m &\lesssim \sqrt{m}, 
    \qquad m=1,\dots, n_i\,.
\end{align*}
As such if we assume the non-linearities $\phi$ at each layer obey the extended envelope property, then we have by Lemma~\ref{prop:moment-preservation}: 
\begin{align*}
   \left\Vert \phi(g_{i,j}(\v{x})) \right\Vert_m = \left\Vert \widehat{h}_{i+1,j}(\v{x}) \right\Vert_m &\asymp
\left\Vert \kappa(g_{i,j}(\v{x}))\right\Vert_m \\
\Leftrightarrow \left\Vert\widehat{h}_{i+1,j}(\v{x})\right\Vert_m &\lesssim \sqrt{m}, \qquad j=1,\dots, n_{i+1}\,. \nonumber
\end{align*}

Note that $\widetilde{\v{h}}_{i+1} = \widehat{\v{h}}_{i+1} + \v{\epsilon}_{i+1}, \v{\epsilon}_{i+1} \sim \mathcal{N}(0, \sigma_{i+1}^2)$. 
By Lemma~\ref{lem:expmomentssum}, once again $\|\widetilde{h}_{i+1,j}(\v{x})\|_m \lesssim \sqrt{m}$, is Gaussian in its tails. 
By recursion, with $\widetilde{\v{h}}_{0}$ as the base case, we have that 
\begin{align}
        \left\Vert \widetilde{h}_{i,l}(\v{x})\right\Vert _m \lesssim \sqrt{m},  \qquad \text{for any $m \geq 1$}; \ i = 1,\dots, L-1; \ l=1,\dots, n_i\,. \nonumber
\end{align}

\textbf{Multiplicative Noise.} Consider first the noised data, $\widetilde{\v{h}}_0(\v{x}) = \v{x} \circ \v{\epsilon}_0, \v{\epsilon}_0 \sim \mathcal{N}(1,\sigma_0^2)$. As Lemma~\ref{lem:gaussianmoments} shows, Gaussian random variables have an $m^{\mathrm{th}}$ norm that is asymptotically equivalent to $\sqrt{m}$,
\[\left\Vert \widetilde{h}_{0,l}(\v{x}) \right\Vert_m \asymp  \sqrt{m}, \qquad\text{for any $l = 1, \dots , n_0$}\,,\]
where $n_0$ is the dimensionality of data.

Let us now assume that $\left\Vert \widetilde{h}_{i,l}(\v{x}) \right\Vert_m \lesssim m^r, \ \forall l = 1, \dots , n_i$, for some layer $i$. 
The pre-non-linearity at this layer is given by $\tilde{\v{g}}_i = \v{W}_{i+1}\widetilde{\v{h}}_i$. The $j^\mathrm{th}$ element of $\tilde{\v{g}}_i$ is defined as a sum,
\[
g_{i,j}(\v{x}) = \sum_{l=1}^{n_i} W_{i+1,l,j}\widetilde{h}_{i,l}(\v{x})\,,
\]
where $W_{i+1,l,j}$ is the weight that maps from the $l^\mathrm{th}$ neuron in layer $i$ to the $j^\mathrm{th}$ in layer $i+1$. 
By Lemma~\ref{lem:expmomentssum}, 
\begin{align*}
    \left\Vert g_{i,j}(\v{x}) \right\Vert_m &\lesssim m^r, \qquad m=1,\dots, n_i\,.
\end{align*}

As such if we assume the non-linearities $\phi$ at each layer obey the extended envelope property, then we have by Lemma~\ref{prop:moment-preservation}: 
\begin{align*}
    \left\Vert \phi(g_{i,j}(\v{x})) \right\Vert_m = \left\Vert \widehat{h}_{i+1,j}(\v{x}) \right\Vert_m &\asymp
\left\Vert \kappa(g_{i,j}(\v{x})) \right\Vert_m \\
\Leftrightarrow \left\Vert \widehat{h}_{i+1,j}(\v{x}) \right\Vert_m &\lesssim m^r, \qquad j=1,\dots, n_{i+1}\,. \nonumber
\end{align*}

Note that $\widetilde{\v{h}}_{i+1} = \widehat{\v{h}}_{i+1} \circ \v{\epsilon}_{i+1}, \v{\epsilon}_{i+1} \sim \mathcal{N}(1, \sigma_{i+1}^2)$. 
By H\"{o}lder's inequality we have that, 
\[
\left\Vert\widetilde{\v{h}}_{i+1,j}(\v{x})\right\Vert_m \lesssim m^{r+\frac{1}{2}},\qquad  j=1,\dots, n_{i+1}\,.
\]

By recursion, with $\widetilde{\v{h}}_{0}$ as the base case, we have that 
\begin{align}
        \left\Vert \widetilde{h}_{i,l}(\v{x}) \right\Vert_m \lesssim m^{\frac{i+1}{2}},  \qquad \text{for any $m \geq 1$}; \ i = 1,\dots, L-1; \ l=1,\dots, n_i\,. \nonumber
\end{align}

\end{proof}

%%%%%%%%%%%%%%%%%%%%%%%%%%%%%%%%%%%%%%%%%%%%%%%%%%%%%%%%%%
\subsection{Proof of Theorem~\ref{thm:gn_tail_properties}}

Before we proceed to the proof of Theorem~\ref{thm:gn_tail_properties} we present some intermediary results that are required for the proof.

\begin{lemma}\label{prop:boundedmoments}
    Let $X$ be a bounded random variable such that $|X| \leq C$,
    then $X$ is sub-Weibull with parameter $\theta=0$,
    \begin{align}
        \left\Vert X \right\Vert_m
        \lesssim m^{0},  \qquad\text{for every $m \geq 1$}.
    \end{align}
\end{lemma}
\begin{proof}[Proof of Lemma~\ref{prop:boundedmoments}]
The moments of $X$ obey 

\begin{align*}
    \expect\left[|X|^m\right] \leq C^k\,.
\end{align*}
Taking the root of this we find that $\|X\|_m := \expect[|X|^m]^{\frac{1}{m}}$ is not dependent on $m$ and scales as a constant, $\left\Vert X \right\Vert_m \asymp m^{0}$. 
\end{proof}

\begin{proof}[Proof of Theorem~\ref{thm:gn_tail_properties}]

We first consider the gradient for a single data-label pair of $W_{i,l,j}$ the weight that maps from neuron $l$ in layer $i-1$ to neuron $j$ in layer $i$.
We study this gradient using the chain rule, where we decompose $\partial E_{\mathcal{L}}(\cdot)/\partial W_{i,l,j}$ as  \[\frac{\partial E_{\mathcal{L}}(\cdot)}{\partial \widetilde{h}_{i,j}} \cdot \frac{\partial \widetilde{h}_{i,j}}{ \partial W_{i,l,j}}, \] 
where $\widetilde{h}_{i,j}$ is the (noised) activation of the $j^\mathrm{th}$ neuron in the $i^\mathrm{th}$ layer. 
Thus, the gradient noise on the weights can be described as the product of two random variables.

\textbf{Additive Noise.} Let us first consider the properties of $\partial E_{\mathcal{L}}(\cdot)/\partial \widetilde{h}_{i,j}$ for the additive case. 

\textit{Regression}. In the case of regression we use a mean-square-error (MSE) we have that: 
\[
\Delta\mathcal{L}(\v{x}, \v{y}) = 2(\v{y}-\v{h}_L(\v{x}))\mathcal{E}_L(\v{x}) + \left(\mathcal{E}_L(\v{x})\right)^2\,,
\]
where we imply all terms' dependence on $\v{w}, \v{\epsilon}$ for brevity of notation. 
One can already see that the derivative of this object with respect to each element of $\widetilde{\v{h}}_{L}$ will have tail properties that are asymptotically equivalent to those of $\mathcal{E}_L$, which we know by Lemma~\ref{prop:accum_tail_properties}.  
\[
\left\Vert  \frac{\partial \Delta\mathcal{L}(\v{x}, \v{y})}{\partial \widetilde{h}_{L,j}}\right\Vert_m \lesssim \sqrt{m}, \qquad m = 1, \dots, n_L\,.
\]
If we center this distribution the tail properties of this variable are unchanged \citep{vershynin_2018, Kuchibhotla2018}. 
In particular the asymptotic behaviour of $\|\cdot\|_m$ is unchanged and we have that: 
\begin{align*}
    \left\Vert \frac{\partial \Delta\mathcal{L}(\v{x}, \v{y})}{\partial \widetilde{h}_{L,j}} - \expect_{\v{\epsilon}} \left[\frac{\partial \Delta\mathcal{L}(\v{x}, \v{y})}{\partial \widetilde{h}_{L,j}}\right]
 \right\Vert_m &\lesssim \sqrt{m},  \\
\Leftrightarrow \left\Vert
\frac{\partial E_{\mathcal{L}}(\v{x}, \v{y})}{\partial \widetilde{h}_{L,j}}
 \right\Vert_m &\lesssim \sqrt{m}, \qquad m = 1, \dots, n_L\,.
\end{align*}

\textit{Classification.} In the case of classification we use a cross-entropy (CE) error. 
There is no easy closed-form for $\Delta\mathcal{L}(\v{x}, \v{y})$ here, but we can infer the properties of $ \nabla_{\widetilde{\v{h}}_{L}} \Delta\mathcal{L}(\v{x}, \v{y})$ by studying the properties of the gradient 
$\nabla_{\v{h}_{L}} \mathcal{L}(\v{x}, \v{y})$. 
For CE we know that: 
\[
\nabla_{\v{h}_{L}} \mathcal{L}(\v{x}, \v{y}) = \mathrm{sigmoid}(\v{h}_{L}(\v{x})) - \v{y} 
\]
in the binary label case. In the multi-label classification case we typically use a $\mathrm{softmax}$, which is also a bounded function. 
We can already see that any noise $\mathcal{E}_L$ added to $\v{h}_{L}(\v{x})$ will induce a change in the gradient that is inherently bounded by the $\mathrm{sigmoid}$ non-linearity, meaning that $\Delta\mathcal{L}(\v{x}, \v{y})$ will be bounded. 
As such the centered variable $E_{\mathcal{L}}(\v{x}, \v{y})$ will also be bounded \textit{and} zero mean,. 
By Lemma~\ref{prop:boundedmoments} any bounded and zero mean distribution will be sub-Weibull with parameter $\theta=0$, and will thus \textit{also} be sub-Gaussian
\[
\left\Vert \frac{\partial E_{\mathcal{L}}(\v{x}, \v{y})}{\partial \widetilde{h}_{L,j}}
 \right\Vert_m \lesssim \sqrt{m}, \qquad m = 1, \dots, n_L\,.
\]

Synthesizing the regression and classification settings we can conclude that each constitutive element of $\nabla_{\widetilde{\v{h}}_{L}} E_{\mathcal{L}}(\v{x}, \v{y})$ will be sub-Gaussian and will have zero mean.

We can now turn to the partial derivatives of the form $\partial E_{\mathcal{L}}(\cdot)/\partial \widetilde{h}_{i}^{m}$.
Assume $\nabla_{\widetilde{\v{h}}_{i}} E_{\mathcal{L}}(\v{x}, \v{y})$ is of order $m^r$, 
with, 
\[
\left\Vert \frac{\partial E_{\mathcal{L}}(\v{x}, \v{y})}{\partial \widetilde{h}_{i,j}}
 \right\Vert_m \lesssim m^r, \qquad m = 1, \dots, n_i\,,
\]
which entails for gradients at the previous $(i-1)^\mathrm{th}$ layer we have 
\[
\frac{\partial E_{\mathcal{L}}(\cdot)}{\partial \widetilde{h}_{i-1,l}} = \nabla_{\widetilde{\v{h}}_{i}} E_{\mathcal{L}}(\v{x}, \v{y})
\nabla_{\widetilde{h}_{i-1,l}} \widetilde{\v{h}}_{i}(\v{x})\,,
\]
where 
\[
\nabla_{\widetilde{h}_{i-1,l}} \widetilde{\v{h}}_{i}(\v{x}) =  \kappa'\left(\v{W}_{i, l}\widetilde{h}_{i-1,l}(\v{x})\right) \circ (\v{W}_{i, l})\,, 
\]
where $\circ$ denotes the element wise product and $\v{W}_{i,l}$ is the $l^\mathrm{th}$ column of the weight matrix $\v{W}_i$. 
By definition, activation functions that obey the extended envelope property will have gradients that are bounded in norm, by some constant $d_2$. 
As such, by Lemma~\ref{prop:boundedmoments} $\kappa'$ will be sub-Weibull with $r=0$. 
By H\"{o}lder's inequality we have that \[
\left\|\left(\nabla_{\widetilde{\v{h}}_{i}} E_{\mathcal{L}}(\v{x}, \v{y})\right)_z
\left(\kappa'\left(\v{W}_{i, l}\widetilde{h}_{i-1,l}(\v{x})\right)\right)_j\right\|_m \leq \left\|\left(\nabla_{\widetilde{\v{h}}_{i}} E_{\mathcal{L}}(\v{x}, \v{y})\right)_j\right\|_{2m}\left\|\left(\kappa'\left(\v{W}_{i, l}\widetilde{h}_{i-1,l}(\v{x})\right)\right)_j\right\|_{2m}\,,
\]
where $j$ indexes over the elements of both Jacobians. By definition, we know that there exists $A > 0$ and $B > 0$ such that $\|(\nabla_{\widetilde{\v{h}}_{i}} E_{\mathcal{L}}(\v{x}, \v{y}))_z\|_{2m} \leq A (2m)^{p}$ and $\|(\kappa'(\v{W}_{i, l}\widetilde{h}_{i-1,l}(\v{x})))_j\|_{2m} \leq B$. 
As such \[
\left\|\left(\nabla_{\widetilde{\v{h}}_{i}} E_{\mathcal{L}}(\v{x}, \v{y})\right)_z
\left(\kappa'\left(\v{W}_{i, l}\widetilde{h}_{i-1,l}(\v{x})\right)\right)_j\right\|_m \leq \left\|\left(\nabla_{\widetilde{\v{h}}_{i}} E_{\mathcal{L}}(\v{x}, \v{y})\right)_j\right\|_{2m}\left\|\left(\kappa'\left(\v{W}_{i, l}\widetilde{h}_{i-1,l}(\v{x})\right)\right)_j\right\|_{2m} \leq AB2^pm^r\,.
\]
Thus we know that the product of these two variables will be asymptotically upper-bounded by
\[
\left\Vert \left(
\nabla_{\widetilde{\v{h}}_{i}} E_{\mathcal{L}}(\v{x}, \v{y})\right)_j
\left(\kappa'\left(\v{W}_{i, l}\widetilde{h}_{i-1,l}(\v{x})\right)\right)_j \right\Vert_m \lesssim m^r\,.
\]
We now need to take into account the weighted sum across rows and columns (i.e. over indices $j$) that occurs. 
By Lemma~\ref{lem:expmomentssum} we know that

\begin{align*}
\left\Vert
\sum_{j=1}^{n_i}
\left(\nabla_{\widetilde{\v{h}}_{i}} E_{\mathcal{L}}(\v{x}, \v{y})\right)_j
\left(\kappa'\left(\v{W}_{i, l}\widetilde{h}_{i-1,l}(\v{x})\right)\circ (\v{W}_{i, l})\right)_j \right\Vert_m \lesssim  m^r, \\ 
\Leftrightarrow 
  \left\Vert \frac{\partial E_{\mathcal{L}}(\cdot)}{\partial \widetilde{h}_{i-1,l} }  \right\Vert_m \lesssim m^r, \qquad m = 1, \dots, n_i\,.
\end{align*}

By recursion, with $\widetilde{\v{h}}_L$ as the base case, gradients at layer $i$ bounded in norm by $m^r$ induce gradients at layer $i-1$, also bounded in norm by $m^{r}$. 
By recursion with the $L^\mathrm{th}$ layer as the base case we can claim that,
\begin{align*}
  \left\Vert \frac{\partial E_{\mathcal{L}}(\cdot)}{\partial \widetilde{h}_{i,j}}  \right\Vert_m \lesssim \sqrt{m}\,.
\end{align*}
We have now defined the first constitutive term of  $\partial E_{\mathcal{L}}(\cdot)/\partial W_{i,l,j}$. 
Defining $\partial \widetilde{h}_{i,j} / \partial W_{i,l,j}$ is much simpler:
\[
\frac{\partial \widetilde{h}_{i,j}}{ \partial W_{i,l,j}} =  \kappa'\left(W_{i,l,j}\widetilde{h}_{i-1,l}(\v{x})\right) \left(\widetilde{h}_{i-1,l}(\v{x})\right)\,.
\]
Here $\widetilde{h}_{i-1,l}(\v{x})$, which we know is sub-Gaussian by Lemma~\ref{prop:accum_tail_properties}, is once again multiplied to a bounded variable, $\kappa'$. 
Thus reapplying H\"{o}lder's inequality we obtain that 
\[
\left\Vert \frac{\partial \widetilde{h}_{i,j}}{ \partial W_{i,l,j}}\right\Vert_m  \lesssim \sqrt{m}\,.
\]

We can now bring together the characterisations of the gradients that constitute $\partial E_{\mathcal{L}}(\cdot)/\partial W_{i,l,j}$.
We can re-use H\"{o}lder's inequality to show that the product of these variables will be sub-exponential
\[
\left\Vert\frac{\partial E_{\mathcal{L}}((\v{x}, \v{y}); \v{w}, \v{\epsilon})}{\partial W_{i,l,j}}\right\Vert_m
        \lesssim m,  \qquad\text{for every $m \geq 1$}\,.
\]

%%%%%%%%%%%%%%%%%%%%%%%%%%%%%%%%%%%%
\textbf{Multiplicative Noise.} In the case of multiplicative noise we know that by Lemma~\ref{prop:accum_tail_properties}, the accumulated noise at layer $L$ will be of the same order as that at layer $L-1$, because we are not multiplying noise to the final layer, thus 
\[
\left\Vert \frac{\partial E_{\mathcal{L}}(\v{x}, \v{y})}{\partial \widetilde{h}_{L,j}}
 \right\Vert_m \lesssim m^\frac{{L}}{2}, \qquad m = 1, \dots, n_L\,.
\]
Repeating the analysis done for the additive case we can claim that,
\begin{align*}
  \left\Vert \frac{\partial E_{\mathcal{L}}(\cdot)}{\partial \widetilde{h}_{i}^{m}}  \right\Vert_m \lesssim m^\frac{{L}}{2}.
\end{align*}
We have now defined the first constitutive term of  $\partial E_{\mathcal{L}}(\cdot)/\partial W_{i,l,j}$. 
We now define $\partial \widetilde{h}_{i,j} / \partial W_{i,l,j}$.
\begin{equation*}
\frac{\partial \widetilde{h}_{i}^{m}}{ \partial W_{i,l,j}} =  \kappa'\left(W_{i,l,j}\widetilde{h}_{i-1,l}(\v{x})\right) \left(\widetilde{h}_{i-1,l}(\v{x})\right)\,. 
\end{equation*}
Here $\widetilde{h}_{i-1,l}(\v{x})$, which we know is sub-Weibull with a parameter $p=\frac{i}{2}$ by Lemma~\ref{prop:accum_tail_properties}, is once again multiplied to a bounded variable, $\kappa'$. 
Thus reapplying H\"{o}lder's inequality we obtain that 
\[
\left\Vert \frac{\partial \widetilde{h}_{i}^{m}}{ \partial W_{i,l,j}}\right\Vert_m  \lesssim m^\frac{i}{2}. 
\]
We can re-use H\"{o}lder's inequality to show that the product of these variables will be sub-Weibull with $p=\frac{L+i}{2}$
\[
\left\Vert\frac{\partial E_{\mathcal{L}}((\v{x}, \v{y}); \v{w}, \v{\epsilon})}{\partial W_{i,l,j}}\right\Vert_m
        \lesssim m^{\frac{L+i}{2}},  \qquad\text{for every $m \geq 1$}\,.
\]

\textbf{Mean of Noise.} Finally, by definition these gradients are zero mean, \[\frac{\partial E_{\mathcal{L}}((\v{x}, \v{y}); \v{w}, \v{\epsilon})}{\partial W_{i,l,j}} = \frac{\partial \Delta\mathcal{L}(\v{x}, \v{y})}{\partial W_{i,l,j}} - \expect_{\v{\epsilon}} \left[\frac{\partial \Delta\mathcal{L}(\v{x}, \v{y})}{\partial W_{i,l,j}}\right]. \]
\end{proof}

%%%%%%%%%%%%%%%%%%%%%%%%%%%%%%%%%%%%%%%%%%%%%
\subsection{Proof of Theorem~\ref{thm:AFLD}}

Before we proceed to the proof of Theorem~\ref{thm:AFLD}, we present
some technical results that will be used
in the proof of Theorem~\ref{thm:AFLD} later.

For the $d$-dimensional asymmetric fractional Langevin dynamics $\v{w}_{t}$,
its infinitesimal generator is given in the following proposition.

\begin{proposition}\label{prop:generator}
The asymmetric fractional Langevin dynamics $\v{w}_{t}$ 
has the infinitesimal generator:
\begin{align}
\mathcal{L}f(\v{w})=\sum_{i=1}^{d}
\left((b(\v{w},\alpha,\theta))_{i}-\varepsilon^{\alpha}\frac{\theta_{i}}{\cos(\alpha\pi/2)}\frac{\alpha}{(\alpha-1)\Gamma(1-\alpha)}\right)\frac{\partial f(\v{w})}{\partial w_{i}}
+\varepsilon^{\alpha}\sum_{i=1}^{d}\mathcal{H}^{\alpha,\theta_{i}}_{w_{i}}f(\v{w}),
\end{align}
where
\begin{align}
\mathcal{H}^{\alpha,\theta_{i}}_{w_{i}}f(\v{w})
&:=\left(\frac{1}{2}+\frac{\theta_{i}}{2}\right)\frac{1}{\cos(\alpha\pi/2)}
\frac{\alpha}{\Gamma(1-\alpha)}\int_{0}^{\infty}\frac{f(\v{w}+\xi\mathbf{e}_{i})-f(\v{w})-\partial_{w_{i}}f(\v{w})\xi}{\xi^{\alpha+1}}d\xi
\nonumber
\\
&\qquad
+\left(\frac{1}{2}-\frac{\theta_{i}}{2}\right)\frac{1}{\cos(\alpha\pi/2)}\frac{\alpha}{\Gamma(1-\alpha)}\int_{0}^{\infty}\frac{f(\v{w}-\xi\mathbf{e}_{i})-f(\v{w})+\partial_{w_{i}}f(\v{w})\xi}{\xi^{\alpha+1}}d\xi,\label{defn:H:operator}
\end{align}
where $\mathbf{e}_{i}$ is the $i$-th basis vector in $\mathbb{R}^{d}$, i.e. a $d$-dimensional unit vector with $i$-th coordinate being $1$ and all the other coordinates being $0$.
\end{proposition}

\begin{proof}[Proof of Proposition~\ref{prop:generator}]
Since the asymmetric fractional Langevin dynamics
is driven by the $d$-dimensional $\v{L}_{t}^{\alpha,\theta}$,
it suffices to show that
the infinitesimal generator of $d$-dimensional $\v{L}_{t}^{\alpha,\theta}$ is given by
\begin{align}
\sum_{i=1}^{d}\mathcal{G}^{\alpha,\theta_{i}}_{w_{i}}f(\v{w})
=\sum_{i=1}^{d}\mathcal{H}^{\alpha,\theta_{i}}_{w_{i}}f(\v{w})
-\sum_{i=1}^{d}\frac{\theta_{i}}{\cos(\alpha\pi/2)}
\frac{\alpha}{(\alpha-1)\Gamma(1-\alpha)}\frac{\partial}{\partial w_{i}}f(\v{w}).
\end{align}

We start the proof by considering the dimension $d=1$ first.
The one-dimensional $\alpha$-stable L\'{e}vy motion with tail-index $1<\alpha<2$
and skewness $\theta\in (-1,1)$ has the infinitesimal generator given by:
\begin{align}
\mathcal{G}^{\alpha,\theta}f(w)
&:=\frac{1+\theta}{2}\frac{1}{\cos(\alpha\pi/2)}
\frac{\alpha}{\Gamma(1-\alpha)}
\int_{z>0}[f(w+z)-f(w)-1_{|z|\leq 1}f'(w)z]\frac{dz}{|z|^{1+\alpha}}\nonumber
\\
&\quad
+\frac{1-\theta}{2}\frac{1}{\cos(\alpha\pi/2)}
\frac{\alpha}{\Gamma(1-\alpha)}
\int_{z<0}[f(w+z)-f(w)-1_{|z|\leq 1}f'(w)z]\frac{dz}{|z|^{1+\alpha}}
+af'(w),
\end{align}
where $a\in\mathbb{R}$ is chosen so that $\mathcal{G}^{\alpha,\theta}w=0$
to be consistent with $\mu=0$ in $\v{L}_{t}^{\alpha,\theta}-\v{L}_{s}^{\alpha,\theta}\sim\mathcal{S}_{\alpha}((t-s)^{1/\alpha},\theta,\mu)$ for any $t>s$.
Thus, we can compute that
\begin{align}
\mathcal{G}^{\alpha,\theta}w
=\frac{1+\theta}{2}\frac{1}{\cos(\alpha\pi/2)}
\frac{\alpha}{\Gamma(1-\alpha)}
\int_{z>1}\frac{dz}{z^{\alpha}}
-\frac{1-\theta}{2}\frac{1}{\cos(\alpha\pi/2)}
\frac{\alpha}{\Gamma(1-\alpha)}
\int_{z>1}\frac{dz}{z^{\alpha}}+a=0,
\end{align}
which yields that
\begin{align}
a=-\theta\frac{1}{\cos(\alpha\pi/2)}
\frac{\alpha}{(\alpha-1)\Gamma(1-\alpha)}.
\end{align}
Therefore, with $1<\alpha<2$,
\begin{align}
\mathcal{G}^{\alpha,\theta}f(w)
=\mathcal{H}^{\alpha,\theta}f(w)
-\theta\frac{1}{\cos(\alpha\pi/2)}
\frac{\alpha}{(\alpha-1)\Gamma(1-\alpha)}f'(w),
\end{align}
where
\begin{align}
\mathcal{H}^{\alpha,\theta}f(w)
&:=\left(\frac{1}{2}+\frac{\theta}{2}\right)\frac{1}{\cos(\alpha\pi/2)}
\frac{\alpha}{\Gamma(1-\alpha)}\int_{0}^{\infty}\frac{f(w+\xi)-f(w)-f'(w)\xi}{\xi^{\alpha+1}}d\xi
\nonumber
\\
&\qquad\qquad
+\left(\frac{1}{2}-\frac{\theta}{2}\right)\frac{1}{\cos(\alpha\pi/2)}\frac{\alpha}{\Gamma(1-\alpha)}\int_{0}^{\infty}\frac{f(w-\xi)-f(w)+f'(w)\xi}{\xi^{\alpha+1}}d\xi.
\end{align}
Similarly, for the multi-dimensional case, we can show
that the infinitesimal generator for $\v{L}_{t}^{\alpha,\theta}$ 
is given by:
\begin{align}
\sum_{i=1}^{d}\mathcal{G}^{\alpha,\theta_{i}}_{w_{i}}f(\v{w})
=\sum_{i=1}^{d}\mathcal{H}^{\alpha,\theta_{i}}_{w_{i}}f(\v{w})
-\sum_{i=1}^{d}\frac{\theta_{i}}{\cos(\alpha\pi/2)}
\frac{\alpha}{(\alpha-1)\Gamma(1-\alpha)}\frac{\partial}{\partial w_{i}}f(\v{w}),
\end{align}
where
\begin{align}
\mathcal{H}^{\alpha,\theta_{i}}_{w_{i}}f(\v{w})
&:=\left(\frac{1}{2}+\frac{\theta_{i}}{2}\right)\frac{1}{\cos(\alpha\pi/2)}
\frac{\alpha}{\Gamma(1-\alpha)}\int_{0}^{\infty}\frac{f(\v{w}+\xi\mathbf{e}_{i})-f(\v{w})-\partial_{w_{i}}f(\v{w})\xi}{\xi^{\alpha+1}}d\xi
\nonumber
\\
&\qquad\qquad
+\left(\frac{1}{2}-\frac{\theta_{i}}{2}\right)\frac{1}{\cos(\alpha\pi/2)}\frac{\alpha}{\Gamma(1-\alpha)}\int_{0}^{\infty}\frac{f(\v{w}-\xi\mathbf{e}_{i})-f(\v{w})+\partial_{w_{i}}f(\v{w})\xi}{\xi^{\alpha+1}}d\xi,
\end{align}
where $\mathbf{e}_{i}$ is the $i$-th basis vector in $\mathbb{R}^{d}$, i.e. a $d$-dimensional unit vector with $i$-th coordinate being $1$ and all the other coordinates being $0$.
\end{proof}

We recall that
for any $\theta_{i}\in(-1,1)$, $1\leq i\leq d$, and $1<\alpha<2$, we have
\begin{align}
(b(\v{w},\alpha,\theta))_{i}
=\frac{\varepsilon^{\alpha}}{\varphi(\v{w})}\mathcal{D}^{\alpha-2,-\theta_{i}}_{w_{i}}(\partial_{w_{i}}\varphi(\v{w})),
\qquad
\varphi(\v{w})=e^{-\varepsilon^{-\alpha}f(\v{w})},
\end{align}
where
\begin{align}
\mathcal{D}^{\alpha-2,-\theta_{i}}_{w_{i}}(\partial_{w_{i}}\varphi(\v{w}))
:=\frac{-1}{2\cos(\alpha\pi/2)}\left[(1-\theta_{i})\mathcal{I}^{2-\alpha}_{+,w_{i}}(\partial_{w_{i}}\varphi(\v{w}))
+(1+\theta_{i})\mathcal{I}^{2-\alpha}_{-,w_{i}}(\partial_{w_{i}}\varphi(\v{w}))\right],
\end{align}
and
\begin{align}
\mathcal{I}^{2-\alpha}_{\pm,w_{i}}(\partial_{w_{i}}\varphi(\v{w}))
:=\frac{1}{\Gamma(2-\alpha)}\int_{0}^{\infty}\frac{\partial_{w_{i}}\varphi(\v{w}\pm\xi\mathbf{e}_{i})}{\xi^{\alpha-1}}d\xi.
\end{align}

In the next result, we provide an alternative formula
for $b(\v{w},\alpha,\theta)=((b(\v{w},\alpha,\theta))_{i},1\leq i\leq d)$
that is defined in \eqref{b:defn}. 

\begin{proposition}\label{prop:b:D}
For any $\theta_{i}\in(-1,1)$, $1\leq i\leq d$, and $1<\alpha<2$, we have
\begin{align}
(b(\v{w},\alpha,\theta))_{i}
&:=\frac{\varepsilon^{\alpha}}{\varphi(\v{w})}\left(\frac{1}{2}-\frac{\theta_{i}}{2}\right)\frac{1}{\cos(\alpha\pi/2)}
\frac{\alpha}{\Gamma(1-\alpha)}\int_{0}^{\infty}\frac{\int_{w_{i}}^{w_{i}+\xi}\varphi(\v{w}+(y-w_{i})\mathbf{e}_{i})dy-\varphi(\v{w})\xi}{\xi^{\alpha+1}}d\xi
\nonumber
\\
&\qquad
+\frac{\varepsilon^{\alpha}}{\varphi(\v{w})}\left(\frac{1}{2}+\frac{\theta_{i}}{2}\right)\frac{1}{\cos(\alpha\pi/2)}\frac{\alpha}{\Gamma(1-\alpha)}\int_{0}^{\infty}\frac{\int_{w_{i}}^{w_{i}-\xi}\varphi(\v{w}+(y-w_{i})\mathbf{e}_{i})dy+\varphi(\v{w})\xi}{\xi^{\alpha+1}}d\xi
\nonumber
\\
&\qquad\qquad
+\varepsilon^{\alpha}\theta_{i}\frac{1}{\cos(\alpha\pi/2)}\frac{\alpha}{(\alpha-1)\Gamma(1-\alpha)}\,, 
\end{align}
with $\varphi(\v{w}):=\exp(-\varepsilon^{-\alpha}f(\v{w}))$. 
\end{proposition}

\begin{proof}[Proof of Proposition~\ref{prop:b:D}]
Let us first consider the case
when $\theta_{i}=0$, $1\leq i\leq d$. 
We have
\begin{align}
\mathcal{D}^{\alpha-2}_{w_{i}}(\partial_{w_{i}}\varphi(\v{w}))
&:=-\mathcal{I}^{2-\alpha}_{w_{i}}(\partial_{w_{i}}\varphi(\v{w}))
\nonumber
\\
&=\frac{-1}{2\cos(\alpha\pi/2)}\left[\mathcal{I}^{2-\alpha}_{+,w_{i}}(\partial_{w_{i}}\varphi(\v{w}))
+\mathcal{I}^{2-\alpha}_{-,w_{i}}(\partial_{w_{i}}\varphi(\v{w}))\right],
\end{align}
where
\begin{align}
&\mathcal{I}^{2-\alpha}_{+,w_{i}}(\partial_{w_{i}}\varphi(\v{w}))
:=\frac{1}{\Gamma(2-\alpha)}\int_{0}^{\infty}\frac{\partial_{w_{i}}\varphi(\v{w}+\xi\mathbf{e}_{i})}{\xi^{\alpha-1}}d\xi,
\\
&\mathcal{I}^{2-\alpha}_{-,w_{i}}(\partial_{w_{i}}\varphi(\v{w}))
:=\frac{1}{\Gamma(2-\alpha)}\int_{0}^{\infty}\frac{\partial_{w_{i}}\varphi(\v{w}-\xi\mathbf{e}_{i})}{\xi^{\alpha-1}}d\xi.
\end{align}

Similarly, when $\theta_{i}\in(-1,1)$, $1\leq i\leq d$, and $1<\alpha<2$, we have
\begin{align}
(b(\v{w},\alpha,\theta))_{i}
=\frac{\varepsilon^{\alpha}}{\varphi(\v{w})}\mathcal{D}^{\alpha-2,-\theta_{i}}_{w_{i}}(\partial_{w_{i}}\varphi(\v{w})),
\qquad
\varphi(\v{w})=e^{-\varepsilon^{-\alpha}f(\v{w})},
\end{align}
where
\begin{align}
\mathcal{D}^{\alpha-2,-\theta_{i}}_{w_{i}}(\partial_{w_{i}}\varphi(\v{w}))
:=\frac{-1}{2\cos(\alpha\pi/2)}\left[(1-\theta_{i})\mathcal{I}^{2-\alpha}_{+,w_{i}}(\partial_{w_{i}}\varphi(\v{w}))
+(1+\theta_{i})\mathcal{I}^{2-\alpha}_{-,w_{i}}(\partial_{w_{i}}\varphi(\v{w}))\right].
\end{align}
\end{proof}

%%%%%%%%%%%%%%%%%%%%%%%%%%%%%%%%%%%%%%%%%%%%%%%%%%%%%%%%%
Now, we are ready to prove Theorem~\ref{thm:AFLD}.

\begin{proof}[Proof of Theorem~\ref{thm:AFLD}]
We recall from Proposition~\ref{prop:generator}
that the asymmetric fractional Langevin dynamics $\v{w}_{t}$ 
has the infinitesimal generator:
\begin{align} \label{eqn:generator}
\mathcal{L}f(\v{w})=\sum_{i=1}^{d}
\left((b(\v{w},\alpha,\theta))_{i}-\varepsilon^{\alpha}\frac{\theta_{i}}{\cos(\alpha\pi/2)}\frac{\alpha}{(\alpha-1)\Gamma(1-\alpha)}\right)\frac{\partial f(\v{w})}{\partial w_{i}}
+\varepsilon^{\alpha}\sum_{i=1}^{d}\mathcal{H}^{\alpha,\theta_{i}}_{w_{i}}f(\v{w}),
\end{align}
where $\mathcal{H}^{\alpha,\theta_{i}}_{w_{i}}f(\v{w})$ is given in \eqref{defn:H:operator}.

It follows that the adjoint operator $\mathcal{L}^{\ast}$
of $\mathcal{L}$ is given by:
\begin{align}
\mathcal{L}^{\ast}f(\v{w})
=-\sum_{i=1}^{d}\frac{\partial}{\partial w_{i}}\left(\left((b(\v{w},\alpha,\theta))_{i}
-\varepsilon^{\alpha}\frac{\theta_{i}}{\cos(\alpha\pi/2)}\frac{\alpha}{(\alpha-1)\Gamma(1-\alpha)}\right)f(\v{w})\right)
+\varepsilon^{\alpha}\sum_{i=1}^{d}\mathcal{H}^{\alpha,-\theta_{i}}_{w_{i}}f(\v{w})\,.
\end{align}

The probability density function $p(\v{w},t)$ of the 
L\'{e}vy-driven SDE
satisfies the Fokker-Planck equation \cite{SLDYL}:
\begin{align}
\partial_t p(\v{w},t)
&=\mathcal{L}^{\ast}p(\v{w},t)
\nonumber
\\
&=-\sum_{i=1}^{d}\frac{\partial}{\partial w_i} \left[\left((b(\v{w},\alpha,\theta))_{i}-\varepsilon^{\alpha}\theta_{i}\frac{1}{\cos(\alpha\pi/2)}\frac{\alpha}{(\alpha-1)\Gamma(1-\alpha)}\right) p(\v{w},t)\right] 
+\varepsilon^{\alpha}\sum_{i=1}^{d}\mathcal{H}_{w_{i}}^{\alpha,-\theta_{i}}p(\v{w},t).
\end{align}
We can compute that
\begin{align}
&\sum_{i=1}^{d}\frac{\partial}{\partial w_i} \left[\left((b(\v{w},\alpha,\theta))_{i}-\varepsilon^{\alpha}\theta_{i}\frac{1}{\cos(\alpha\pi/2)}\frac{\alpha}{(\alpha-1)\Gamma(1-\alpha)}\right)\varphi(\v{w})\right] 
+\varepsilon^{\alpha}\sum_{i=1}^{d}\mathcal{H}_{w_{i}}^{\alpha,-\theta_{i}}\varphi(\v{w})
\nonumber
\\
&=-\varepsilon^{\alpha}\sum_{i=1}^{d}\frac{\partial}{\partial w_i}\left[\left(\frac{1}{2}-\frac{\theta_{i}}{2}\right)\frac{1}{\cos(\alpha\pi/2)}
\frac{\alpha}{\Gamma(1-\alpha)}\int_{0}^{\infty}\frac{\int_{w_{i}}^{w_{i}+\xi}\varphi(\v{w}+(y-w_{i})\mathbf{e}_{i})dy-\varphi(\v{w})\xi}{\xi^{\alpha+1}}d\xi\right]
\nonumber
\\
&\quad
-\varepsilon^{\alpha}\sum_{i=1}^{d}\frac{\partial}{\partial w_i}\left[\left(\frac{1}{2}+\frac{\theta_{i}}{2}\right)\frac{1}{\cos(\alpha\pi/2)}\frac{\alpha}{\Gamma(1-\alpha)}\int_{0}^{\infty}\frac{\int_{w_{i}}^{w_{i}-\xi}\varphi(\v{w}+(y-w_{i})\mathbf{e}_{i})dy+\varphi(\v{w})\xi}{\xi^{\alpha+1}}d\xi\right]
\nonumber
\\
&\qquad
+\varepsilon^{\alpha}\sum_{i=1}^{d}\mathcal{H}_{w_{i}}^{\alpha,-\theta_{i}}\varphi(\v{w})
\nonumber
\\
&=
-\varepsilon^{\alpha}\sum_{i=1}^{d}\left(\frac{1}{2}-\frac{\theta_{i}}{2}\right)\frac{1}{\cos(\alpha\pi/2)}
\frac{\alpha}{\Gamma(1-\alpha)}\int_{0}^{\infty}\frac{\varphi(\v{w}+\xi\mathbf{e}_{i})-\varphi(\v{w})-\partial_{w_{i}}\varphi(\v{w})\xi}{\xi^{\alpha+1}}d\xi
\nonumber
\\
&\qquad
-\varepsilon^{\alpha}\sum_{i=1}^{d}\left(\frac{1}{2}+\frac{\theta_{i}}{2}\right)\frac{1}{\cos(\alpha\pi/2)}\frac{\alpha}{\Gamma(1-\alpha)}\int_{0}^{\infty}\frac{\varphi(\v{w}-\xi\mathbf{e}_{i})-\varphi(\v{w})+\partial_{w_{i}}\varphi(\v{w})\xi}{\xi^{\alpha+1}}d\xi
\nonumber
\\
&\qquad\qquad\qquad\qquad\qquad
+\varepsilon^{\alpha}\sum_{i=1}^{d}\mathcal{H}_{w_{i}}^{\alpha,-\theta_{i}}\varphi(\v{w})
=0.
\nonumber
\end{align}
Hence, we conclude that 
$\pi(d\v{w})=\exp(-\varepsilon^{-\alpha}f(\v{w}))d\v{w}/\int_{\mathbb{R}^{d}}\exp(-\varepsilon^{-\alpha}(\v{w}))d\v{w}$
is an invariant distribution of the asymmetric fractional Langevin dynamics \eqref{eqn:AFLD}.
Finally, if $b(\v{w},\alpha,\theta)$ is Lipschitz continuous in $\v{w}$, 
then $\pi(d\v{w})$ is the unique invariant distribution of \eqref{eqn:AFLD}, 
see e.g. \citet{SLDYL}.
\end{proof}

%%%%%%%%%%%%%%%%%%%%%%%%%%%%%%%%%%%%%%%%%%%%%%

%%%%%%%%%%%%%%%%%%%%%%%%%%%%%%%%%%%%%%%%%%%%

\subsection{Proof of Theorem~\ref{thm:approx}}

Theorem~\ref{thm:approx} provides
a first-order approximation of
the fractional derivative $\mathcal{D}^{\gamma,-\theta}$ when $d=1$.

Based on the work of \citet{MT2004}, we will show a first-order approximation for the asymmetric fractional derivative $\mathcal{D}^{-\gamma,-\theta}$ when $-1 < \gamma < 0$ by using the shifted Gr\"unwald-Letnikov difference operators defined in~\eqref{eqn:left:A} and ~\eqref{eqn:right:B}. Before we proceed to the proof of Theorem~\ref{thm:approx}, 
we will first present the Fourier transform property from equations (1) and (12) in~\citet{TZD2015}.

\begin{property}[\cite{TZD2015}]
Let $-1<\gamma<0$ and $f \in L^1(\mathbb{R})$. The Fourier transform of $\mathcal{I}_{-}^{-\gamma}f$
and $\mathcal{I}_{+}^{-\gamma}f$ satisfy the following identities:
\begin{equation} \label{eqn:FT}
\mathcal{F}\left[\mathcal{I}^{-\gamma}_{-}f(w)\right](\zeta) = (i\zeta)^{\gamma}\hat{f}(\zeta),\quad \mathcal{F}\left[\mathcal{I}^{-\gamma}_{+}f(w)\right](\zeta) = (-i\zeta)^{\gamma}\hat{f}(\zeta),
\end{equation}
where $\hat{f}(\zeta)$ denotes the Fourier transform of $f$, such that $\hat{f}(\zeta) = \int_{-\infty}^{\infty} e^{-i\zeta w}f(w)dw.$
\end{property}

Now, we are ready to prove Theorem~\ref{thm:approx}.

\begin{proof}[Proof of Theorem~\ref{thm:approx}]
The main idea for the proof of Theorem~\ref{thm:approx} is to use the Fourier transform to estimate the difference between $\mathcal{F}\left[\left((1+\theta)\mathcal{A}^{\gamma}_{h,p} + (1-\theta)\mathcal{B}^{\gamma}_{h,q}\right)f(w)\right](\zeta) $ and $\mathcal{F}\left[\left((1+\theta)\mathcal{I}^{-\gamma}_{-} + (1-\theta)\mathcal{I}^{-\gamma}_{+}\right)f(w)\right](\zeta)$, and then apply the inverse Fourier transform to complete the proof. By the linearity of Fourier transforms, we can apply Fourier transform to~\eqref{eqn:left:A} to obtain
\begin{align}
\mathcal{F}\left[\mathcal{A}^{\gamma}_{h,p}f(w)\right](\zeta) & = \frac{1}{h^{\gamma}} \sum_{k=0}^{\infty} (-1)^k\binom{-\gamma+k-1}{k}e^{-i\zeta(k-p)h}\hat{f}(\zeta) \nonumber \\
& = \frac{1}{h^{\gamma}}e^{i\zeta ph} \left(1-e^{-i \zeta h}\right)^{\gamma} \hat{f}(\zeta) \nonumber 
\\
& = (i \zeta )^\gamma \left(\frac{1-e^{-i \zeta h}}{i\zeta h}\right)^{\gamma} e^{i\zeta ph} \hat{f}(\zeta) 
\nonumber
\\
& = W_p(i\zeta h)(i\zeta)^{\gamma} \hat{f}(\zeta).
\end{align}
Similarly, we can compute that
\begin{align}
\mathcal{F}\left[\mathcal{B}^{\gamma}_{h,q}f(w)\right](\zeta) &
= \frac{1}{h^{\gamma}} \sum_{k=0}^{\infty} (-1)^k\binom{-\gamma+k-1}{k}e^{i\zeta(k-q)h}\hat{f}(\zeta) \nonumber \\
& = (-i \zeta )^\gamma \left(\frac{1-e^{i \zeta h}}{-i\zeta h}\right)^{\gamma}e^{-i\zeta q h} \hat{f}(\zeta) 
\nonumber \\
& = W_{-q}(-i \zeta h)(-i\zeta)^{\gamma}\hat{f}(\zeta).
\end{align}
In addition, since $W_p(z)$ and $W_{-q}(-z)$ are analytic for any complex number $\lvert z \rvert \leq 1$, 
there exist series expansions so that by the first-order Taylor expansion we have
\begin{align}
& W_p(z) := \left(\frac{1-e^{-z}}{z}\right)^\gamma e^{pz} = 1 + \left(p-\frac{\gamma}{2}\right)z + \mathcal{O}\left(\lvert z \rvert ^2\right),
\nonumber \\
& W_{-q}(-z) := \left(\frac{1-e^{z}}{-z}\right)^\gamma e^{-qz} = 1 -\left(q-\frac{\gamma}{2}\right)z + \mathcal{O}\left(\lvert z \rvert ^2\right).
\label{eqn:series}
\end{align}
Next, define a function $\hat{\psi}(h,\zeta)$ as the difference between $\mathcal{F}\left[\left((1+\theta)\mathcal{A}^{\gamma}_{h,p} + (1-\theta)\,\mathcal{B}^{\gamma}_{h,q}\right)f(w)\right](\zeta) $ and $\mathcal{F}\left[\left((1+\theta)\mathcal{I}^{-\gamma}_{-}+ (1-\theta)\,\mathcal{I}^{-\gamma}_{+}\right)f(w)\right](\zeta)$. By the linearity of Fourier transform, we have 
\begin{align}
\hat{\psi}(h,\zeta) & = (1+\theta)\left(\mathcal{F}\left[\mathcal{A}^{\gamma}_{h,p}f(w)\right](\zeta)-\mathcal{F}\left[\mathcal{I}_{-}^{-\gamma}f(w)\right](\zeta)\right) + (1-\theta)\left(\mathcal{F}\left[\mathcal{B}^{\gamma}_{h,q}f(w)\right](\zeta)-\mathcal{F}\left[\mathcal{I}_{+}^{-\gamma}f(w)\right](\zeta)\right) 
\nonumber 
\\
& = (1+\theta) (i\zeta)^{\gamma} \hat{f}(\zeta) \left( W_p(i\zeta h) - 1\right) + (1-\theta) (-i\zeta)^{\gamma}  \hat{f}(\zeta) \left(W_{-q}(-i\zeta h) - 1\right)
\nonumber \\
& = (1+\theta) (i\zeta)^{\gamma} \hat{f}(\zeta) \left( p -\frac{\gamma}{2}\right)(i\zeta h) - (1-\theta) (-i\zeta)^{\gamma}  \hat{f}(\zeta) \left( q -\frac{\gamma}{2}\right)(i\zeta h)
\nonumber \\
& = (1+\theta) (i\zeta)^{\gamma+1} \hat{f}(\zeta) \left( p -\frac{\gamma}{2}\right)h + (1-\theta) (-i\zeta)^{\gamma+1}  \hat{f}(\zeta) \left( q -\frac{\gamma}{2}\right)h
\nonumber \\
& \stackrel{\text{(a)}}{=} \left[(1+\theta) e^{i\pi\gamma/2} \left( p -\frac{\gamma}{2}\right)h - (1-\theta) e^{-i\pi\gamma/2}  \left( q -\frac{\gamma}{2}\right)\right]|\zeta|^{1+\gamma}\hat{f}(\zeta) h
\nonumber \\
& \stackrel{\text{(b)}}{=} \left[\cos\left(\frac{\gamma\pi}{2}\right)\left((p-q) + \theta(p+q-\gamma) \right) + \sin\left(\frac{\gamma\pi}{2}\right)\left((p+q-\gamma) + \theta(p-q) \right)\, i\right]  \, |\zeta|^{1+\gamma}\hat{f}(\zeta) h,
\end{align}
where we used the fact that for any real $x$, $ 0< 1+\gamma < 1$, we have $ix = |x|e^{i\,\sign(x)\pi/2}$ so that $(ix)^{1+\gamma} = |x|^{1+\gamma}e^{i\,\sign(x)\pi(\gamma+1)/2} = \sign(x)|x|^{1+\gamma}e^{i\,\sign(x)\pi\gamma/2}$ which implies equality (a), 
and we applied Euler's formula with $-1 < \gamma < 0$ to get equality (b). 
By our assumption $f \in \mathcal{C}^{4}(\mathbb{R})$, we have 
$$
\lvert \hat{f}(\zeta) \rvert \leq C(1+\lvert \zeta \rvert)^{-4},
$$ 
for a constant $C>0$ that may depend on $f$. Hence, by taking a sufficiently small $h$, we obtain,
\begin{align*}
\lvert \hat{\psi}(h,\zeta) \rvert & \leq \left| \left[\cos^2\left(\frac{\gamma\pi}{2}\right)\left(p-q + \theta(p+q-\gamma) \right)^2 + \sin^2\left(\frac{\gamma\pi}{2}\right)\left(p+q-\gamma + \theta(p-q) \right)^2 \right]\right|^{\frac{1}{2}} \, C(1+\lvert \zeta\rvert)^{\gamma-3}h +  c_0 h^2
\nonumber \\
& \leq \left[\cos\left(\frac{\gamma\pi}{2}\right)\left|p-q + \theta(p+q-\gamma) \right| + \left|\sin\left(\frac{\gamma\pi}{2}\right)\right|\left|p+q-\gamma + \theta(p-q) \right| \right]\, C(1+\lvert \zeta\rvert)^{\gamma-3}h +  c_0 h^2,
\end{align*}
where we also used the inequality that $\lvert \zeta \rvert^{\gamma+1} \leq (1+\lvert \zeta \rvert)^{\gamma+1}$ for $-1<\gamma<0$,
and $c_0 > 0$ is a constant may depend on $p, q$.
When $f \in L^1(\mathbb{R})$, the inverse Fourier transform exists with $-1<\gamma<0$, i.e. $\psi(h,w)=\frac{1}{2\pi i} \int_{-\infty}^{\infty} e^{-i\zeta w}\hat{\psi}(h,\zeta)d\zeta$, and it follows that
\begin{align}
\lvert \psi(h,w) \rvert 
&= \frac{1}{2\pi} \int_{-\infty}^{\infty} \hat{\psi}(h,\zeta) e^{-i \zeta w}  d \zeta 
\nonumber \\
& \leq  \frac{1}{2\pi} \int_{-\infty}^{\infty} \lvert \hat{\psi}(h,\zeta) \rvert d \zeta 
\nonumber 
\\
&\leq\left[\cos\left(\frac{\gamma\pi}{2}\right)\left|p-q + \theta(p+q-\gamma) \right| + \left|\sin\left(\frac{\gamma\pi}{2}\right)\right|\left|p+q-\gamma + \theta(p-q) \right| \right] \, \frac{C}{2\pi(|\gamma|+2)}h + \mathcal{O}\left(h^2\right),
\end{align}
where $\mathcal{O}(\cdot)$ hides the dependence on $p$, $q$ and $\gamma$, and $C>0$ is a constant that may depend on $f \in L^1(\mathbb{R}) \cap \mathcal{C}^{4}(\mathbb{R})$.

Hence, we conclude that
\begin{align}
&\left\lvert \mathcal{D}^{\gamma,-\theta}f(w) - \Delta_{h,p,q}^{\gamma,-\theta}f(w) \right\rvert 
\nonumber \\
&  = \frac{1}{2\cos(\pi\gamma/2)}\left\lvert (1+\theta)\left(\mathcal{A}^{\gamma}_{h,p}f(w) -\mathcal{I}_{-}^{\gamma}f(w)\right) +(1-\theta)\left(\mathcal{B}^{\gamma}_{h,q}f(w) - \mathcal{I}_{+}^{\gamma}f(w)\right)\right\rvert
\nonumber \\
& \leq \left[|p-q| + |\theta|(p+q-\gamma) + \left|\tan\left(\frac{\gamma\pi}{2}\right)\right|\left(p+q-\gamma + |\theta||p-q| \right)\right] \, \frac{C}{4\pi(|\gamma|+2)}h + \mathcal{O}\left(h^2\right) .
\end{align}
The proof is complete.
\end{proof}

%%%%%%%%%%%%%%%%%%%%%%%%%%%%%%%%%%%%%%%%%%%%%%%
\subsection{Proof of Corollary~\ref{cor:K}}
With the definitions of the truncated series $\mathcal{A}_{h,p,K}^{\gamma}$ defined in~\eqref{eqn:left:At} and $\mathcal{B}_{h,q,K}^{\gamma}$ in  \eqref{eqn:right:Bt}, we are now ready to prove Corollary~\ref{cor:K}.

\begin{proof}[Proof of Corollary~\ref{cor:K}]
We will first control the difference $\left\lvert \Delta_{h,p,q}^{\gamma,-\theta}\partial_{w}\varphi(w)  - \Delta_{h,p,q,K}^{\gamma,-\theta}\partial_{w}\varphi(w)  \right\rvert $. Then the triangular inequality can be applied with the fractional derivative approximation error bound in Theorem~\ref{thm:approx} to get the numerical truncation error.  

By using the definitions of $\mathcal{A}_{h,p}, \mathcal{B}_{h,q}$ and $\mathcal{A}_{h,p,K}, \mathcal{B}_{h,q,K}$, under the Assumption~\ref{H0}, there exist two universal constants $C_p > 0$ and $C_q > 0$ so that
\begin{align}
& \left|\Delta_{h,p,q}^{\gamma,-\theta}\partial_{w}\varphi(w)  - \Delta_{h,p,q,K}^{\gamma,-\theta}\partial_{w}\varphi(w)  \right|
\nonumber
\\ 
& = \frac{1}{2\lvert \cos(\pi\gamma/2) \rvert}\left\lvert (1+\theta)\left( \mathcal{A}_{h,p}^{\gamma}\partial_w \varphi(w-(k-p)h) -  \mathcal{A}_{h,p,K}^{\gamma}\partial_w \varphi(w-(k-p)h) \right) 
\right.
\nonumber
\\
& \qquad\qquad\qquad\qquad\qquad\qquad \left. + (1-\theta)\left( \mathcal{B}_{h,q}^{\gamma}\partial_w \varphi(w+(k-q)h) -  \mathcal{B}_{h,q,K}^{\gamma}\partial_w \varphi(w+(k-q)h) \right)    \right\rvert
\nonumber
\\
& \leq \frac{1}{2\lvert \cos(\pi\gamma/2) \rvert}\frac{1}{\Gamma(-\gamma)} \frac{1}{h^{\gamma}} \left(\sum_{k=K+p+1}^{\infty} \frac{\Gamma(-\gamma+k)}{\Gamma(k+1)}(1+\theta) \lvert\partial_{w}\varphi(w-(k-p)h) \rvert 
\right.
\nonumber
\\
& \qquad\qquad\qquad\qquad\qquad\qquad \left. + \sum_{k=K+q+1}^{\infty} \frac{\Gamma(-\gamma+k)}{\Gamma(k+1)}(1-\theta) \lvert \partial_{w}\varphi(w+(k-q)h) \rvert \right)
\nonumber
\\
& \leq \frac{(1+\theta)C_p}{h^{\gamma}} \sum_{k=K+p+1}^{\infty} \frac{\Gamma(-\gamma+k)}{\Gamma(k+1)} e^{-(k-p)h} +  \frac{(1-\theta)C_q }{h^{\gamma}} \sum_{k=K+q+1}^{\infty} \frac{\Gamma(-\gamma+k)}{\Gamma(k+1)}e^{-(k-q)h}.\label{follow:part:1}
\end{align}
Next, by applying Stirling's formula, we have as $k \rightarrow \infty$: 
\begin{align*}
\frac{\Gamma(-\gamma+k)}{\Gamma(k+1)} & \sim \frac{\sqrt{2\pi (k-1-\gamma)} \, (k-1-\gamma)^ {k-1-\gamma} \, e^{- (k-1-\gamma)}}{\sqrt{2\pi k} \,k^k \,e^{-k}} \\
& = \frac{(k-1-\gamma)\,^{k-1/2-\gamma}}{k\,^{k+1/2}}e^{1+\gamma} \\
& = k^{-\gamma-1}\left( 1- \frac{1+\gamma}{k} \right)^k \left( \frac{k-1-\gamma}{k} \right)^{-1/2-\gamma}e^{1+\gamma} \\
& \sim k^{-\gamma-1}.
\end{align*}
Therefore, it follows from \eqref{follow:part:1} that
\begin{align}
& \left|\Delta_{h,p,q}^{\gamma,-\theta}\partial_{w}\varphi(w)  - \Delta_{h,p,q,K}^{\gamma,-\theta}\partial_{w}\varphi(w)  \right|
\nonumber
\\
&\leq (1+\theta)C_ph \sum_{k=K+p+1}^{\infty} (hk)^{-\gamma-1} e^{-(k-p)h} + (1-\theta)C_qh \sum_{k=K+q+1}^{\infty} (hk)^{-\gamma-1} e^{-(k-p)h}
\nonumber
\\
 & \leq \left((1+\theta)C_p + (1-\theta)C_q  \right)\frac{1}{hK},\label{follow:part:2}
 \end{align}
where we abused the notation such that $C_p$, $C_q$ in \eqref{follow:part:2} may differ from
$C_p$, $C_q$ in \eqref{follow:part:1}.
Finally, the triangular inequality yields that
\begin{align*}
& \left\lvert \mathcal{D}^{\gamma,-\theta}\partial_{w}\varphi(w) - \Delta_{h,p,q,K}^{\gamma,-\theta}\partial_{w}\varphi(w)  \right\rvert 
\\
& \quad \leq \left|\mathcal{D}^{\gamma,-\theta}\partial_{w}\varphi(w)  - \Delta_{h,p,q}^{\gamma,-\theta}\partial_{w}\varphi(w)  \right| + \left|\Delta_{h,p,q}^{\gamma,-\theta}\partial_{w}\varphi(w)  - \Delta_{h,p,q,K}^{\gamma,-\theta}\partial_{w}\varphi(w)  \right|
\\
& \quad \leq \left[|p-q| + |\theta|(p+q-\gamma) + \left|\tan\left(\frac{\gamma\pi}{2}\right)\right|\left(p+q-\gamma + |\theta||p-q| \right)\right]\frac{C}{4\pi(|\gamma|+2)} h 
\nonumber \\
& \qquad\qquad\qquad + \left((1+\theta)C_p + (1-\theta)C_q  \right)\frac{1}{hK} +  \mathcal{O}\left(h^2\right),
\end{align*}
where $C_p$ and $C_q$ are two universal constants following Assumption~\ref{H0}. The proof is completed. 
\end{proof}

%%%%%%%%%%%%%%%%%%%%%%%%%%%%%%%%%%%%%%%%%%%%%%%%%%%%
\subsection{Proof of Theorem~\ref{thm:bias:approx}}
\label{sec:appx_thmbias}

First, let us recall that
$\tilde{\nu}_{N}(g)=\frac{1}{H_{N}}\sum_{k=1}^{N}\eta_{k}g(\v{w}_{k})$ is the sample average, where $\v{w}_{k}$
satisfies the Euler-Maruyama discretisation with the approximated drift $b_{h,K}$:
\begin{align}\label{correspond:SDE}
\tilde{\v{w}}_{n+1} = \tilde{\v{w}}_n + \eta_{n+1}b_{h,K}(\tilde{\v{w}}_n,\alpha,\theta) + \varepsilon \eta_{n+1}^{1/\alpha} \Delta \v{L}^{\alpha,\theta}_{n+1},
\end{align}
The corresponding SDE of \eqref{correspond:SDE} is given as
\begin{align} \label{SDE2}
d\tilde{\v{w}}_t=b_{h,K}(\tilde{\v{w}}_{t-},\alpha,\theta)dt +\varepsilon d\v{L}_t^{\alpha,\theta},
\end{align}
and we define $\tilde{\nu}(g)=\int g(\v{w})\tilde{\pi}(d\v{w})$, where $\tilde{\pi}$ is the stationary distribution of \eqref{SDE2}. 

Next, let us introduce the following assumption that is
needed for Theorem~\ref{thm:bias:approx}.

\begin{assumption} \label{H1}
(i) Assume that the step sizes are decreasing and the sum diverges such that $
\lim_{n \rightarrow \infty} \eta_n=0\,,\lim_{N \rightarrow \infty} H_N = \infty$ .

(ii) Let $V: \mathbb{R} \rightarrow \mathbb{R}^*_+$ be a function in $\mathcal{C}^2$, if $\,\lim_{|x| \rightarrow \infty} V(w) = \infty$, $|\partial_w V| \leq C\sqrt{V}$ with some constant $C>0$ and $\partial^2_x V$ is bounded. Then there exists $a \in (0,1]$, $\delta>0$ and $\beta \in \mathbb{R}$, such that $|b|^2 \leq CV^a$ and $b(\partial_w V) \leq \beta - \delta V^a$ with $b$ defined in~\eqref{b:defn}. And the statement also holds for $\tilde{b}$. 

(iii) The SDEs defined in \eqref{eqn:AFLD} and \eqref{SDE2} are geometrically ergodic with their unique invariant measures.
\end{assumption} 

Before we proceed to the proof of Theorem~\ref{thm:bias:approx}, let us state a 
technical lemma bounding the error 
of $\left\lvert \mathbb{E}[g(\v{w}_{t})] - \mathbb{E}[g(\tilde{\v{w}}_t)] \right\rvert$, where $(\v{w}_{t})_{t \geq 0}$ and $(\tilde{\v{w}}_t)_{t \geq 0}$ follow SDEs in~\eqref{eqn:AFLD} and~\eqref{SDE2}.

\begin{lemma} \label{expErr}
Let $(\v{w}_{t})_{t \geq 0}$ and $(\tilde{\v{w}}_t)_{t \geq 0}$ follow SDEs in~\eqref{eqn:AFLD} and~\eqref{SDE2} and $g$ be a given test function with bounded $|\partial_w g|$. Suppose $K \in \mathbb{N}\cup\{0\}$ is a constant satisfying Assumption~\ref{H0} with respect to $\partial_x \varphi$ and Assumption~\ref{H1} holds, then the following bound holds:
\begin{align}
\left\lvert \mathbb{E}\left[g\left(\v{w}_{t}\right)\right] - \mathbb{E}\left[g\left(\tilde{\v{w}}_t\right)\right] \right\rvert 
& \leq\frac{\tilde{C}}{4\pi(|\gamma|+2)}\left[|p-q| + |\theta|(p+q-\gamma) + \left|\tan\left(\frac{\gamma\pi}{2}\right)\right|\left(p+q-\gamma + |\theta||p-q| \right)\right] h 
\nonumber \\
& \qquad\qquad\qquad+ \left((1+\theta)C'_p + (1-\theta)C'_q  \right)\frac{1}{hK}+\mathcal{O}\left(h^{2}\right)\,,
\end{align}
where the constants $\tilde{C},C_p',C_q'>0$ may depend on the function $\partial_w \varphi$ and the bound for $\lvert \partial_w g \rvert$.
\end{lemma}

\begin{proof}
The proof is inspired by the proof of Lemma~3 in \citet{FLMC}. Let $\{ P^{\v{w}}_t \}_{t \geq 0}$ and $\{ P^{\tilde{\v{w}}}_t \}_{t \geq 0}$ be the corresponding Markov semigroups, i.e.
$
P_t^{\v{w}}g(w) = \mathbb{E}_{w}[g(\v{w}_{t})],\, P_t^{\tilde{\v{w}}}g(w) = \mathbb{E}_{w}[g(\tilde{\v{w}}_t)].
$
Using the Markov semigroup property, 
following Lemma~3 in \citet{FLMC}, we have
\begin{align*}
\left\lvert \mathbb{E}\left[g\left(\v{w}_{t}\right)\right] - \mathbb{E}[g(\tilde{\v{w}}_t)] \right\rvert & = \left\lvert  \int_0^t P_s^{\v{w}} \left( \mathcal{L}^{\v{w}} - \mathcal{L}^{\tilde{\v{w}}}\right) P_{t-s}^{\tilde{\v{w}}}g(w) ds \right\rvert\,,
\end{align*}
where $\mathcal{L}^{\v{w}}$ and $\mathcal{L}^{\tilde{\v{w}}}$ are the linear generators of $P_t^{\v{w}}$ and $P_t^{\tilde{\v{w}}}$, such that, for $g \in L^2(\pi)$,\,$ \partial_t P_t g = \mathcal{L} P_t g = P_t \mathcal{L} g $.
The infinitesimal generators $\mathcal{L}^{\v{w}}$ and $\mathcal{L}^{\tilde{\v{w}}}$ are computed in~\eqref{eqn:generator}. By the interchangeability of integration and differentiation, we have
\begin{align*}
\left\lvert  \int_0^t P_s^{\v{w}} \left( \mathcal{L}^{\v{w}} - \mathcal{L}^{\tilde{\v{w}}}\right) P_{t-s}^{\tilde{\v{w}}}g(w) ds \right\rvert = \left\lvert  \int_0^t P_s^{\v{w}} \left( b(w,\alpha,\theta) - b_{h,K}(w,\alpha,\theta)\right) P_{t-s}^{\tilde{\v{w}}}\partial_wg(w) ds \right\rvert.
\end{align*}
By the ergodicity assumptions, for a bounded function $f$, there exist some constants $c>0$ and $\lambda_w, \lambda_{\tilde{w}}>0$ so that
\begin{align}
\lvert P_s^{\v{w}}f\rvert \leq c \, e^{-\lambda_w s}\lVert f \rVert_{\infty}, \qquad \lvert P_{t-s}^{\tilde{\v{w}}}f\rvert \leq c \, e^{-\lambda_{\tilde{w}}(t-s)}\lVert f \rVert_{\infty}.
\end{align}
Using the boundedness assumption for $\lvert \partial_w g \rvert $ and Corollary~\ref{cor:K}, 
and the fact that $\int_0^t e^{-\lambda_w s}ds \leq \frac{1}{\lambda_{w}}$, we have
\begin{align}
\left\lvert \mathbb{E}[g(\v{w}_{t})] - \mathbb{E}[g(\tilde{\v{w}}_t)] \right\rvert
& \leq \frac{\tilde{C}}{4\pi(|\gamma|+2)}\left[|p-q| + |\theta|(p+q-\gamma) + \left|\tan\left(\frac{\gamma\pi}{2}\right)\right|\left(p+q-\gamma + |\theta||p-q| \right)\right] h 
\nonumber \\
& \qquad\qquad\qquad+ \left((1+\theta)C'_p + (1-\theta)C'_q  \right)\frac{1}{hK}+\mathcal{O}\left(h^{2}\right)\,,
\end{align}
where the constants $\tilde{C} = \frac{c}{\lambda_{w}}\,C$, $C'_p = \frac{c}{\lambda_{w}} \, C_p$ and $C'_q = \frac{c}{\lambda_{w}} \, C_q$ may depend on $\partial_w \varphi$ and the bound for $\lvert \partial_w g \rvert$.
\end{proof}

Now we are ready to prove Theorem~\ref{thm:bias:approx}.

\begin{proof}[Proof of Theorem~\ref{thm:bias:approx}]
With the ergodicity assumptions, we have
\begin{align}
\left|\nu(g) - \lim_{N \rightarrow \infty} \tilde{\nu}_N(g) \right|  & = \left\lvert \nu(g)-\tilde{\nu}(g)+\tilde{\nu}(g) - \lim_{N \rightarrow \infty} \tilde{\nu}_N(g) \right\rvert 
\leq \lim_{t \rightarrow \infty} \left| \mathbb{E}\left[g(\v{w}_{t})\right] - \mathbb{E}\left[g(\tilde{\v{w}}_t)\right] \right|  + \left\lvert\tilde{\nu}(g) - \lim_{N \rightarrow \infty} \tilde{\nu}_N(g) \right\rvert .
\end{align}
By Assumption~\ref{H1}(ii),  \cite{panloup2008recursive} and  similar 
arguments as in \cite{FLMC}, we get,
$$
\left\lvert \tilde{\nu}(g) - \lim_{N \rightarrow \infty} \tilde{\nu}_N(g) \right\rvert =0, \quad \text{a.s.}
$$
By applying Lemma~\ref{expErr} as $t \rightarrow \infty$, we obtain:
\begin{align}
\left\lvert \nu(g) - \lim_{N \rightarrow \infty} \tilde{\nu}_N (g) \right\rvert 
&\leq \lim_{t \rightarrow \infty}\left\lvert \mathbb{E}[g(\v{w}_{t})] - \mathbb{E}[g(\tilde{\v{w}}_t)]\right\rvert\nonumber \\
&  \leq \frac{\tilde{C}}{4\pi(|\gamma|+2)}\left[|p-q| + |\theta|(p+q-\gamma) + \left|\tan\left(\frac{\gamma\pi}{2}\right)\right|\left(p+q-\gamma + |\theta||p-q| \right)\right] h 
\nonumber \\
& \qquad\qquad\qquad + \left((1+\theta)C'_p + (1-\theta)C'_q  \right)\frac{1}{hK}+\mathcal{O}\left(h^{2}\right)\,,
\end{align}
where $\tilde{C},\,C_p',\,C_q'>0$ are constants 
that may depend on $\partial_w \varphi$ and the bound of $\lvert \partial_w g \rvert$.
Finally, by taking $p=q=0$, we complete the proof.
\end{proof}
\end{document}